\RequirePackage{expl3}
\PassOptionsToPackage{table}{xcolor}
\documentclass[acmsmall]{acmart}
\AtBeginDocument{%
  }

\setcopyright{acmlicensed}
\copyrightyear{2018}
\acmYear{2018}
\acmDOI{XXXXXXX.XXXXXXX}

\acmJournal{JACM}
\acmVolume{37}
\acmNumber{4}
\acmArticle{111}
\acmMonth{8}

\usepackage{amsmath}
\usepackage{amsthm}

\newtheorem{proposition}{\textbf{Proposition}}
\usepackage{algorithm}
\usepackage{algorithmic}
\usepackage{graphicx}
\usepackage{subfigure}
\usepackage{multirow}
\usepackage{multicol}
\usepackage{amsmath}
\usepackage{bm}
\allowdisplaybreaks

\usepackage{dashrule}
\usepackage{algorithm}
\usepackage{algorithmic}
\usepackage{booktabs}
\usepackage{subfigure}
\usepackage{diagbox}
\usepackage{expl3}
\begin{document}

\title{Novelty Search with Cross-Task Collaborative Discovery}

\author{Jiao Liu}
\affiliation{%
	\institution{College of Computing \& Data Science, Nanyang Technological University}
	\city{Singapore}
	\country{Singapore}}
\email{jiao.liu@ntu.edu.sg}

\author{Yanchi Li}
\affiliation{%
	\institution{School of Computer Science, China University of Geosciences}
	\city{Wuhan}
	\country{China}}
\email{int\_lyc@cug.edu.cn}

\author{Hua Yu}
\affiliation{%
	\institution{School of Intelligent Rail Transportation, Dalian Jiaotong University}
	\city{Dalian}
	\country{China}}
\email{yhiccd@163.com}

\author{Abhishek Gupta}
\affiliation{%
	\institution{School of Mechanical Sciences, Indian Institute of Technology, Goa}
	\city{Goa}
	\country{India}}
\email{abhishekgupta@iitgoa.ac.in}

\author{Yew-Soon Ong}
\affiliation{%
	\institution{College of Computing \& Data Science, Nanyang Technological University}
	\streetaddress{50 Nanyang Ave}
	\city{Singapore}
	\country{Singapore}}
\email{asyong@ntu.edu.sg}


\renewcommand{\shortauthors}{Trovato et al.}

\begin{abstract}
  \textit{Novelty Search} (NS) promotes exploration by rewarding behaviorally novel solutions rather than directly optimizing predefined objectives, making it particularly useful when objective guidance is sparse or deceptive. However, existing NS methods are typically designed for a single search task. When multiple related NS tasks are considered independently, their search processes may repeatedly explore similar regions or rediscover solutions that could be useful across tasks, leading to redundant use of the evaluation budget. To address this issue, we formulate a multitask novelty search setting and propose \textit{Multitask Novelty Search with Cross-Task Collaborative Discovery} (MTNS-CoD). The central idea is to coordinate discovery across tasks so that different task-specific searches explore complementary regions while useful discoveries can still be reused across tasks. Specifically, MTNS-CoD introduces a \textit{multitask repulsion} mechanism to discourage redundant exploration in similar genotype regions, together with an \textit{adaptive inter-task transfer} mechanism that adjusts transfer probabilities according to the observed utility of cross-task exchanges during evolution. Their joint use promotes complementary exploration while selectively reusing beneficial discoveries. We further extend MTNS-CoD to novelty-augmented optimization, where behavioral novelty and objective information are jointly considered to support exploration under deceptive objective landscapes. Experiments on synthetic benchmarks, deceptive maze navigation, MuJoCo policy optimization, and generative novelty search show that MTNS-CoD can improve behavioral discovery and search coverage over the considered single-task and multitask baselines, with additional benefits observed on problems containing deceptive objectives.

\end{abstract}



\received{20 February 2007}
\received[revised]{12 March 2009}
\received[accepted]{5 June 2009}

\maketitle

\section{Introduction}
\textit{Novelty Search} (NS) was introduced as an alternative to objective-driven search for problems in which predefined fitness may provide sparse, misleading, or deceptive guidance~\cite{lehman2011abandoning}. Rather than selecting candidate solutions according to their immediate objective values, NS employs the behavioral novelty as the objective to guide the searching toward to unexplored regions with respect to previously discovered behaviors~\cite{lehman2011novelty}. By directing the search toward sparsely explored regions, NS maintains exploratory pressure and can facilitate the discovery of difficult-to-reach regions by leveraging intermediate behaviors as stepping stones~\cite{pugh2016quality}. This makes NS particularly suitable for problems in which objective-driven guidance may be sparse or deceptive, while diverse behavioral exploration can provide alternative pathways toward useful regions of the search space.

An important consideration in NS is to avoid redundant exploration of regions that have already been sufficiently explored and to allocate more search effort to new behavioral discoveries. Existing NS methods typically achieve this by using behavioral archives and novelty measures to bias the search away from previously encountered behaviors~\cite{lehman2011abandoning,kistemaker2011critical}. In practice, however, a system may need to perform NS under multiple different but related environments or operating conditions~\cite{gupta2017insights}. Although the resulting tasks may differ in their behavior mappings or reachable behaviors, some useful solutions may remain relevant across tasks. When these tasks are searched independently, discoveries obtained in one task cannot directly benefit the others, and separate search processes may therefore spend evaluations repeatedly exploring related regions and rediscovering similar solutions. Such redundant search consumes evaluation resources that could otherwise be allocated to discover additional novel behaviors. This motivates \textit{cross-task collaborative discovery} across multiple related NS tasks, where the search processes are coordinated to explore complementary regions while reusing discoveries that remain useful across tasks. Rather than allowing each task to independently explore possibly similar regions, cross-task collaborative discovery distributes search effort across tasks and shares beneficial discoveries when appropriate, thereby reducing redundant search and devoting more of the available budget to undiscovered behaviors.

In this paper, we propose \textit{Multitask Novelty Search with Cross-Task Collaborative Discovery} (MTNS-CoD) to jointly solve multiple related NS tasks within a single unified evolutionary search process. 
Specifically, MTNS-CoD introduces a \textit{multitask repulsion} mechanism that discourages different task-specific search processes from repeatedly concentrating on similar regions of the {genotype space} (i.e., decision space), thereby promoting complementary exploration across tasks. In parallel, an \textit{adaptive inter-task transfer} mechanism allows candidate solutions originating from one task to be probabilistically reused by others according to their observed contribution during evolution. Through the joint effect of repulsion and transfer, MTNS-CoD tends to reduce redundant search across related tasks and allocate a larger fraction of the evaluation budget to the discovery of more novel behaviors. We further extend the framework to \textit{novelty-augmented optimization}~\cite{conti2018improving}, where behavioral novelty and task objectives are jointly considered to support exploration under deceptive objective landscapes.
The main contributions of this paper are summarized as follows:
\begin{itemize}
    \item We propose MTNS-CoD for simultaneously solving multiple pure novelty search tasks through cross-task collaborative discovery. MTNS-CoD coordinates multitask repulsion and adaptive inter-task transfer to enable collaborative discovery across tasks. Specifically, multitask repulsion promotes complementary exploration to reduce redundant search, whereas adaptive transfer facilitates cross-task knowledge sharing to accelerate the discovery of novel solutions. 


    \item We further expand MTNS-CoD to multitask novelty-augmented optimization, which is more representative of real-world applications. By enabling a dynamic balance between novelty-driven exploration and objective-driven exploitation, the proposed approach helps the search process escape deceptive local optima and ultimately attain higher-quality solutions over complex objective landscapes.

    \item The proposed MTNS-CoD is extensively evaluated on four groups of problems, namely synthetic basin-type problems, deceptive maze navigation problems, MuJoCo policy optimization problems, and generative novelty search problems. The results show that, compared with single-task search, MTNS-CoD is able to discovery more novel solutions. Furthermore, in novelty-augmented optimization, MTNS-CoD demonstrates the capability on coping with deceptive landscapes arising from complex objectives. 
\end{itemize}

The remainder of this paper is organized as follows. Section II introduces the background of novelty search and novelty-augmented optimization, and reviews the related literature. Section III presents the motivation and idealized analysis of multitask novelty search. Section IV describes the overall framework of MTNS-CoD for pure novelty search. Section V expands MTNS-CoD to novelty-augmented optimization. Section VI reports the experimental results on four groups of problems. Finally, Section VII concludes this paper.

\section{Preliminaries}

\subsection{Background Concepts}
In this subsection, we provide the formal definitions of two distinct types of novelty search problems addressed in this study: \textit{pure novelty search} and \textit{novelty-augmented optimization}.

\subsubsection{Pure Novelty Search}
Consider a search problem defined over a genotype space, where each candidate solution is associated with a corresponding behavioral descriptor in a phenotype space. Pure novelty search prioritizes behavioral diversity by encouraging the exploration of previously unvisited regions of the phenotype space~\cite{lehman2011abandoning}. To formalize this setting, the key concepts and notations used throughout this work are summarized as follows.

\begin{itemize}
    \item \textbf{Genotype Space}: Let $\mathcal{X}$ denote the \textit{genotype space} (or decision space), which contains all possible solution encodings. For continuous problems, $\mathcal{X} \subseteq \mathbb{R}^D$, where $D$ is the dimensionality of the decision variables.

    \item \textbf{Phenotype Space}: Let $\mathcal{B}$ denote the \textit{phenotype space} (or behavior space), which represents the set of distinct behavioral descriptors exhibited by candidate solutions. Typically, $\mathcal{B} \subseteq \mathbb{R}^M$, where $M$ is the dimensionality of the behavior descriptor.

    \item \textbf{Behavior Mapping}: A behavior mapping function $\phi: \mathcal{X} \to \mathcal{B}$ is defined to map a genotype $x \in \mathcal{X}$ to its corresponding behavioral descriptor $b \in \mathcal{B}$, i.e., $b = \phi(x)$.

    \item \textbf{Novelty Score}: The novelty of a solution is quantified according to the sparsity of its local neighborhood in the phenotype space. Given an archive of previously encountered solutions, denoted by $\mathcal{A}$, the novelty score of a candidate solution $x$ is defined as
    \begin{equation}
    \rho(x,\mathcal{A}) = \frac{1}{K}\sum_{k=1}^{K} \mathrm{dist}(\phi(x), \mu_k),
    \end{equation}
    where $\mu_k$ is the $k$-th nearest neighbor of $\phi(x)$ in $\mathcal{A}$, and $\mathrm{dist}(\cdot,\cdot)$ is a distance metric in $\mathcal{B}$, such as the Euclidean distance.

    \item \textbf{Goal of Pure Novelty Search}: The goal of pure novelty search is to collect the novel solutions by maximizing the novelty score rather than the objective function value, i.e.,
    \begin{equation}
    \max_{x \in \mathcal{X}} : \; F_{\text{nov}}(x) = \rho(x,\mathcal{A}),
    \end{equation}
    so as to iteratively discover behaviorally novel solutions and expand the coverage of the phenotype space $\mathcal{B}$. It is worth noting that, as NS continuously discovers new novel solutions, $F_{\mathrm{nov}}$ changes dynamically with the growth of the archive $\mathcal{A}$. This dynamic objective landscape constitutes an important distinction between NS and conventional optimization problems with fixed objective functions.
\end{itemize}

\subsubsection{Novelty-Augmented Optimization}
Building upon the formulation of the pure novelty search, novelty-augmented optimization further incorporates objective performance into the search criterion~\cite{conti2018improving}, so as to balance behavioral exploration and objective exploitation. This formulation is motivated by the fact that purely objective-driven search may become trapped in deceptive local optima, whereas novelty can provide an additional exploratory pressure toward less-visited yet potentially promising regions. The additional key concepts in novelty-augmented optimization are summarized as follows:

\begin{itemize}
    \item \textbf{Objective Function}: Let $f: \mathcal{X} \to \mathbb{R}$ denote the objective function that evaluates the quality of a solution. Without loss of generality, we consider a maximization setting, where a larger value of $f(x)$ indicates better solution quality. It is worth noting, however, that in pure novelty search this objective is not considered.
    
    \item \textbf{Novelty-Augmented Objective}: The objective value and the novelty score are linearly combined to define the composite search criterion:
    \begin{equation}\label{eq:nov_aug_obj}
    F_{\text{nov-aug}}(x) = w \cdot f(x) + (1-w) \cdot \rho(x,\mathcal{A}),
    \end{equation}
    where $w \in [0,1]$ is a trade-off coefficient controlling the relative emphasis on the objective performance and the novelty.

    \item \textbf{Goal of Novelty-Augmented Optimization}: The goal of novelty-augmented optimization is to identify a solution that maximizes the above composite objective, i.e.,
    \begin{equation}\label{eq:nov_aug}
    \max_{x \in \mathcal{X}} : \; F_{\text{nov-aug}}(x).
    \end{equation}
    In \eqref{eq:nov_aug}, since $\rho(x,\mathcal{A})$ depends on the evolving archive $\mathcal{A}$, the effective search landscape induced by $F_{\text{nov-aug}}(x)$ is dynamic. As a result, by considering the novelty score, the search of novelty-augmented optimization is encouraged to move away from already well-explored regions and toward less-visited regions with high objective potential, thereby improving the ability to escape deceptive local optima.
\end{itemize}

\subsection{Related Work}

\subsubsection{Novelty Search via Evolutionary Algorithms}
Conventional evolutionary algorithms typically rely on a predefined objective function to guide the search toward solutions with improved objective values~\cite{branke2008multiobjective}. In contrast, NS rewards candidates according to their behavioral novelty relative to previously discovered solutions rather than their performance on a fixed objective~\cite{lehman2011abandoning,gomes2015devising,lehman2010efficiently}. Candidate solutions are mapped from the genotype space to a phenotype space, where novelty is evaluated based on behavior differences. As newly discovered behaviors are accumulated, the resulting novelty landscape evolves and continually redirects the search toward less explored regions. NS was originally motivated by coping with deceptive problems, where direct objective optimization may overlook intermediate behaviors that can serve as useful \textit{stepping stones} toward otherwise difficult-to-reach solutions~\cite{lehman2011abandoning,meyerson2017discovering}. More broadly, behavior diversity has also become relevant in AI settings where a repertoire of distinct behaviors can support adaptation to different environments or downstream tasks~\cite{rho2025language}. For example, unsupervised skill-discovery methods such as DIAYN~\cite{eysenbach2018diversity} learn behaviorally distinguishable policies without task-specific rewards, providing reusable skills for subsequent decision making. Although such methods differ from classical NS in formulation, they reflect a related motivation of broadening the set of behaviors available to an intelligent system, highlighting the relevance of novelty driven search as a mechanism for behavior discovery beyond its original role in deceptive optimization.

Although NS offers a useful alternative to purely objective driven optimization, completely discarding task objectives can be restrictive in practical applications. In particular, multiple candidate solutions may exhibit similar or even identical behavioral descriptors while differing substantially in task performance, which makes behavioral novelty alone insufficient for distinguishing their practical utility. This limitation has motivated approaches that jointly consider behavioral diversity and solution quality, including novelty-augmented optimization~\cite{conti2018improving} and quality diversity optimization~\cite{pugh2016quality,mouret2015illuminating,fontaine2020covariance}. Rather than relying exclusively on novelty, these methods combine behavioral exploration with objective-based evaluation to discover solutions that are both diverse and high-performing. In recent years, a broad range of \textit{novelty-augmented} and \textit{quality-diversity} methods have been developed and applied to domains such as reinforcement learning~\cite{faldor2023map,batra2024proximal}, robotics~\cite{cully2015robots,nilsson2021policy}, and design optimization~\cite{gaier2017data,galanos2021arch}. 

However, most existing studies formulate such search processes at the level of a single task. In many practical settings, multiple search tasks may arise under related environments or operating conditions, leading to related structures in their behavior mappings or objective landscapes. Such relationships may provide opportunities for one task to benefit from discoveries made in another, yet how to coordinate behavioral discovery across multiple related novelty-driven tasks remains comparatively underexplored. Existing work has explored connections between evolutionary multitasking and quality diversity~\cite{mouret2020quality}. However, this approaches primarily use multitasking mechanisms to facilitate the search within a single quality diversity problem, rather than jointly addressing multiple NS, novelty-augmented, or quality diversity tasks. To the best of our knowledge, no prior study has explicitly considered solving multiple NS tasks simultaneously.

This work primarily focuses on NS tasks. Although quality diversity methods may offer broader practical utility by jointly considering behavioral diversity and solution quality, they also introduce additional algorithmic and problem structural complexity beyond pure NS. We therefore employ NS as the main setting to study the idea of cross-task collaborative discovery and to examine its potential benefits in a comparatively controlled formulation. The extension of this perspective to more general quality diversity settings is left for future work.

\subsubsection{Evolutionary Multitasking}
Evolutionary multitasking has emerged as a powerful paradigm that exploits the implicit parallelism of evolutionary algorithms to address multiple optimization tasks simultaneously \cite{gupta2017insights,gupta2015multifactorial}. By enabling cross-task knowledge transfer, evolutionary multitasking can substantially improve the convergence and the solution quality \cite{bali2019multifactorial,feng2020explicit}. Existing approaches can generally be categorized into two main classes: \textit{implicit transfer} and \textit{explicit transfer}.

Implicit transfer frameworks, exemplified by the MFEA family of methods~\cite{gupta2015multifactorial,gupta2016multiobjective}, avoid the need for complex domain mapping by encoding different task spaces within a unified chromosomal representation. Under this paradigm, knowledge sharing is naturally induced through crossover across different tasks, whereby individuals associated with different \textit{skill factors} interact to implicitly exchange useful information. To further reduce the risk of negative transfer, advanced variants such as MFEA-II~\cite{bali2019multifactorial,bali2020cognizant} incorporate online inter-task synergy estimation. These methods employ probabilistic models or mixture distributions to adaptively assess task relatedness during evolution, thereby dynamically regulating the extent of genetic transfer across tasks~\cite{zhou2020toward,martinez2021adaptive,Li2025MTDE-MKTA}. 

In contrast, explicit transfer approaches are designed to construct direct mappings between heterogeneous problem domains, thus inject promising solutions from the source to the target task to guide the search~\cite{feng2020explicit}. Representative techniques include affine transformation~\cite{xue2020affine}, ensemble of transformations~\cite{lin2023ensemble}, optimal transport-based distribution alignment \cite{liu2025optimal}, as well as manifold alignment methods \cite{zhang2026multiobjective}. Although explicit transfer methods often incur higher computational overhead due to the mapping construction process, they provide more targeted and precise knowledge injection, which is particularly beneficial when handling tasks with heterogeneous dimensionalities or substantially different fitness landscapes.

Beyond these two mainstream categories, recent years have also witnessed increasing interest in \textit{surrogate-assisted transfer} approaches for expensive optimization problems~\cite{tan2024surrogate}. In such settings, surrogate models are introduced to reduce the number of costly fitness evaluations, while transfer mechanisms exploit data or models collected from multiple related tasks to construct more accurate multitask or transfer-aware surrogates for evolutionary search~\cite{min2017multiproblem,liu2024extremo,wang2020transfer}. These methods have demonstrated strong potential for improving search efficiency under limited evaluation budgets. Furthermore, such ideas have been extended to parametric optimization~\cite{wei2025theta,cheng2026parametric}, where infinitely many tasks are continuously defined over a parameter space and can be optimized in a unified manner~\cite{wei2025amortized}. In addition, \textit{inverse transfer}~\cite{liu2025bayesian} strategies have been proposed to learn from the nondominated solution sets of expensive multitask optimization problems, thereby accelerating the solution of expensive multiobjective optimization tasks~\cite{wei2024bayesian}.

Despite the extensive development of evolutionary multitasking, most existing studies focus on objective-driven optimization, where cross-task knowledge transfer is primarily designed to accelerate convergence or improve solution quality. Such approaches typically exploit similarities or shared structures among related tasks, for example, by encouraging their search processes to benefit from similar promising regions or aligned search directions. As a result, cross-task interaction is generally used to promote convergence toward compatible high-quality regions. Multitask NS follows a fundamentally different search logic. Rather than encouraging the search processes of different tasks to move toward similar regions, collaborative discovery deliberately maintains separation among them and uses repulsion to drive exploration toward complementary regions of the search space. This helps distribute the search effort over a broader range of regions and reduces redundant exploration, which is more consistent with the objective of discovering novel solutions. At the same time, useful discoveries can still be selectively reused across tasks through adaptive inter-task transfer. Motivated by this distinction, this work investigates multitask NS and develops a collaborative discovery framework that coordinates complementary exploration with selective knowledge sharing.

\section{Repulsion \& Adaptative Transfer: The Key of Cross-Task Collaborative Discovery}\label{sec3:analysis}

\subsection{Intuitive Explanation of Multitask Novelty Search}\label{sec:intuitive_analysis}
In evolutionary multitasking, cross-task transfer is employed to facilitate information sharing among related optimization tasks. In MFEA, for example, this is typically realized through inter-task crossover between individuals associated with different \textit{skill factors}, allowing genetic information to be implicitly exchanged across tasks. The effectiveness of such transfer generally relies on certain ``\textbf{\textit{consistent}}'' search characteristics between tasks, such as similar promising regions or aligned search directions~\cite{bai2021multitask,yin2019multifactorial}. This consistency can be reflected in \textit{simulated binary crossover} (SBX)~\cite{deb1995simulated} in MFEA settings, where offspring generated from parents associated with different tasks inherit information from both parent solutions. When the parental search information is compatible across tasks, such recombination may help exploit their shared search structures during evolution.

While exploiting such inter-task consistency can be beneficial for convergence-oriented optimization, NS places greater emphasis on the sustained discovery of behaviorally novel solutions. Rather than concentrating search effort around similar promising regions, NS continually encourages exploration toward previously underexplored behaviors. Therefore, in multitask NS, cross-task interaction may not merely exploit shared search structures, but also coordinate task-specific exploration so that the tasks collectively discover a broader range of behaviors with less redundant search. This motivates the proposed notion of ``\textbf{\textit{cross-task collaborative discovery}}", in which multiple tasks coordinate \textit{where to explore} while selectively sharing \textit{what they discover}. 
We consider \textit{repulsion} and \textit{adaptive inter-task transfer} as two key mechanisms for realizing cross-task collaborative discovery.
Specifically, \textit{repulsion} promotes a division of exploration by steering different tasks toward complementary regions, whereas \textit{adaptive inter-task transfer} allows discoveries made by one task to be reused by another when beneficial. Together, these mechanisms coordinate discovery across tasks by promoting complementary exploration and selective reuse of useful discoveries, thereby reducing redundant search. In the following, we provide an intuitive explanation of how these two mechanisms jointly realize cross-task collaborative discovery.

To provide a clear and rigorous formalization of the potential benefits of repulsion, we consider the setting of \textit{multitask pure NS}. Without loss of generality, we focus on the two-task case, consisting of Task~1 (i.e., $\mathcal{T}_1$) and Task~2 (i.e., $\mathcal{T}_2$). Let $\mathcal{X}$ denote the \textit{unified} genotype space. For each task $\mathcal{T}_i$ ($i \in \{1,2\}$), let $\phi_i:\mathcal{X} \to \mathcal{B}_i$ denote the mapping from the unified genotype space to the corresponding behavior space, where ${\mathcal{B}}_i$ denotes the overall behavior space associated with task $\mathcal{T}_i$. Let $\tilde{\mathcal{B}}_i \subseteq {\mathcal{B}}_i$ denote the set of reachable behaviors for task $\mathcal{T}_i$. We then define the \textit{target genotype set} $\Omega_i \subseteq \mathcal{X}$ as the set of all genotypes whose induced behaviors are reachable under task $\mathcal{T}_i$, i.e.,
\begin{equation}
    \Omega_i=\left\{x \mid \phi_i(x)\in \tilde{\mathcal{B}}_i\right\}.
\end{equation}
From the perspective of NS, the goal is to obtain a representative subset of $\Omega_i$ whose image under $\phi_i$ sufficiently covers the reachable behavior space of interest. According to the geometric relationship between $\Omega_1$ and $\Omega_2$ in the genotype space, the resulting \textit{multitask pure NS} problem can be categorized into the following three representative cases.

\begin{figure*}[!t]
	\begin{center}
        \subfigure[]{\label{fig:scenario_1}\includegraphics[width=0.32\columnwidth]{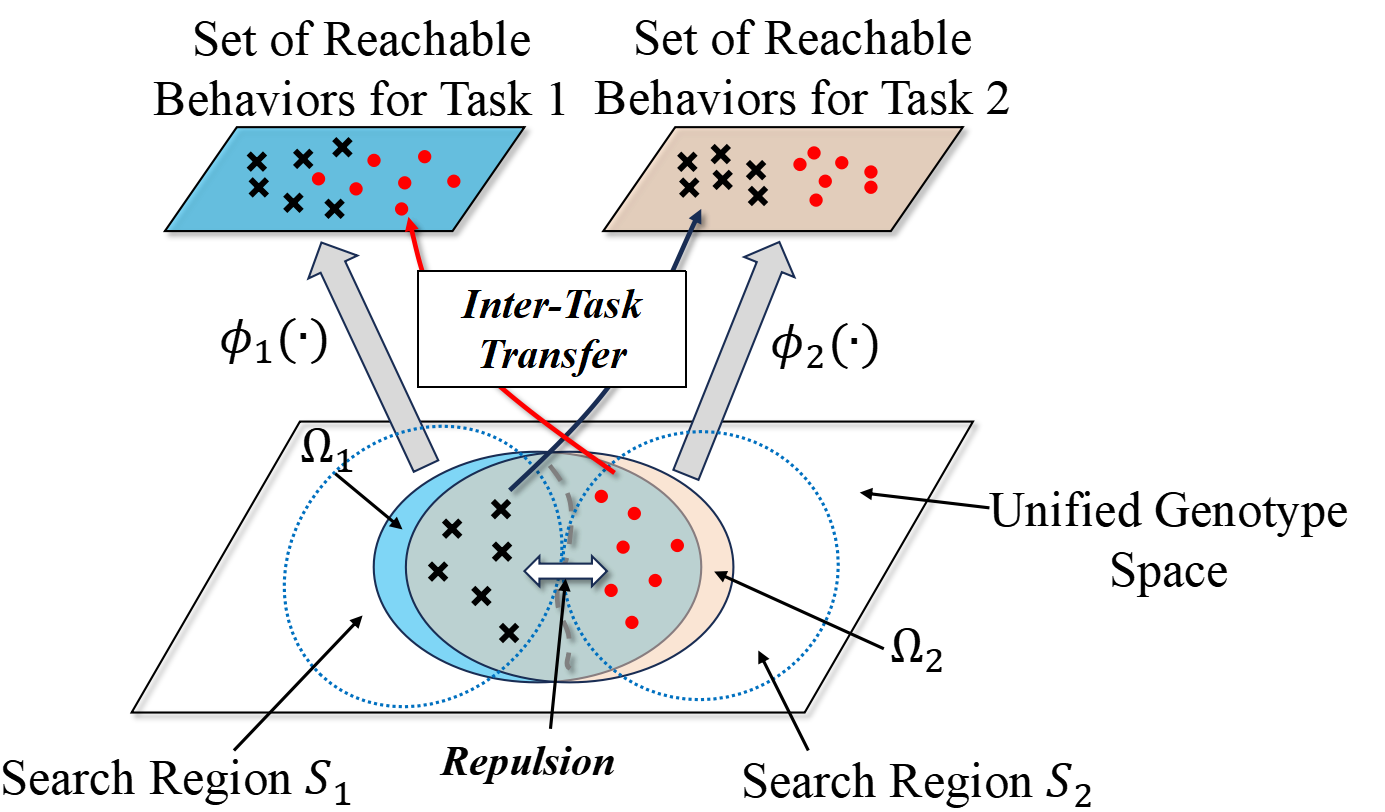}}
		\subfigure[]{\label{fig:scenario_2}\includegraphics[width=0.32\columnwidth]{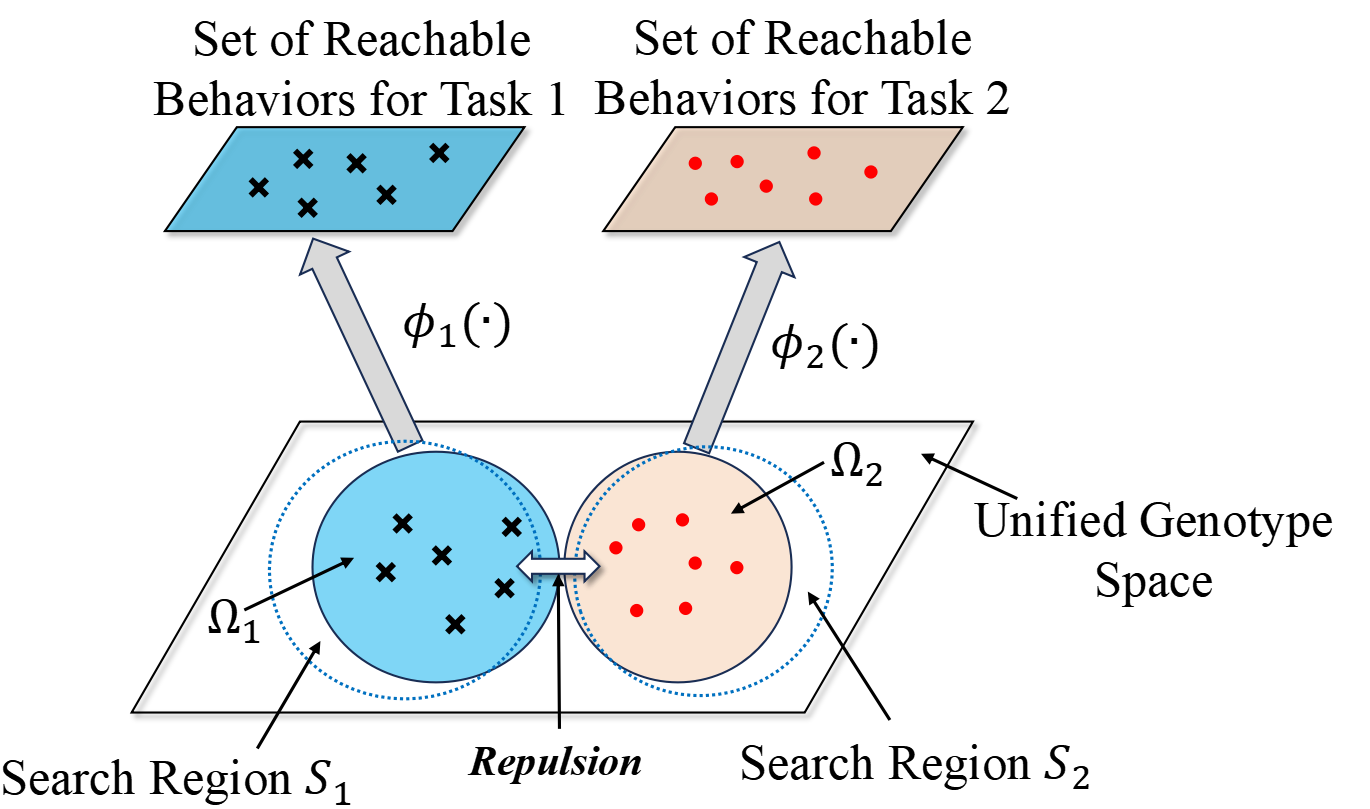}}
        \subfigure[]{\label{fig:scenario_3}\includegraphics[width=0.32\columnwidth]{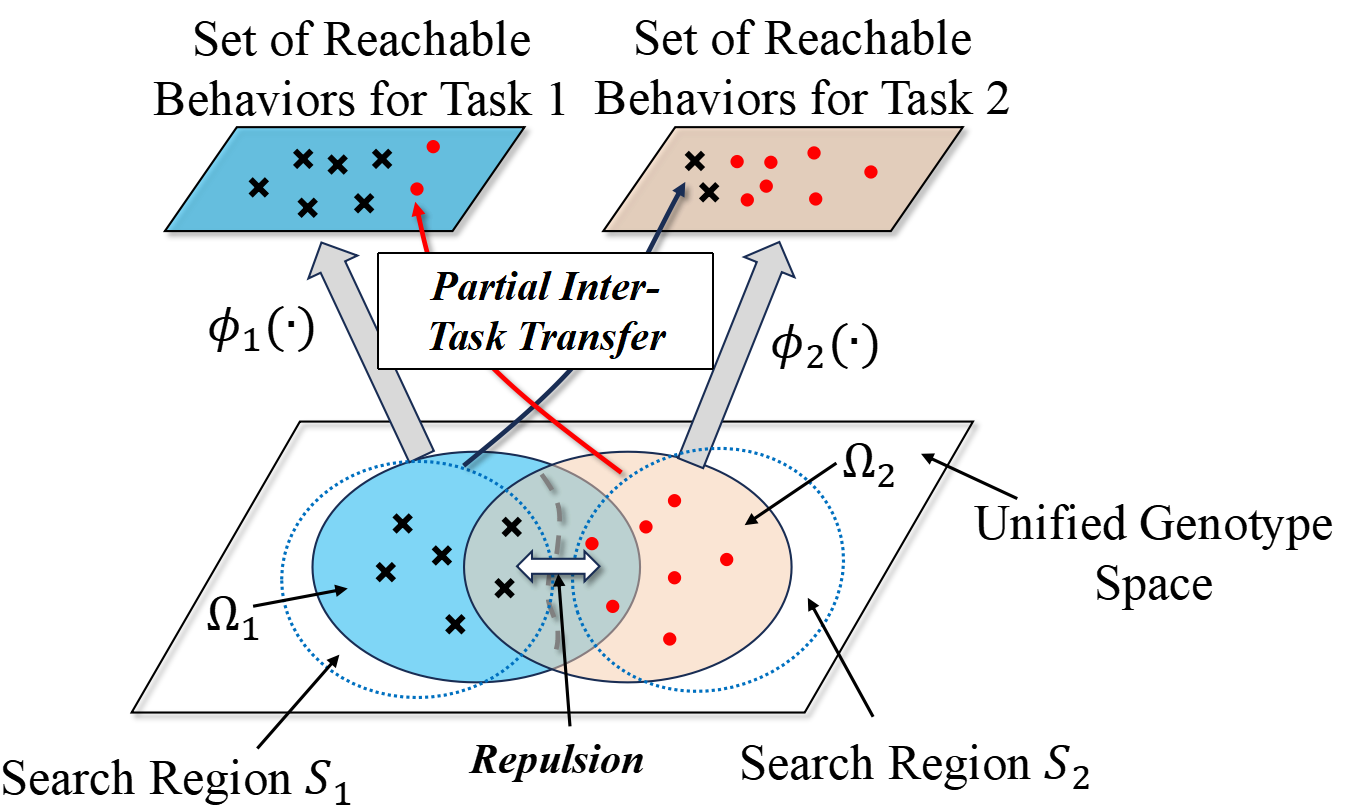}}
        \caption{
        Illustration of multitask novelty search under three representative overlap scenarios of the target genotype sets. (a) \textbf{Scenario I: Overlapping Target Genotype Sets.} Repulsion promotes complementary exploration of the shared search space, while inter-task transfer facilitates efficient sharing of discovered novel solutions. (b) \textbf{Scenario II: Disjoint Target Genotype Sets.} Repulsion provides an inductive bias toward non-overlapping search regions, helping each task rapidly focus on its own target genotype set. (c) \textbf{Scenario III: Partially Overlapping Target Genotype Sets.} Repulsion preserves complementary exploration, while incomplete inter-task transfer enables partial sharing of novel solutions within the overlapping region.
        }
        \label{fig:explanation}
        \vspace{-12pt}
	\end{center}
\end{figure*}

\begin{itemize}
\item \textbf{Scenario I: Overlapping Target Genotype Sets.}
In this scenario, we assume that $\Omega_1$ and $\Omega_2$ substantially overlap and occupy nearly the same region in the unified genotype space. If both tasks independently explore this shared region, their search efforts may become redundant. A natural strategy is therefore to adopt a \textit{divide-and-conquer} search scheme, in which the two tasks are encouraged to explore complementary portions of the overlapping target region. Specifically, let $\Omega^{\#}=\bigcup_{i=1,2}\Omega_i$ denote the combined target region, and suppose that $\Omega^{\#}\subseteq \mathcal{S}_1\cup \mathcal{S}_2$, where $\mathcal{S}_1$ and $\mathcal{S}_2$ are disjoint search regions satisfying $\mathcal{S}_1\cap \mathcal{S}_2=\emptyset$ and are implicitly associated with $\mathcal{T}_1$ and $\mathcal{T}_2$, respectively. Under this partition, each task concentrates its exploration on a different portion of the shared target region, thereby distributing the search effort and reducing redundant exploration. Intuitively, exploring a smaller region $\mathcal{S}_i$ can require less effort than having each task independently search and cover the entire target set $\Omega_i$. The discoveries obtained from these complementary regions can subsequently be shared through \textit{inter-task transfer}, allowing each task to benefit from solutions discovered by the other. \textit{Hence, cross-task collaborative discovery in this scenario involves two complementary roles: 1) maintaining sufficient separation between the task-specific search processes during exploration, and 2) exchanging useful discoveries across tasks. The former is supported by the repulsion mechanism, which encourages the two tasks to explore complementary regions and thereby reduces redundant search.} A schematic illustration of this scenario is provided in Fig.~\ref{fig:scenario_1}.


\item \textbf{Scenario II: Disjoint Target Genotype Sets.} In this scenario, the target genotype sets $\Omega_1$ and $\Omega_2$ are assumed to be disjoint, i.e., they occupy separate regions of the unified genotype space. Under this condition, repulsion is well aligned with the problem structure of multitask NS. Since valid solutions for different tasks are distributed in distinct regions, separating the search region of the tasks discourages aggregation and promotes exploration over mutually exclusive subspaces. As a result, each task can remain focused on its own target genotype set, thereby reducing inter-task interference. \textit{From this perspective, repulsion acts as an inductive bias that encourages task-specific exploration while preserving the parallelism of multitask search.} A schematic illustration of this scenario is provided in Fig.~\ref{fig:scenario_2}.


\item \textbf{Scenario III: Partially Overlapping Target Genotype Sets.}
This is the most general case, where $\Omega_1$ and $\Omega_2$ partially overlap while each still contains task-specific regions; Scenarios~I and~II can therefore be viewed as two limiting cases. In this setting, repulsion serves as a persistent mechanism for maintaining complementary search regions across tasks, whereas \textit{partial} inter-task transfer selectively reuses discoveries that are likely to remain useful for the target task. Here, ``\textit{partial}'' means that discoveries are not indiscriminately transferred across tasks; instead, only a subset of cross-task discoveries is reused, while ineffective transfers are suppressed. Intuitively, a larger overlap between the target genotype sets provides more opportunity for useful cross-task reuse, whereas a smaller overlap calls for more conservative transfer to avoid ineffective information exchange. \textit{Therefore, efficient multitask novelty search in Scenario~III is characterized by persistent repulsion together with partial and adaptive inter-task transfer.} A schematic illustration is given in Fig.~\ref{fig:scenario_3}.

\end{itemize}

The above analysis suggests that effective methods for multitask NS requires the joint action of two mechanisms: (i) \textbf{repulsion}, which promotes the searching over distinct regions, and (ii) \textbf{adaptative inter-task transfer}, which exploits the shared region when beneficial. 

\subsection{Idealized Analysis of Multitask Novelty Search}
\label{sec:theoretical_analysis}

We next provide an idealized analysis of the potential gain of cross-task collaborative discovery in multitask NS. Rather than attempting to reproduce the complete dynamics of an evolutionary search process, the analysis abstractly isolates two forms of redundant computation that can arise when related NS tasks are solved independently. First, when the target genotype sets overlap, different tasks may repeatedly discover novelty units corresponding to the same shared region. Second, even when the target genotype sets are disjoint, the search regions explored by different tasks may still overlap, causing evaluations to be spent in regions that are not required by the corresponding task. These two sources of redundancy are respectively addressed by inter-task reuse and complementary exploration.

\subsubsection{Common Analysis Setting and Assumptions.}
Following the definitions in Section~\ref{sec:intuitive_analysis}, we consider two pure NS tasks $\mathcal{T}_1$ and $\mathcal{T}_2$ defined on the common genotype space $\mathcal{X}$, with task-specific target genotype sets $\Omega_1$ and $\Omega_2$. For the theoretical analysis, we further use $\mathcal{S}_i\subseteq\mathcal{X}$ to denote the search region directly explored by Task~$i$. Unlike $\Omega_i$, which characterizes the search region associated with the reachable behaviors of Task~$i$, $\mathcal{S}_i$ describes how the available search effort is allocated. 

To obtain a finite abstraction of novelty coverage, we consider a fixed behavior-space resolution $r>0$. Let
$\psi_i^{(r)}(\phi_i(x))$ denote the resolvable novelty contribution induced by genotype $x$ for Task~$i$ at resolution $r$. We define
\begin{equation}
x\sim_i y
\quad\Longleftrightarrow\quad
\psi_i^{(r)}(\phi_i(x))
=
\psi_i^{(r)}(\phi_i(y)),
\end{equation}
so that two genotypes are regarded as belonging to the same \textit{novelty unit} if they induce the same resolvable behavioral contribution at resolution $r$. The equivalence classes of $\sim_i$ therefore form a finite set of phenotype-resolvable novelty units for Task~$i$, and discovering all novelty units associated with $\Omega_i$ is regarded as attaining its target novelty coverage at this resolution.
For the analysis of search effort, we additionally discretize each directly explored search region $\mathcal{S}_i$ into a finite number of \textit{search units}. At the genotype-space resolution used for the analysis, each effective search region $\mathcal{S}_i$ is represented by a finite number of \textit{search units}. We adopt the following simplifying assumptions.
\begin{enumerate}
    \item \textbf{Coupon-Collector Search Abstraction.}
    Each evaluation independently samples one search unit uniformly from the effective search region of the corresponding task. Suppose a search region contains $M$ search units, among which $N$ distinct units correspond to novelty units that must be discovered. If $k$ of these $N$ novelty units have already been discovered, the probability that the next evaluation discovers a new required unit is $(N-k)/M$. Therefore, the expected number of computational efforts required to discover all $N$ required novelty units is
    \begin{equation}
        M H_N,
        \qquad
        H_N=\sum_{k=1}^{N}\frac{1}{k}.
    \end{equation}
    We assume $M\geq N$ throughout the analysis.

    \item \textbf{Direct Discovery Cost.}
    Directly discovering a new novelty unit is assigned unit relative cost. This cost abstracts the computational effort required to obtain the unit through \textit{multiple search iterations}, where each iteration may involve candidate generation, task evaluation, behavior characterization, and archive update.

    \item \textbf{Transfer Cost.}
    Once a solution corresponding to a novelty unit has been discovered by one task, transferring this solution to another task does not require repeating the iterative search process. The receiving task may still need to evaluate the transferred solution, compute its behavior, and update its archive. We therefore denote the relative computational cost of one transfer by $c_{\mathrm{tr}}$ unit and assume
    \begin{equation}
        0< c_{\mathrm{tr}}\leq1.
    \end{equation}
    Thus, transfer is assumed to be cheaper than independently rediscovering a corresponding novelty unit from scratch.

    \item \textbf{Ideal Complementary Search.}
    Repulsion is modeled as an idealized mechanism that partitions an overlapping portion of the effective search regions between the two tasks without removing novelty units required by either task. This assumption represents the best-case search specialization that complementary search seeks to approximate.
\end{enumerate}

Under the assumption with the above abstractions, we next examine two limiting cases corresponding to overlapping and disjoint target genotype sets.

\subsubsection{Overlapping Target Genotype Sets.}
We first consider the case where the two tasks share part of their target genotype regions. In this setting, independent search may repeatedly explore overlapping search regions and rediscover novelty units that are useful to both tasks. Cross-task collaborative discovery can reduce these two sources of redundancy by promoting complementary exploration and reusing shared discoveries through lower-cost inter-task transfer.

\begin{proposition}
\label{thm:multitask_novelty_overlap}
Suppose $O:=\Omega_1\cap\Omega_2\neq\emptyset$, and define $E_1:=\Omega_1\setminus\Omega_2, E_2:=\Omega_2\setminus\Omega_1$. Assume that the two tasks induce the same novelty partition on the shared target region, i.e.,
\begin{equation}
\forall x,y\in O,\qquad
x\sim_1 y
\iff
x\sim_2 y.
\end{equation}
Let $\sim_O$ denote the common equivalence relation on $O$, and let
\begin{equation}
    \Pi_O:=O/{\sim_O},
    \qquad
    c:=|\Pi_O|>0,
\end{equation}
where $c$ is the number of novelty units shared by the two tasks. For Task~$i$, let $a_i$ denote the number of novelty units contained in its exclusive target region $E_i$. We assume that the novelty partition is consistent with the decomposition of $\Omega_i$ into $E_i$ and $O$, so that each required novelty unit belongs entirely to either the exclusive region or the shared region. Therefore, Task~$i$ requires $a_i+c$ novelty units in total.
Let
$\mathcal{S}_1\setminus\mathcal{S}_2$ and
$\mathcal{S}_2\setminus\mathcal{S}_1$
contain $b_1$ and $b_2$ search units, respectively, and let
$\mathcal{S}_1\cap\mathcal{S}_2$
contain $d>0$ search units. Under independent search,
\begin{equation}
\mathbb{E}[T_{\mathrm{ind}}]
=
(b_1+d)H_{a_1+c}
+
(b_2+d)H_{a_2+c}.
\end{equation}
Under ideal collaborative discovery, suppose the overlapping search units are partitioned into two disjoint subsets of sizes $d_1$ and $d_2$, with
\begin{equation}
    d_1+d_2=d,
\end{equation}
and the shared novelty units are partitioned as
\begin{equation}
    \Pi_O^{(1)}\cap\Pi_O^{(2)}=\emptyset,
    \qquad
    \Pi_O^{(1)}\cup\Pi_O^{(2)}=\Pi_O,
\end{equation}
where
\begin{equation}
    c_i:=|\Pi_O^{(i)}|,
    \qquad
    c_1+c_2=c.
\end{equation}
Assume that all novelty units assigned to Task~$i$ remain contained in its corresponding reduced search region.
Task~$i$ then directly discovers its $a_i+c_i$ assigned novelty units from a search region containing $b_i+d_i$ units, while the remaining shared novelty units are obtained through cross-task transfer. Therefore,
\begin{equation}
\mathbb{E}[T_{\mathrm{mt}}]
=
(b_1+d_1)H_{a_1+c_1}
+
(b_2+d_2)H_{a_2+c_2}
+
c\,c_{\mathrm{tr}}.
\end{equation}
If $0 < c_{\mathrm{tr}} \leq 1$, then
$\mathbb{E}[T_{\mathrm{mt}}] \leq \mathbb{E}[T_{\mathrm{ind}}]$.
\end{proposition}

The proof of \textbf{Proposition~\ref{thm:multitask_novelty_overlap}} can be found in Appendix A.1. \textbf{Proposition~\ref{thm:multitask_novelty_overlap}} characterizes an idealized sufficient condition under which cross-task collaborative discovery can reduce computational cost when the tasks share novelty units. The gain has two sources. First, complementary exploration reduces the portion of the overlapping search region that must be explored independently by both tasks. Second, each shared novelty unit is directly discovered only once and subsequently reused by the other task. When $c_{\mathrm{tr}}\leq1$, the cost of such reuse is lower than that of independently rediscovering the same novelty unit through iterative search. 

\subsubsection{Disjoint Target Genotype Sets.}
We next consider the opposite limiting case where the target genotype sets are disjoint. In this setting, cross-task reuse provides little direct benefit because the two tasks require different novelty units. However, their search regions may still overlap. Cross-task collaborative discovery can therefore improve efficiency primarily by using repulsion to reduce unnecessary overlap between the two search processes, while adaptive transfer suppresses ineffective cross-task exchanges.

\begin{proposition}
\label{thm:multitask_novelty_disjoint}
Suppose $\Omega_1\cap\Omega_2=\emptyset.$
Let $\Omega_i$ contain $a_i>0$ required novelty units. Let
$\mathcal{S}_1\setminus\mathcal{S}_2$ and
$\mathcal{S}_2\setminus\mathcal{S}_1$
contain $b_1$ and $b_2$ search units, respectively, and let
$\mathcal{S}_1\cap\mathcal{S}_2$
contain $d \geq 0$ search units.
Under independent search, Task~$i$ discovers its $a_i$ required novelty units from a search region containing $b_i+d$ units. Therefore,
\begin{equation}
\mathbb{E}[T_{\mathrm{ind}}]
=
(b_1+d)H_{a_1}
+
(b_2+d)H_{a_2}.
\end{equation}
Under ideal collaborative discovery, suppose repulsion partitions the $d$ overlapping search units into two disjoint subsets of sizes $d_1$ and $d_2$, with
\begin{equation}
    d_1+d_2=d,
\end{equation}
and assume that all novelty units required by Task~$i$ remain contained in its corresponding reduced search region. Task~$i$ then discovers its $a_i$ novelty units from a region containing $b_i+d_i$ search units.
If ineffective cross-task transfer is fully suppressed in this limiting case, the expected total computational cost becomes
\begin{equation}
\mathbb{E}[T_{\mathrm{mt}}]
=
(b_1+d_1)H_{a_1}
+
(b_2+d_2)H_{a_2}.
\end{equation}
Hence, we have $\mathbb{E}[T_{\mathrm{mt}}] \leq \mathbb{E}[T_{\mathrm{ind}}].$
\end{proposition}

The proof of \textbf{Proposition~\ref{thm:multitask_novelty_disjoint}} can be found in Appendix A.1. \textbf{Proposition~\ref{thm:multitask_novelty_disjoint}} describes a different source of efficiency from \textbf{Proposition~\ref{thm:multitask_novelty_overlap}}. Because no shared novelty units exist, the gain does not arise from transferring discoveries between tasks. Instead, it results from reducing the amount of possibly overlapping search that may be performed independently by both tasks. Under the idealized specialization represented by the propositions, repulsion allows each task to search a smaller effective region while preserving all novelty units required for its own coverage.

Taken together, \textbf{Propositions~\ref{thm:multitask_novelty_overlap}} and \textbf{\ref{thm:multitask_novelty_disjoint}} characterize the two complementary mechanisms underlying cross-task collaborative discovery. When the target genotype sets overlap, discovery can be improved by distributing the exploration of shared novelty units across tasks and replacing redundant rediscovery with lower-cost transfer. When the target genotype sets are disjoint, the main benefit arises from specializing the task-specific search regions and reducing unnecessary exploration overlap. The partially overlapping case naturally combines these two effects: repulsion promotes complementary allocation of search effort, while adaptive inter-task transfer determines which discoveries are sufficiently useful to justify cross-task reuse. Although these results rely on idealized coupon-collector abstractions, they formalize the central motivation of MTNS-CoD and provide guidance for the design of its repulsion and adaptive-transfer mechanisms.

\section{MTNS-CoD for Pure Novelty Search}
In this section, the technical components of MTNS-CoD on the pure novelty search are presented in detail. We first outline the overall framework to provide a holistic view, and then delve into the designed \textit{multitask repulsive operator} and the \textit{inter-task transfer probability adaption}.

\subsection{Overall Framework of the MTNS-CoD}
{
\setlength{\textfloatsep}{10pt}
\begin{algorithm}[!t]
    \caption{MTNS-CoD}
    \label{Alg:MFEA-Repulsion}
    \begin{algorithmic}[1]
        \REQUIRE $I$: Number of tasks $\{\mathcal{T}_i\}_{i=1}^I$; $K$: Number of emitters per task; $N$: Batch size; $t_{\max}$: Max generations; $\{\mathcal{A}_i\}_{i=1}^I$: Initial archives.
        \ENSURE Updated archives $\{\mathcal{A}_i\}_{i=1}^I$.
        
        \STATE Initialize the inter-task transfer probability matrix $\textit{ITP} \in [0, 1]^{I \times I}$;
        
        \FOR{$i = 1$ \TO $I$}
            \STATE Extract emitter set $\mathcal{E}_i$ such that $|\mathcal{E}_i| = K$;
            \STATE Assign source skill factor $\tau(e) = i, \forall e \in \mathcal{E}_i$;
        \ENDFOR
        \STATE Assemble unified emitter pool $\mathcal{E} = \bigcup_{i=1}^{I} \mathcal{E}_i$;
        
        \FOR{$t = 1$ \TO $t_{\max}$}
            \FOR{$l = 1$ \TO $ I \times K $}
                \STATE Generate $N$ temporary candidates $\{\hat{x}_n\}_{n=1}^{N}$ based on $e_{l}$;
                \STATE Set temporary skill factor $\tau'_n = \tau(e_l)$;
                \STATE Determine target skill factor $\tau_n$ for $\hat{x}_n$ by sampling from the categorical distribution defined by the $\tau'_n$-th row of $\textit{ITP}$;
                \STATE $x_n \leftarrow \textit{MultitaskRepulsion}(\hat{x}_n)$ \hfill $\triangleright$ See \textbf{Algorithm 2}
                \STATE Evaluate $x_n$ on task $\mathcal{T}_{\tau_n}$ to obtain the behavior descriptor $b_{\tau_n}(x_n)$;
                \STATE Collect information: $\mathcal{P}_l = \{ (x_n, \tau_n, b_{\tau_n}(x_n) ) \}_{n=1}^{N}$;
            \ENDFOR
            \STATE Update archives $\{\mathcal{A}_i\}_{i=1}^I$ with $\mathcal{P}$;
            \STATE Calculate scalar fitness for $\mathcal{P} = \bigcup_{l=1}^{I \times K} \mathcal{P}_l$;
            \STATE Update $\mathcal{E}$ according to the scalar fitness;
            \STATE Update $\textit{ITP}$ via \textbf{Algorithm 3};
        \ENDFOR
    \end{algorithmic}
\end{algorithm}
}

To provide an overview of the proposed MTNS-CoD, we first present its procedural framework. Consistent with divergent search paradigms~\cite{tjanaka2023pyribs}, the proposed framework adopts \emph{emitters} as the basic search units. This description makes MTNS-CoD readily compatible with a broad class of emitter-based illumination and novelty search algorithms. The overall procedure is summarized in \textbf{Algorithm~1}\footnote{Additional implementation details beyond the pseudocode are provided in Appendix~\ref{app:implementation}.}, and its key components are described below:
\begin{itemize}
    \item \textbf{Initialization}: The procedure begins by initializing the inter-task transfer probability matrix and the emitter set. Suppose that $I$ tasks are optimized concurrently. An inter-task transfer probability matrix $\textit{ITP}\in[0,1]^{I\times I}$ is first constructed. Subsequently, for each task $\mathcal{T}_i$, $K$ emitters are initialized, yielding a unified emitter pool $\mathcal{E}=\{e_1,e_2,\dots,e_{I\times K}\}$. Each emitter $e\in\mathcal{E}$ is associated with a skill factor $\tau(e)\in\{1,\dots,I\}$, which specifies the task domain to which it belongs.

    \item \textbf{Offspring Generation}: In each generation, every emitter $e \in \mathcal{E}$ produces $N$ temporary candidate solutions, denoted by $\{\hat{x}_n\}_{n=1}^{N}$, using its emitter-specific reproduction operator. Each candidate solution is assigned a temporary skill factor $\tau'_n=\tau(e)$, indicating its originating task. This process is repeated until all emitters have generated their corresponding candidate solutions.
    
    \item \textbf{Inter-Task Transfer and Repulsion}: This stage performs inter-task transfer followed by repulsive refinement.
        \begin{itemize}
            \item \textit{Probabilistic Transfer}: To allow an emitter specialized for task $i$ to contribute to other tasks, each temporary candidate $\hat{x}_n$ may be reassigned a target skill factor $\tau_n \neq i$. This reassignment is sampled from a categorical distribution parameterized by the $\tau'_n$-th row of the \textit{ITP}.
        
            \item \textit{Repulsive Refinement}: After transfer, a \textit{multitask repulsive operator} (see \textbf{Algorithm~2} in Section~\ref{sec:repulsion}) is applied to $\hat{x}_n$. This operator adjusts the candidate in order to reduce redundant search behavior, thereby producing the final offspring $x_n$.
        \end{itemize}
    
    \item \textbf{Evaluation and Update}: Each offspring $x_n$ is evaluated exclusively on task $\mathcal{T}_{\tau_n}$ to derive its behavioral descriptor $b_{\tau_n}(x_n)$. These descriptors are then utilized to update the task-specific archives $\{\mathcal{A}_i\}_{i=1}^I$. Similar with the MFEA paradigm~\cite{gupta2015multifactorial}, a scalar fitness is computed for each individual based on its performance relative to its assigned skill factor, which subsequently guides the evolution of emitters in $\mathcal{E}$.

    \item \textbf{Inter-Task Transfer Probability Adaptation}: The \textit{ITP} matrix is updated at the end of each generation. This is achieved via the adaptive strategy detailed in \textbf{Algorithm 3} (Section~\ref{sec:itp_adaptation}), which leverages the success rates of historical transfers to bias future inter-task interactions.
\end{itemize}

The evolution process described in Steps~7--20 is repeated until a predefined termination criterion is met. Upon termination, the algorithm returns the task-specific novelty archives $\{\mathcal{A}_i\}_{i=1}^I$, which collectively characterize the explored phenotype spaces. Each archive stores a diverse set of solutions together with their corresponding genotypes and behavior descriptors.

\subsection{Multitask Repulsive Operator}\label{sec:repulsion}

{
\setlength{\textfloatsep}{10pt}
\begin{algorithm}[!t]
    \caption{Multitask Repulsive Operator}
    \label{Alg:repulsion}
    \begin{algorithmic}[1]
        \REQUIRE Candidate solution $\hat{\textbf{x}}$, archive set $\{\mathcal{A}_{i}\}_{i=1}^{I}$, parameter $\eta$ and $N_{i}^{\text{tem}}$.
        \ENSURE Repelled solution ${\textbf{x}}$.
        \FOR {$i = 1 \textbf{ to } I$}
        \STATE Construct the temporary archive $\bar{\mathcal{A}}_i$ based on the recent $N_{i}^{\text{tem}}$ solutions added in $\mathcal{A}_{i}$.
        \STATE $\boldsymbol{\mu}_{i} \gets \text{mean}(\bar{\mathcal{A}}_{i})$;
        \STATE $\boldsymbol{\sigma}_{i} \gets \text{std}(\bar{\mathcal{A}}_{i})$;
        \STATE $\boldsymbol{\delta}_{i} \gets (\hat{\textbf{x}} - \boldsymbol{\mu}_{i}) / \boldsymbol{\sigma}_{i}$;
        \STATE $\mathbf{v}_{i} \gets \eta \cdot e^{-\frac{1}{2}\boldsymbol{\delta}_{i}^2} \odot \text{sgn}(\boldsymbol{\delta}_{i}) \odot \boldsymbol{\sigma}_{i}$;
        \ENDFOR
        \STATE $\textbf{x} = \text{Clip}(\hat{\textbf{x}} + \sum_{i=1}^{I} \mathbf{v}_{i}, \textbf{LB}, \textbf{UB})$;
    \end{algorithmic}
\end{algorithm}
}

In this subsection, we introduce the \textit{multitask repulsive operator} for multitask NS. The main idea is to treat the solution distributions induced by the task-specific archives as repulsive sources, such that newly generated offspring are encouraged to move toward under-explored regions of the unified genotype space.

For each task $\mathcal{T}_i$, a temporary archive $\bar{\mathcal{A}}_i$ is constructed from the most recently collected $N_i^{\text{tem}}$ solutions. The spatial distribution of $\bar{\mathcal{A}}_i$ is then approximated by a multivariate Gaussian distribution with diagonal covariance. Although a single Gaussian cannot capture complex or multimodal distributions over the entire decision space, such a global approximation is not required here. The emitter distribution evolves continuously during the search, and the temporary archive only records the regions explored in the most recent search iterations. Accordingly, the Gaussian distribution is used only as a local approximation of the recently explored region, providing a lightweight estimate of where the current search effort has been concentrated. This local density estimate is sufficient for the repulsion operator to discourage other task-specific searches from repeatedly exploring the same recently visited regions. The diagonal covariance further reduces the computational cost by avoiding the estimation of full cross-dimensional dependencies. Accordingly, for task $\mathcal{T}_i$, we define the centroid $\boldsymbol{\mu}_i=[\mu_{i,1},\dots,\mu_{i,D}]^{T}$ and the variance vector $\boldsymbol{\sigma}_i^2=[\sigma_{i,1}^2,\dots,\sigma_{i,D}^2]^{T}$, where
\[
\mu_{i,d}=\frac{1}{N^{\text{tem}}_i}\sum_{\mathbf{x}\in \bar{\mathcal{A}}_i} x_d, \
\sigma_{i,d}^2=\frac{1}{N^{\text{tem}}_i-1}\sum_{\mathbf{x}\in \bar{\mathcal{A}}_i}(x_d-\mu_{i,d})^2,
\]
for each dimension $d\in\{1,\dots,D\}$. The resulting probability density function associated with task $\mathcal{T}_i$ is given by
\begin{equation}
P_i(\bar{\mathbf{x}})
=
\prod_{d=1}^{D}
\frac{1}{\sqrt{2\pi \sigma_{i,d}^2}}
\exp\!\left(
-\frac{(\bar{x}_d-\mu_{i,d})^2}{2\sigma_{i,d}^2}
\right).
\end{equation}
This model characterizes the occupancy of task $\mathcal{T}_i$ in the decision space. Based on it, the proposed repulsive operator shifts a candidate solution away from regions with high archive density. These statistics can be maintained incrementally as the temporary archive is updated, without recomputing them from all $N_i^{\mathrm{tem}}$ solutions. This makes the density approximation computationally lightweight in practice.

Given a candidate solution $\hat{\mathbf{x}}$ and the temporary archives $\{\bar{\mathcal{A}}_i\}_{i=1}^{I}$, we first define the normalized distance vector $\boldsymbol{\delta}_i$ for task $\mathcal{T}_i$ as
\begin{equation}
\delta_{i,d}=\frac{\hat{x}_d-\mu_{i,d}}{\sigma_{i,d}},
\qquad d\in\{1,\dots,D\}.
\end{equation}
The corresponding repulsive velocity vector is then formulated as
\begin{equation}
\mathbf{v}_i
=
\eta \cdot
\exp\!\left(-\frac{1}{2}\boldsymbol{\delta}_i^2\right)
\odot
\operatorname{sgn}(\boldsymbol{\delta}_i)
\odot
\boldsymbol{\sigma}_i,
\end{equation}
where $\odot$ denotes the Hadamard product and $\eta$ is the repulsion step-size parameter. Intuitively, the exponential term modulates the repulsion intensity according to the candidate's relative position with respect to the temporary archive centroid, so that the repulsive effect is stronger in regions that are more densely occupied. The sign term determines the repulsion direction, and $\boldsymbol{\sigma}_i$ scales the displacement according to the spread of the task distribution.

Finally, the repelled offspring is obtained by aggregating the repulsive contributions from all task archives and projecting the result back into the feasible search domain:
\begin{equation}\label{eq:repulsion_update}
\mathbf{x}
=
\operatorname{Clip}
\left(
\hat{\mathbf{x}}+\sum_{i=1}^{I}\mathbf{v}_i,\,
\mathbf{LB},\,
\mathbf{UB}
\right),
\end{equation}
where $\mathbf{LB}$ and $\mathbf{UB}$ denote the lower and upper bounds of the genotype space, respectively. The pseudocode of the proposed multitask repulsive operator is provided in \textbf{Algorithm~2}.

\subsection{Inter-Task Transfer Probability Adaptation}\label{sec:itp_adaptation}
The strength of the transfer in MTNS-CoD is governed by the inter-task transfer probability matrix, denoted by $\textit{ITP}$, which controls how frequently offspring generated from one task are reassigned to another task. In this subsection, we introduce how the $\textit{ITP}$ is adaptive online.

{
\setlength{\textfloatsep}{10pt}
\begin{algorithm}[!t]
\caption{Inter-Task Transfer Probability Adaptation}
\label{alg:itp_adaptation}
\begin{algorithmic}[1]
    \REQUIRE Offspring set $\mathcal{S}^{(t)}$, archive set $\{\mathcal{A}_j^{(t-1)}\}_{j=1}^{I}$, latent matrix $\mathbf{\Theta}^{(t)}$, cumulative return matrix $\mathbf{G}^{(t-1)}$, learning rate $\alpha$, decay factor $\lambda$, regularization factor $\beta$
    \ENSURE Updated transfer probability matrix $\textit{ITP}^{(t+1)}$

    \STATE Compute the contribution matrix $\mathbf{C}^{(t)}=[C_{i,j}^{(t)}]$ according to \eqref{eq:contribution_matrix};
    \STATE Compute the reward matrix $\mathbf{R}^{(t)}$ according to \eqref{eq:retention_reward_refined};
    \STATE Update the cumulative reward matrix according to \eqref{eq:discounted_return_refined};

    \FOR{$i = 1$ to $I$}
        \STATE $b_i^{(t)} \leftarrow \frac{1}{I}\sum_{j=1}^{I} G_{i,j}^{(t)}$;
        \FOR{$j = 1$ to $I$}
            \STATE $A_{i,j}^{(t)} \leftarrow G_{i,j}^{(t)} - b_i^{(t)}$;
            \STATE $\theta_{i,j}^{(t+1)} \leftarrow \theta_{i,j}^{(t)} + \alpha A_{i,j}^{(t)}(1-\rho_{i,j}(\mathbf{\Theta}^{(t)}))$;
        \ENDFOR
        \STATE Calculate $\textit{ITP}^{(t+1)}$ based on $\mathbf{\Theta}^{(t+1)}$ according to \eqref{eq:itp_row_norm};
    \ENDFOR
    \RETURN $\textit{ITP}^{(t+1)}$
\end{algorithmic}
\end{algorithm}
}

Let $\mathbf{\Theta}=[\theta_{i,j}] \in \mathbb{R}^{I\times I}$ be a latent parameter matrix, where $\theta_{i,j}$ characterizes the preference of assigning offspring generated from source task $\mathcal{T}_i$ to target task $\mathcal{T}_j$. The corresponding transfer probability is obtained by row-wise normalization:
\begin{equation}
\label{eq:itp_row_norm}
\text{itp}_{i,j} = \rho_{i,j}(\mathbf{\Theta}) = \frac{e^{\theta_{i,j}}}{\sum_{j'=1}^{I}e^{\theta_{i,j'}}}.
\end{equation}
Thus, each row of $\textit{ITP}$ defines a categorical distribution over target tasks conditioned on the source task, i.e., $\textit{ITP} = [\rho_{i,j}(\mathbf{\Theta})] \in \mathbb{R}^{I\times I}$.
To adapt $\textit{ITP}$, we evaluate the contribution of each source--target transfer channel at every generation. Let $\mathcal{S}^{(t)}$ denote the set of all offspring generated at generation $t$. For each pair $(i,j)$, let $\mathcal{S}_{i,j}^{(t)} \subseteq \mathcal{S}^{(t)}$ denote the subset of offspring that are generated by emitters associated with source task $\mathcal{T}_i$ and evaluated on target task $\mathcal{T}_j$. Let $\mathcal{A}_j^{(t-1)}$ be the archive of task $\mathcal{T}_j$ before inserting the offspring of generation $t$. We define the \emph{contribution} of channel $(i,j)$ at generation $t$ as
\begin{equation}
\label{eq:contribution_matrix}
C_{i,j}^{(t)}
=
\left|
\mathcal{U}_j\!\left(\mathcal{A}_j^{(t-1)}, \mathcal{S}_{i,j}^{(t)}\right)
\right|
-
\left|
\mathcal{A}_j^{(t-1)}
\right|,
\end{equation}
where $\mathcal{U}_j(\cdot,\cdot)$ denotes the archive update operator of task $\mathcal{T}_j$, including the same admission criteria as those used in the actual archive maintenance procedure, such as novelty thresholds and proximity constraints. Then, $C_{i,j}^{(t)}$ quantifies the effective contribution of transfer channel $(i,j)$ to the archive growth of task $\mathcal{T}_j$. Based on the contributions, the reward is then defined based on the \emph{contribution ratio} of each transfer channel, i.e.,
\begin{equation}
\label{eq:retention_reward_refined}
R_{i,j}^{(t)}=\frac{C_{i,j}^{(t)}}{\max\!\left(N_{i,j}^{\text{eval},(t)},1\right)},
\end{equation}
where $N_{i,j}^{\text{eval},(t)}$ is the number of offspring evaluated through the same channel at the $t$-th iteration. Therefore, $R_{i,j}^{(t)}$ measures the effective retention rate of transfer channel $(i,j)$.
To capture the long-term utility of transfer, we further maintain a discounted cumulative reward matrix $\mathbf{G}^{(t)}=[G_{i,j}^{(t)}]$, updated as
\begin{equation}
\label{eq:discounted_return_refined}
\begin{aligned}
\mathbf{G}^{(t)} =\lambda \mathbf{G}^{(t-1)} + & \mathbf{R}^{(t)} - \beta \cdot \left( \textit{ITP}^{(t-1)} - \mathbf{M} \right), \ 
\mathbf{G}^{(0)} & =\mathbf{0},
\end{aligned}
\end{equation}
where $\lambda\in[0,1]$ is the decay factor, $\mathbf{R}^{(t)}=[R_{i,j}^{(t)}]$ is the reward matrix, $\mathbf{M}$ is the identity matrix, and $\beta  \cdot \left( \textit{ITP}^{(t)} - \mathbf{M} \right)$ is a regularization term that discourages transfer when task contributions are similar, thus reducing the risk of negative transfer.

The adaptation is performed row-wise based on $\mathbf{G}^{(t)}$. For each source task $\mathcal{T}_i$, we define the baseline as the mean cumulative reward over its outgoing transfer channels:
\begin{equation}
\label{eq:baseline_final}
b_i^{(t)}=\frac{1}{I}\sum_{j=1}^{I}G_{i,j}^{(t)}.
\end{equation}
The corresponding advantage is given by
\begin{equation}
\label{eq:advantage_final}
A_{i,j}^{(t)} = G_{i,j}^{(t)}-b_i^{(t)},
\end{equation}
so that channels with above-average cumulative contribution receive positive reinforcement. Using this advantage signal, the latent parameter matrix $\mathbf{\Theta}$ is updated as
\begin{equation}
\label{eq:theta_update_final}
\theta_{i,j}^{(t+1)}
=
\theta_{i,j}^{(t)}
+
\alpha \, \left[ A_{i,j}^{(t)} \left(1-\rho_{i,j}(\mathbf{\Theta}^{(t)})\right)   \right],
\end{equation}
where $\alpha>0$ is the learning rate. Afterward, the transfer probability matrix $\textit{ITP}$ is recomputed according to \eqref{eq:itp_row_norm}.
The pseudo code of the whole adaptation process is described in \textbf{Algorithm~3}.

\section{MTNS-CoD for Novelty-Augmented Optimization}
Unlike pure novelty search, novelty-augmented optimization still aims to identify the optimal solution with respect to a predefined objective, while leveraging novelty as an auxiliary search signal to alleviate deception and premature convergence in complex fitness landscapes. By discouraging the search from stagnating in well-explored local optima, novelty augmentation facilitates the discovery of higher-quality basins of attraction. To expand the proposed MTNS-CoD framework to novelty-augmented optimization, we draw inspiration from NSRA-ES in~\cite{conti2018improving} and equip MTNS-CoD with the ability to dynamically trade off objective-driven exploitation and novelty-guided exploration. The resulting extension is mainly characterized by the following two strategies:
\begin{itemize}
\item \textbf{Weighted Multitask Repulsive Operator:}
In novelty-augmented optimization, the trade-off coefficient $w$ in \eqref{eq:nov_aug_obj} controls the relative emphasis between objective-driven exploitation and novelty-guided exploration. As the multitask repulsive operator in MTNS-CoD is designed to promote novelty-guided exploration, its influence should be pronounced when novelty is emphasized and suppressed when the search is dominated by objective exploitation. Accordingly, the repulsive operator is reformulated heuristically as
\begin{equation}
\label{eq:rep_update_nao}
\textbf{x} = \mathrm{Clip}\!\left(\hat{\textbf{x}} + \sum_{i=1}^{I} w_i \mathbf{v}_i,\ \textbf{LB},\ \textbf{UB}\right),
\end{equation}
where $w_i$ denotes the trade-off coefficient associated with the $i$th task, and thus directly modulates the magnitude of the corresponding repulsive update. Following NSRA-ES~\cite{conti2018improving}, each $w_i$ is adaptively updated during evolution.

\item \textbf{Objective-Driven Inter-Task Transfer Probability Adaptation:}
Unlike pure novelty search, where the $\textit{ITP}$ matrix is adapted according to the contribution of transferred solutions to archive expansion, novelty-augmented optimization requires the transfer mechanism to favor offspring that improve the objective performance. Accordingly, we redefine the reward signal $R_{i,j}^{(t)}$ to evaluate the effectiveness of each transfer channel from an objective-driven perspective:
\begin{equation}
    R_{i,j}^{(t)} = \frac{1}{N_{i,j}^{\text{eval},(t)}} \sum_{k=1}^{N_{i,j}^{\text{eval},(t)}} f_j\!\left(x_{i,j}^{k,(t)}\right),
\end{equation}
where $f_j(x)$ denotes the normalized objective value for the $j$-th task, and $x_{i,j}^{k,(t)}$ denotes the $k$-th offspring generated through the transfer channel $(i,j)$ at the $t$-th generation.
\end{itemize}
Through these modifications, MTNS-CoD achieves a balance between objective-driven exploitation and novelty-guided exploration, thereby promoting the discovery of high-quality solutions.

\section{Experimental Studies}
\subsection{Benchmark Problems}
In this paper, we consider four groups of benchmark problems with distinct characteristics. First, to analyze the search behavior of MTNS-CoD and the effectiveness of its key components, we construct a family of synthetic problems with two-dimensional genotype and phenotype spaces, enabling direct visualization and interpretation of the search dynamics. Second, deceptive maze navigation tasks~\cite{grillotti2023kheperax} are adopted to evaluate the proposed method on pure novelty search problems. Third, MuJoCo control tasks~\cite{todorov2012mujoco} are introduced to further assess its performance on novelty-regularized optimization problems. Finally, generative-model-based novelty search is considered to visually demonstrate the ability of MTNS-CoD to discover behaviorally diverse and novel samples~\cite{karras2020analyzing}.

\subsubsection{Synthetic Basin-Type Multitask Test Problems}
To enable intuitive and visually interpretable analysis, we construct a family of synthetic multitask novelty search problems. Specifically, three representative two-task settings are considered in the unified genotype space, with fully overlapping, partially overlapping, and disjoint target genotype sets, respectively. The target genotype sets are illustrated in Fig.~\ref{fig:problems_target_set}, and the detailed problem definitions are provided in Appendix~\ref{app:benchmark1}. These synthetic basin-type problems serve as an effective testbed for examining the exploration dynamics of MTNS-CoD, particularly the interplay between the searching of two tasks under different overlap structures.

\begin{figure}[!t]
	\begin{center}
        \subfigure[]{\label{result1}\includegraphics[width=0.25\columnwidth]{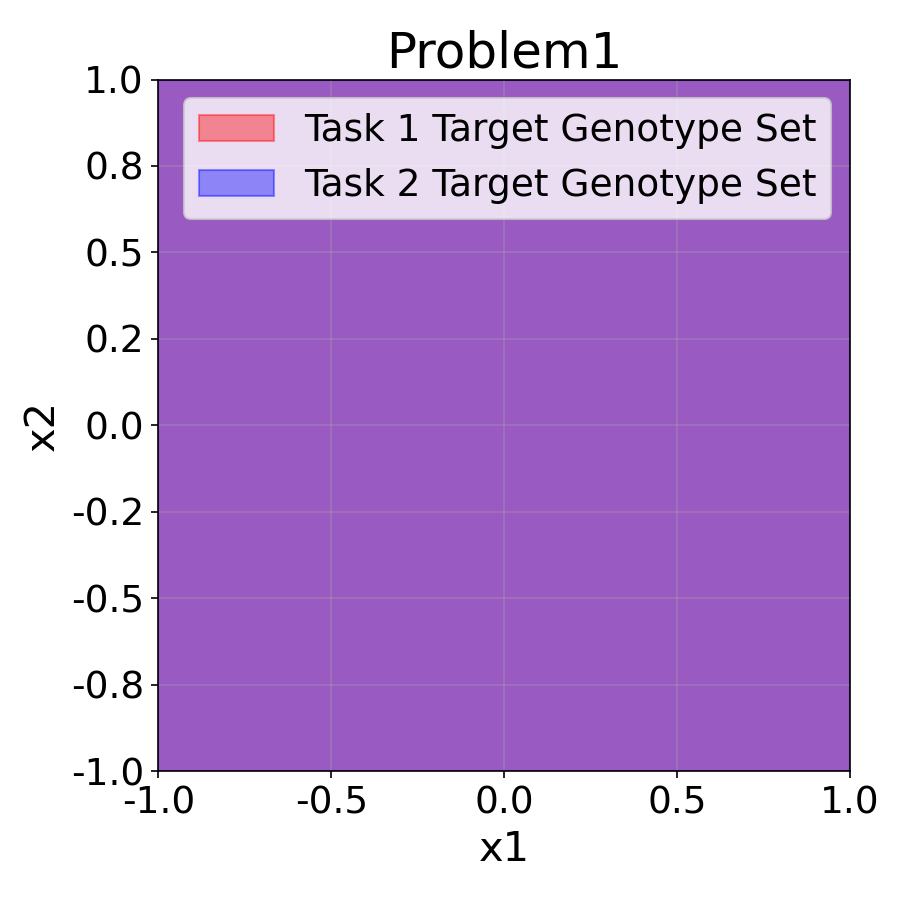}}
		\subfigure[]{\label{result1}\includegraphics[width=0.25\columnwidth]{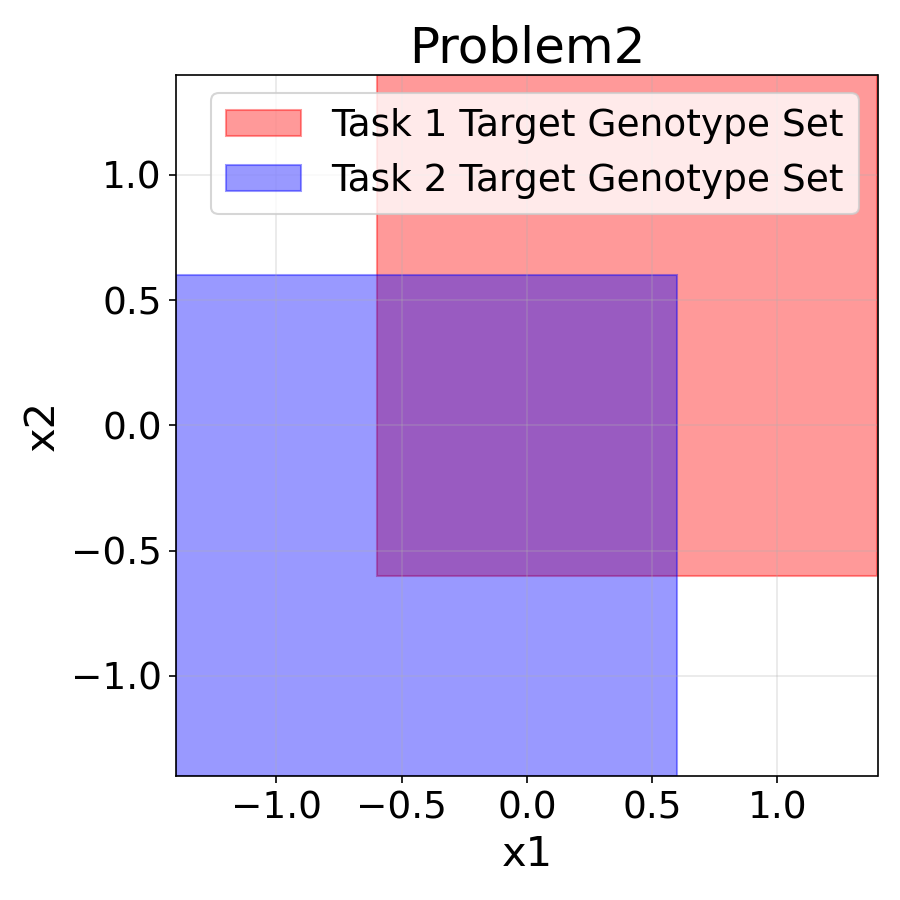}}
        \subfigure[]{\label{result1}\includegraphics[width=0.25\columnwidth]{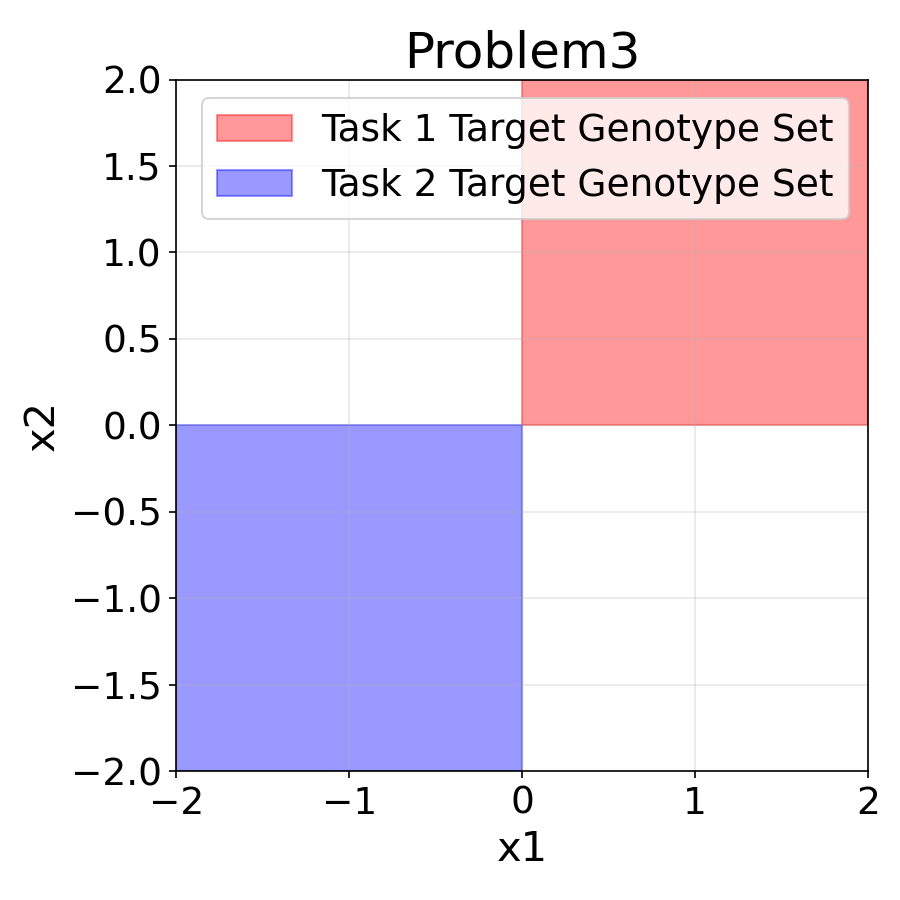}}
        \caption{Target genotype sets of the synthetic basin-type multitask test problems under three representative overlap scenarios. (a) \textbf{Problem 1:} fully overlapping target genotype sets of Task~1 and Task~2. (b) \textbf{Problem 2:} partially overlapping target genotype sets of Task~1 and Task~2. (c) \textbf{Problem 3:} disjoint target genotype sets of Task~1 and Task~2.}
        \label{fig:problems_target_set}
        \vspace{-12pt}
	\end{center}
\end{figure}

\subsubsection{Multitask Deceptive Maze Navigation Problems}
The deceptive maze navigation problem is adopted as a representative benchmark for evaluating pure novelty search algorithms~\cite{lehman2011abandoning,pyribs_ns_maze}. In this problem, a mobile robot controlled by a policy neural network must navigate through a two-dimensional maze from a predefined start position toward a target location, with the network parameters serving as the genotype decision variables. The problem is deceptive because greedily approaching the goal may lead the robot into blocked regions rather than along feasible paths. As such, successful search requires exploratory behaviors that may initially move away from the target, making this benchmark particularly suitable for studying novelty search without explicit objective guidance. To evaluate multitask novelty search under different levels of task relatedness, we construct three two-task cases: identical mazes with different start positions, identical mazes with different target positions, and different maze environments. The mazes are selected from two predefined Kheperax~\cite{grillotti2023kheperax} environments, namely \textit{standard} and \textit{snake}. Accordingly, \textit{Case}~1 corresponds to the same maze with different starts, \textit{Case}~2 to the same maze with different targets, and \textit{Case}~3 to different maze environments. For all the tasks, the behavioral descriptor in the phenotype space is defined as the final robot position at the end of the episode. Other detailed settings are given in Appendix~\ref{app:benchmark2}, and the corresponding maze layouts, start positions, and target positions are shown in Fig.~\ref{fig:problems_maze_settings}.

\begin{figure*}[!t]
	\begin{center}
        \subfigure[]{\label{result1}\includegraphics[width=0.32\columnwidth]{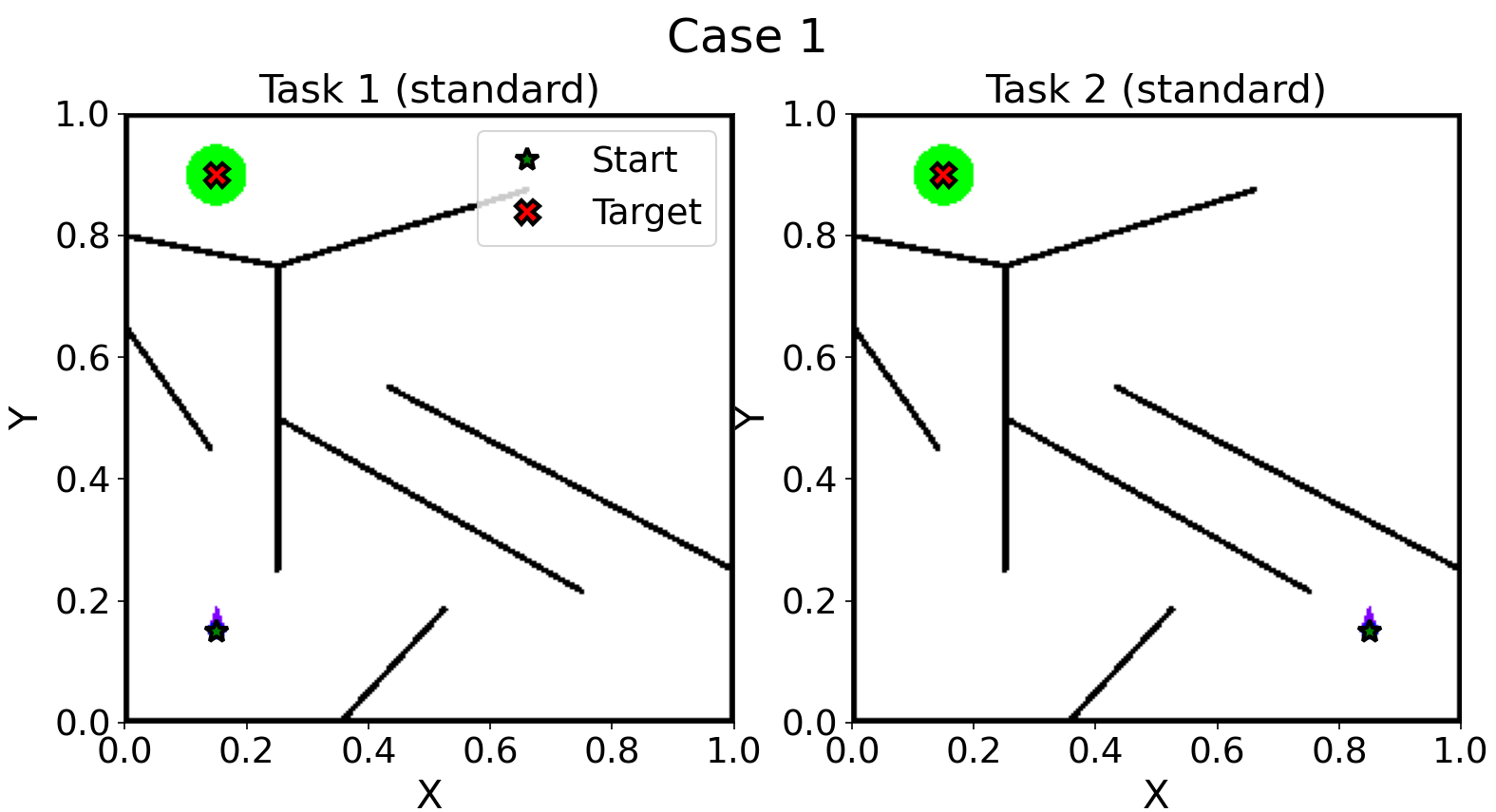}}
		\subfigure[]{\label{result1}\includegraphics[width=0.32\columnwidth]{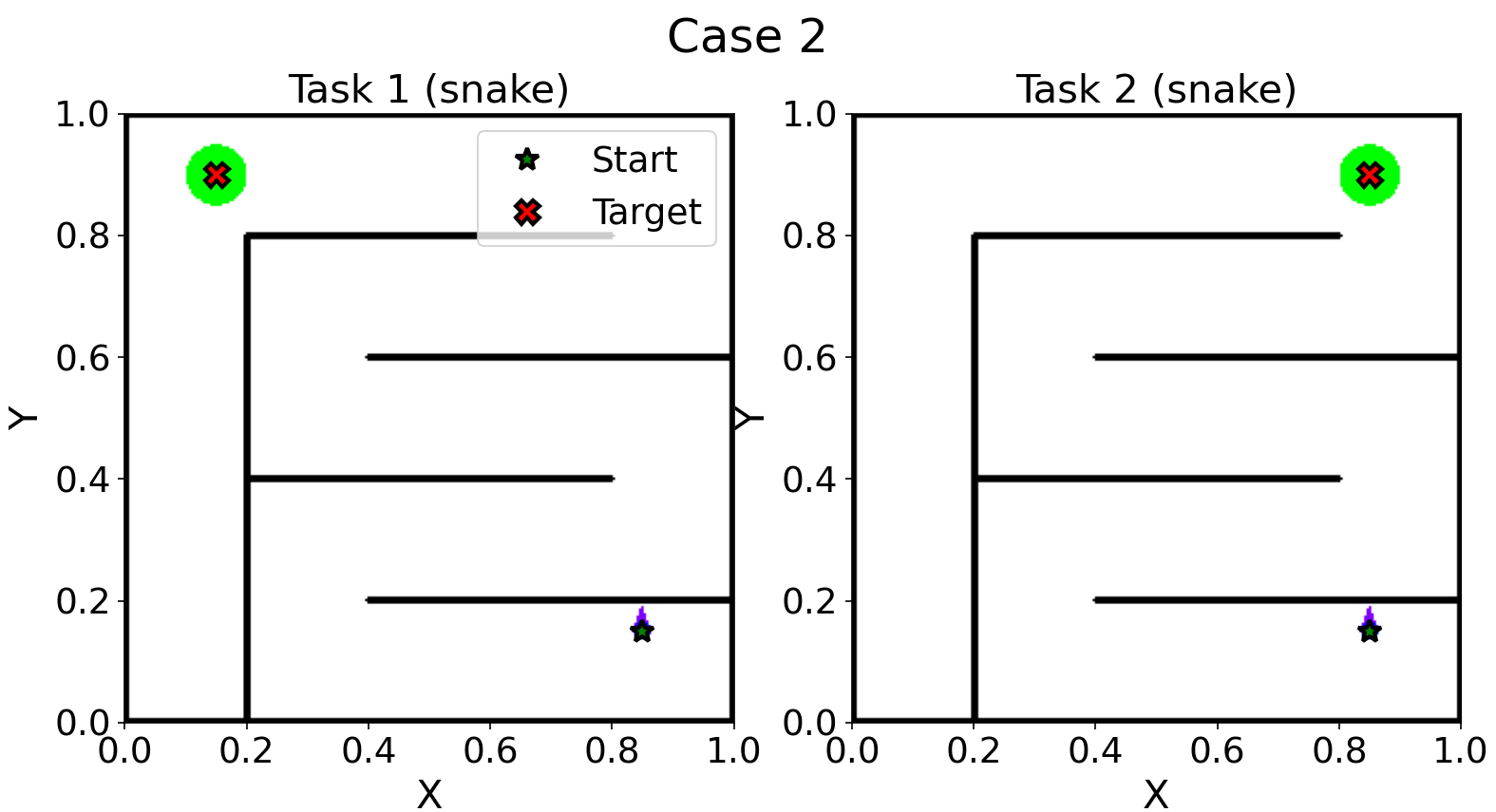}}
        \subfigure[]{\label{result1}\includegraphics[width=0.32\columnwidth]{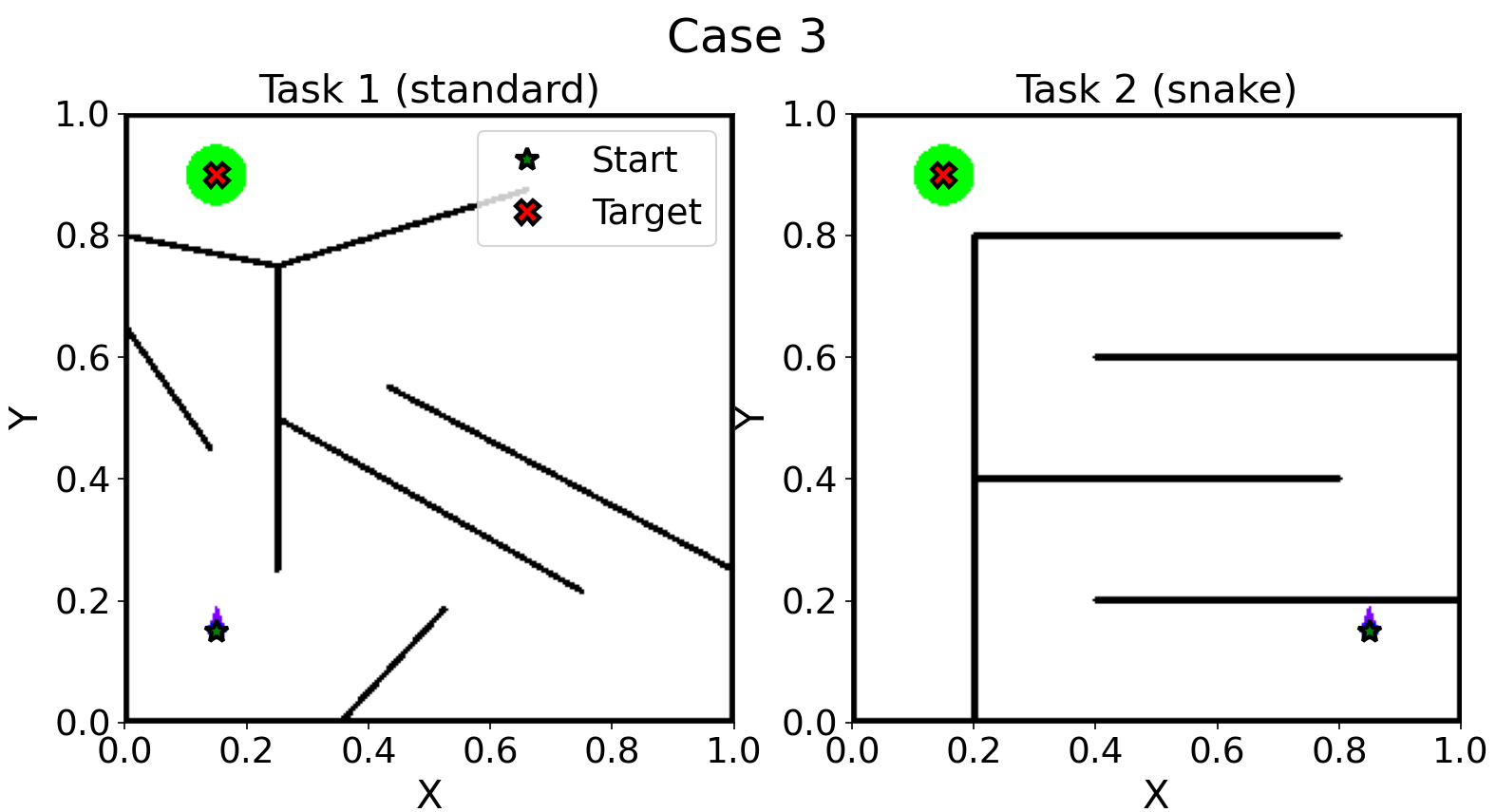}}
        \caption{Illustration of the maze layouts, start positions, and target positions of the three multitask deceptive maze navigation problems. (a) \textbf{Case 1:} identical maze with different start positions. (b) \textbf{Case 2:} identical maze with different target positions. (c) \textbf{Case 3:} different maze environments.}
        \label{fig:problems_maze_settings}
        \vspace{-12pt}
	\end{center}
\end{figure*}

\subsubsection{Multitask MuJoCo Policy Optimization Problems}
To further assess the proposed method on novelty-augmented optimization problems, we construct a family of multitask MuJoCo~\cite{todorov2012mujoco} benchmarks based on Gymnasium~\cite{towers2024gymnasium} environments. In each benchmark, two tasks share the same observation-action structure and policy parameterization, and differ only in task-specific physical parameters. Two representative two-task problems are considered based on the classical MuJoCo benchmarks including the Hopper and Walker2d (as shown in Fig.~\ref{fig:mujoco}). The first is a \textit{Hopper Gravity} problem, where the two tasks correspond to earth-like and moon-like gravity, respectively. The second is the \textit{Walker2d Friction} problem, where each task pair is defined by normal-friction and low-friction dynamics. For all problems, the objective is the cumulative episode reward. The other detailed settings of the problems, including the definitions of the behavioral descriptors, can be found in Appendix~\ref{app:benchmark3}.

\begin{figure}[!h]
	\begin{center}
        \subfigure[]{\label{result1}\includegraphics[width=0.23\columnwidth]{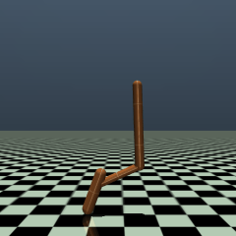}}
		\subfigure[]{\label{result1}\includegraphics[width=0.23\columnwidth]{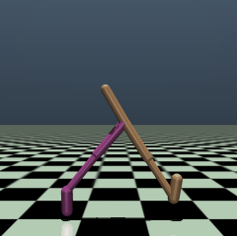}}
        \caption{The considered MuJoCo benchmarks. (a) Hopper. (b) Walker2d.}
        \label{fig:mujoco}
        \vspace{-12pt}
	\end{center}
\end{figure}

\subsubsection{Multitask Generative Novelty Search Problems}
To provide a visually interpretable example of multitask novelty search, we construct a multitask generative novelty search problem based on a pretrained StyleGAN2 generator trained on the FFHQ dataset~\cite{karras2020analyzing}. In this problem, candidate solutions are encoded as latent vectors, and search is conducted by varying the latent vectors to control the generated outputs. The two tasks are defined over different phenotype spaces. In the first task, the two behavior descriptors are age and hair length, representing two attributes of the generated images. In the second task, the behavior descriptors are smile expressions and frontal view. These behavior descriptors are evaluated using CLIP~\cite{radford2021learning}, where the generated images and the corresponding text prompts are projected into the CLIP embedding space and their similarity is computed. More detailed settings can be found in Appendix~\ref{app:benchmark4}.

\subsection{Results on Synthetic Basin-Type Multitask Test Problems}\label{sec:results_problem}

\begin{figure}[!t]
	\begin{center}
        \includegraphics[width=0.25\columnwidth]{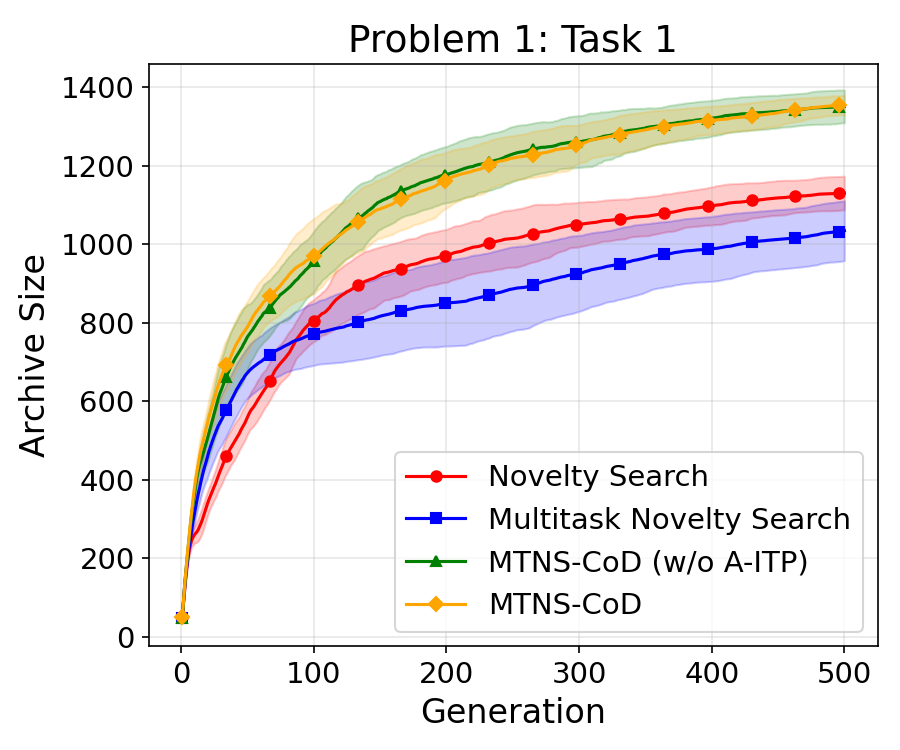}
		\includegraphics[width=0.25\columnwidth]{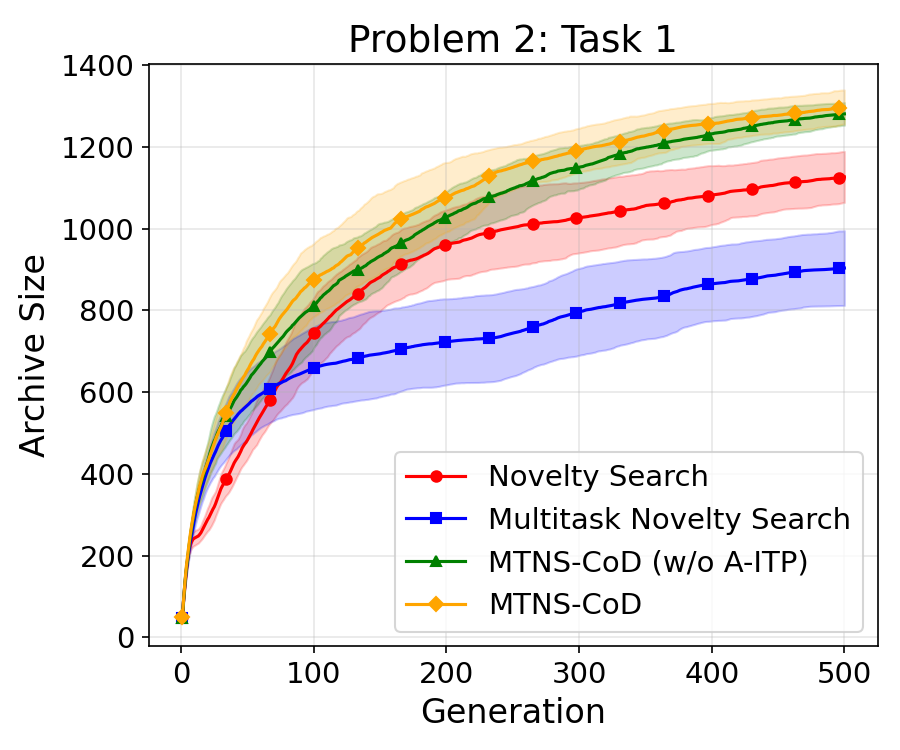}
        \includegraphics[width=0.25\columnwidth]{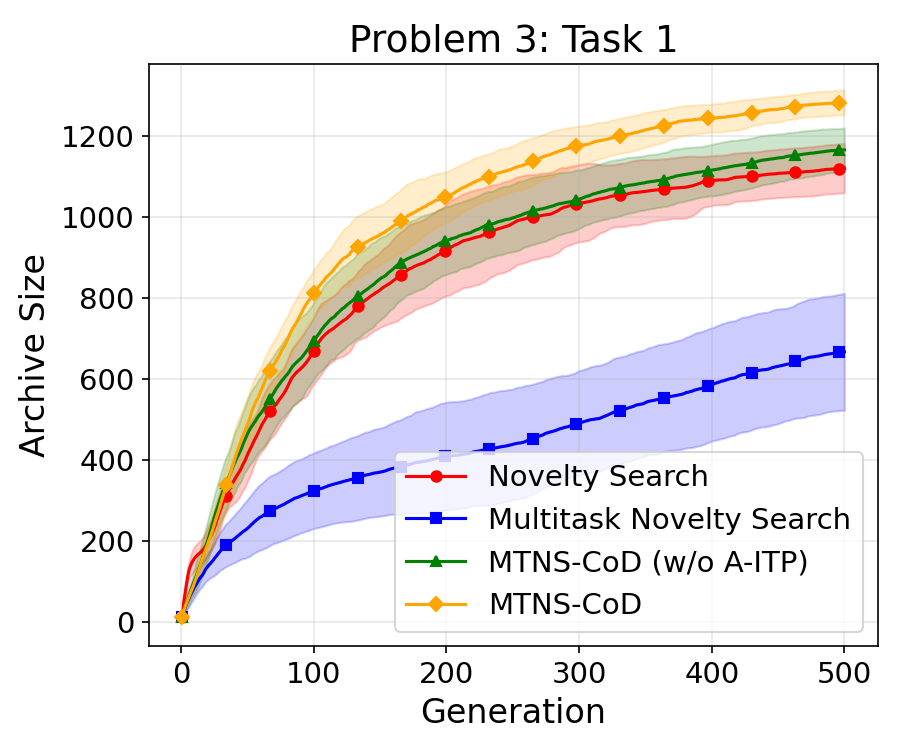}
        \includegraphics[width=0.25\columnwidth]{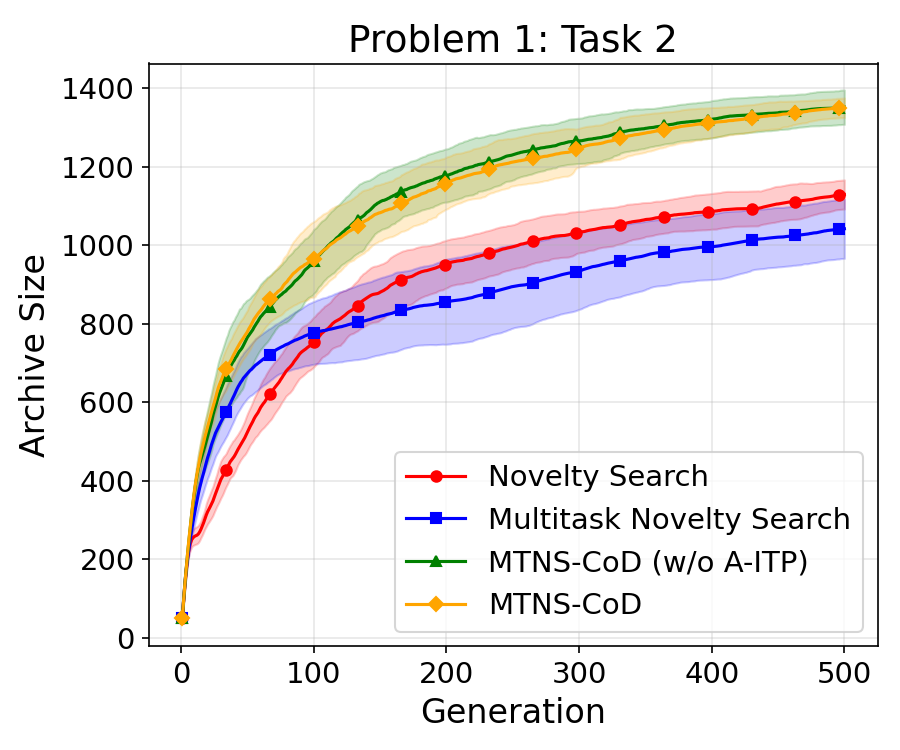}
		\includegraphics[width=0.25\columnwidth]{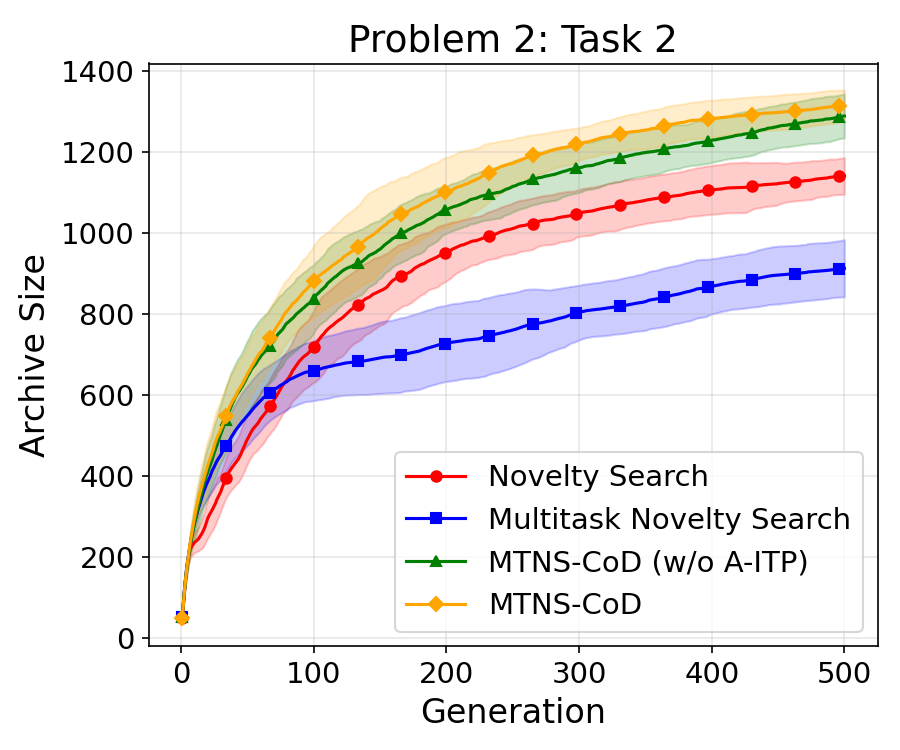}
        \includegraphics[width=0.25\columnwidth]{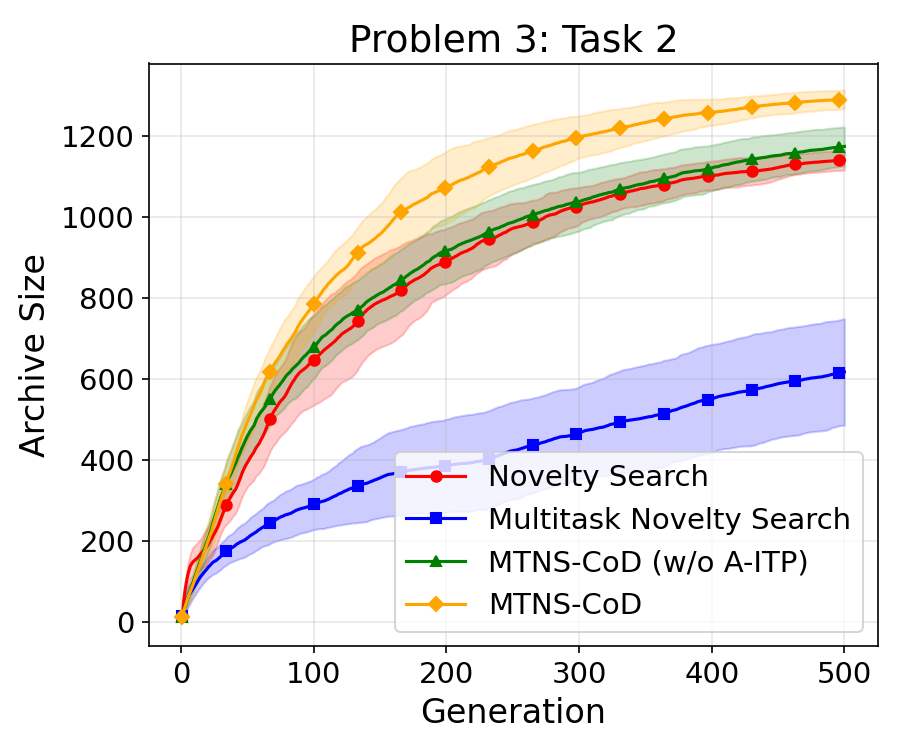}
        \caption{Archive size convergence trends averaged over 20 independent runs of Novelty Search, Multitask Novelty Search, MTNS-CoD (w/o A-ITP), and MTNS-CoD on the synthetic basin-type multitask test problems. 
        }
        \label{fig:problem_convergence}
        \vspace{-12pt}
	\end{center}
\end{figure}

In this subsection, we mainly investigate the performance of the proposed method on the synthetic basin-type multitask test problems. As baselines, MTNS-CoD is compared with the following methods:
\begin{itemize}
    \item \textit{Novelty Search}~\cite{lehman2011abandoning}: This method serves as the standard baseline for pure novelty search. Each task is solved independently using a conventional novelty search algorithm, without any inter-task transfer or repulsion mechanism.
    
    \item \textit{Multitask Novelty Search}: This method serves as a multitask baseline for pure novelty search. It incorporates inter-task transfer with a fixed transfer probability, but does not employ the proposed repulsion mechanism.
    
    \item \textit{MTNS-CoD (w/o A-ITP)}: This variant of the proposed method employs both inter-task transfer and multitask repulsion, while the inter-task transfer probability is fixed rather than being adaptively adjusted.
\end{itemize}
All methods are implemented using the \texttt{pyribs} library~\cite{tjanaka2023pyribs}. In particular, CMA-ES-based emitters are instantiated with the default update mechanism of \texttt{pyribs}' \texttt{EvolutionStrategyEmitter} class, while the \texttt{ProximityArchive} provided by \texttt{pyribs} is adopted to maintain the archive of discovered novel solutions. Detailed parameter settings are provided in Appendix~\ref{app:settings1}. All of the results are obtained via 20 independent runs. The effectiveness and search characteristics of the proposed method are analyzed from three aspects: 1) the number of solutions retained in the final archive and its convergence trend, 2) the accumulated search distribution in the genotype space, and 3) the dynamic of the adaptive inter-task transfer probabilities in MTNS-CoD. 

\subsubsection{Results of the Archive Size}

\begin{table*}[!t]
\centering
\caption{Archive size results of Novelty Search, Multitask Novelty Search, MTNS-CoD (w/o A-ITP), and MTNS-CoD on the synthetic basin-type multitask test problems. The Wilcoxon rank-sum test at the $0.05$ significance level was performed to compare MTNS-CoD with its competitors.}
\label{Tab:Problem_Results}
\resizebox{12cm}{!}{\begin{tabular}{c|c|c|c|c}
\toprule
\multirow{2}{*}{\textbf{Problem / Task}}
& \textbf{Novelty Search}
& \textbf{Multitask Novelty Search}
& \textbf{MTNS-CoD (w/o A-ITP)}
& \textbf{MTNS-CoD} \\
\addlinespace[2pt]
& \textit{Mean$\pm$Std}
& \textit{Mean$\pm$Std}
& \textit{Mean$\pm$Std}
& \textit{Mean$\pm$Std} \\
\midrule
\textbf{Problem 1 -- Task 1}
& 1130.70$\pm$42.92\;($+$)
& 1034.85$\pm$75.95\;($+$)
& 1351.55$\pm$41.73\;($=$)
& \textbf{1354.70$\pm$24.92} \\
\textbf{Problem 1 -- Task 2}
& 1129.55$\pm$37.12\;($+$)
& 1042.00$\pm$75.34\;($+$)
& \textbf{1351.85$\pm$43.84}\;($=$)
& 1350.30$\pm$25.11 \\
\addlinespace[2pt]
\textbf{Problem 2 -- Task 1}
& 1126.85$\pm$62.40\;($+$)
& 903.85$\pm$91.46\;($+$)
& 1281.15$\pm$27.90\;($=$)
& \textbf{1297.80$\pm$41.89} \\
\textbf{Problem 2 -- Task 2}
& 1139.95$\pm$45.54\;($+$)
& 911.80$\pm$70.54\;($+$)
& 1287.70$\pm$54.21\;($=$)
& \textbf{1314.40$\pm$37.35} \\
\addlinespace[2pt]
\textbf{Problem 3 -- Task 1}
& 1120.17$\pm$60.62\;($+$)
& 666.80$\pm$144.05\;($+$)
& 1165.35$\pm$53.47\;($+$)
& \textbf{1283.00$\pm$29.97} \\
\textbf{Problem 3 -- Task 2}
& 1139.56$\pm$24.94\;($+$)
& 617.00$\pm$131.43\;($+$)
& 1173.25$\pm$47.69\;($+$)
& \textbf{1289.50$\pm$23.06} \\
\midrule
\multicolumn{1}{c|}{\textbf{$+/-/=$}} & {$6/0/0$}  & {$6/0/0$} & {$2/4/0$} & / \\
\bottomrule
\end{tabular}}
\end{table*}


\begin{figure*}[!t]
	\begin{center}
        \includegraphics[width=0.495\columnwidth]{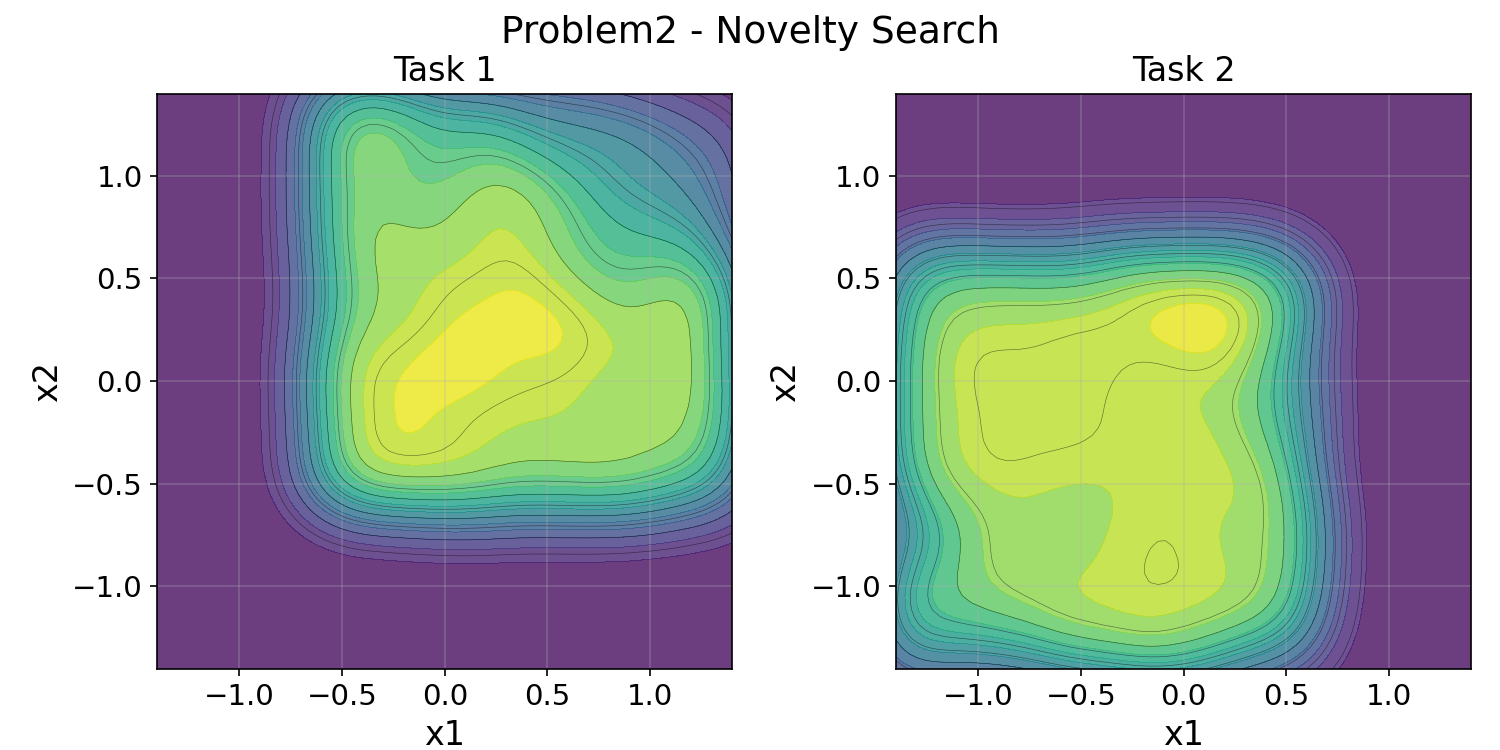}
		\includegraphics[width=0.495\columnwidth]{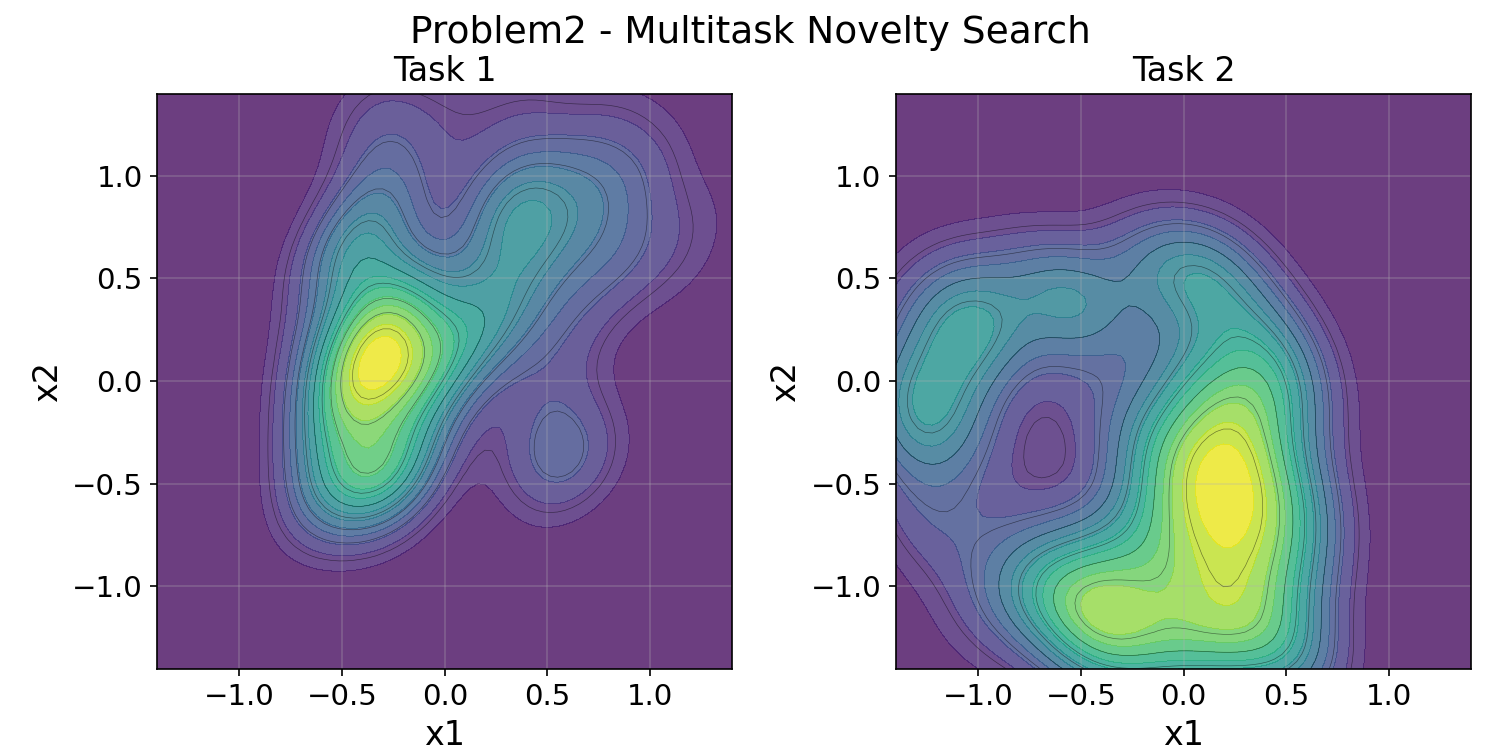}
        \includegraphics[width=0.495\columnwidth]{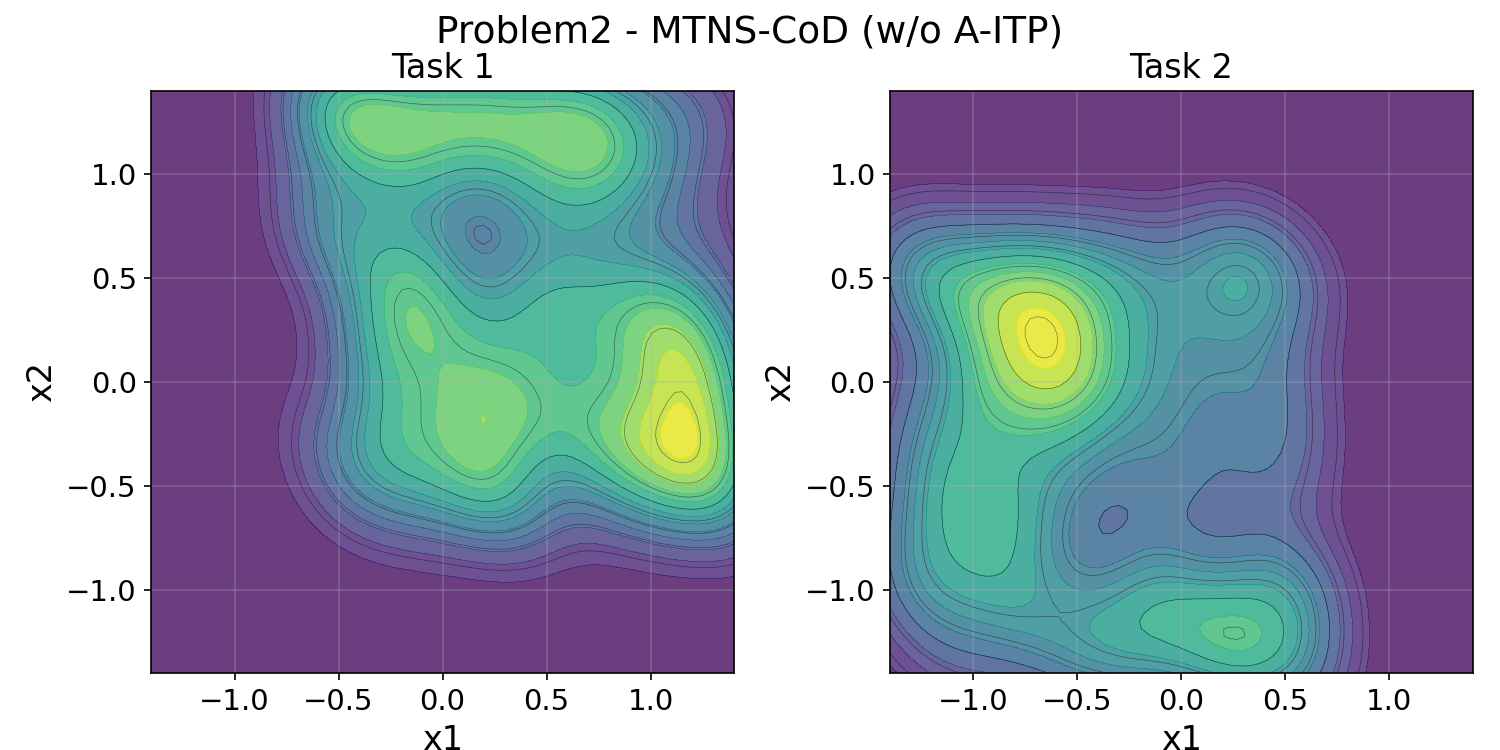}
        \includegraphics[width=0.495\columnwidth]{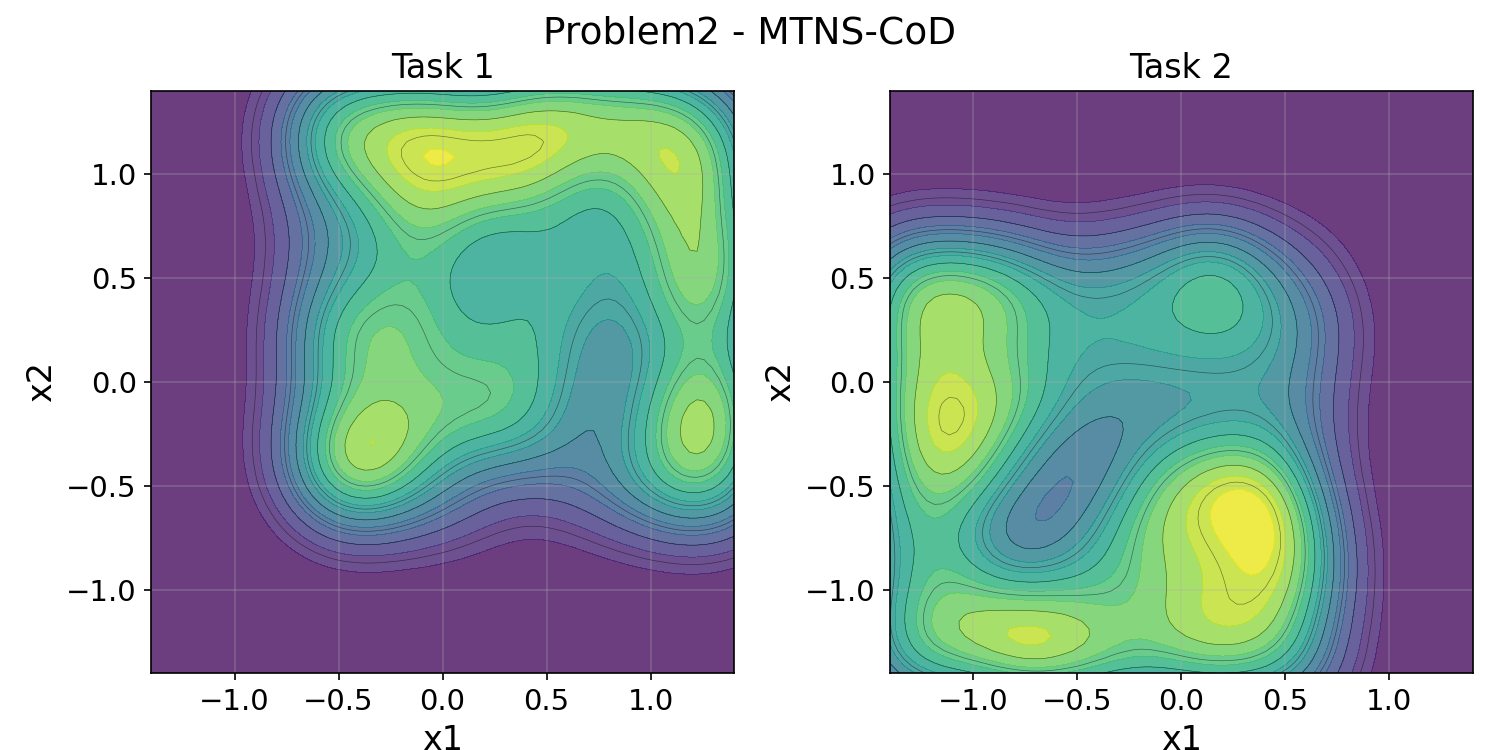}
        \caption{The searching distributions in the genotype space of the Novelty Search, Multitask Novelty Search, MTNS-CoD (w/o A-ITP), and MTNS-CoD on the two tasks of Problem 2 of synthetic basin-type multitask test problems. (a) Novelty Search. (b) Multitask Novelty Search. (c) MTNS-CoD (w/o A-ITP). (d) MTNS-CoD.}
        \label{fig:Problem2_Distribution}
        \vspace{-12pt}
	\end{center}
\end{figure*}

For pure novelty search, the goal is to continuously discover solutions that exhibit novel phenotypes. In the \texttt{pyribs} library, a solution is inserted into the \texttt{ProximityArchive} once it satisfies the novelty criterion. Therefore, the number of retained solutions in the archive serves as one of the most important indicators of the search capability of an algorithm. Based on this metric, Fig.~\ref{fig:problem_convergence} presents the convergence curves of archive size for the four compared algorithms during the evolutionary process. As can be observed from Fig.~\ref{fig:problem_convergence}, on all tasks of Problems~1--3, both MTNS-CoD and MTNS-CoD (w/o A-ITP) converge substantially faster than the other two baseline methods. Notably, these two methods both incorporate the proposed multitask repulsion operator. This result provides the evidence that the repulsion mechanism is effective for multitask novelty search. Meanwhile, compared with MTNS-CoD (w/o A-ITP), the full version of MTNS-CoD exhibits consistently faster convergence on all tasks of Problems~2 and~3, further confirming the effectiveness of the proposed inter-task transfer probability adaptation strategy on this family of synthetic test problems. For Problem~1, these two repulsion-based algorithms exhibit similar convergence trends, likely because the fixed transfer probability used in MTNS-CoD (w/o A-ITP) is already close to an effective setting for this problem. Table~\ref{Tab:Problem_Results} also reports the final archive sizes and the results of the Wilcoxon rank-sum test at the $0.05$ significance level, where ``$+$'', ``$-$'', and ``$\approx$'' indicate that MTNS-CoD performs better than, worse than, or comparably to a competing method, respectively. The results show that MTNS-CoD outperforms all three baselines, achieving statistically significant improvements on six, six, and two tasks, respectively.

\subsubsection{Search Distributions in the Genotype Spaces}
To further understand the search characteristics of MTNS-CoD, we additionally analyze its search distribution in the genotype space. To this end, we construct the accumulated search distribution as follows. First, from the archive collected during evolution, we remove all solutions introduced via inter-task transfer, and retain only those generated by the emitter of the corresponding task and evaluated on that same task. In this way, the resulting solution set reflects the independent search behaviors of each task, while excluding regions contributed by transferred solutions from other tasks. Based on these filtered solutions, kernel density estimation~\cite{wkeglarczyk2018kernel} is then applied to approximate the accumulated search distribution, thereby revealing the search trajectory induced by the emitters of each task over the course of evolution. For reference, the same analysis is also conducted for the other three baseline methods, so as to provide a clearer comparison of their search behaviors.


Fig.~\ref{fig:Problem2_Distribution} illustrates the search distributions of the four algorithms on Problem~2, where the target genotype sets of the two tasks partially overlap. Single-task Novelty Search produces relatively uniform search distributions within the target genotype region of each task. By contrast, the three multitask methods exhibit a clear reduction in search density within the overlapping region: for both tasks, the probability density in the shared region is generally lower than that in their respective non-overlapping regions, indicating that the two search processes tend to spread toward complementary parts of the genotype space rather than repeatedly concentrating on the same shared region. This complementary tendency is already visible in Multitask Novelty Search, although no explicit repulsion operator is used. One possible explanation is that the novelty objective itself introduces an implicit repulsive effect. Once solutions are transferred across tasks, they contribute to the novelty reference of the receiving task, which can reduce the novelty of nearby behaviors and indirectly encourage subsequent search to move toward less explored regions. However, because this effect is induced only through novelty evaluation and transfer, rather than through an explicit separation of the task-specific search distributions, several regions remain sparsely covered. The two methods equipped with the proposed repulsion mechanism exhibit a stronger complementary pattern. In addition to maintaining relatively low search density in the overlapping region, they allocate noticeably more search effort to the task-exclusive regions, resulting in broader coverage of the non-overlapping parts of the target genotype sets. These observations suggest that the explicit repulsion operator can reinforce the complementary exploration that emerges implicitly in multitask novelty search.

\subsubsection{Dynamic of the Inter-Transfer Probabilities}
\begin{figure}[!t]
	\begin{center}
        \includegraphics[width=0.6\columnwidth]{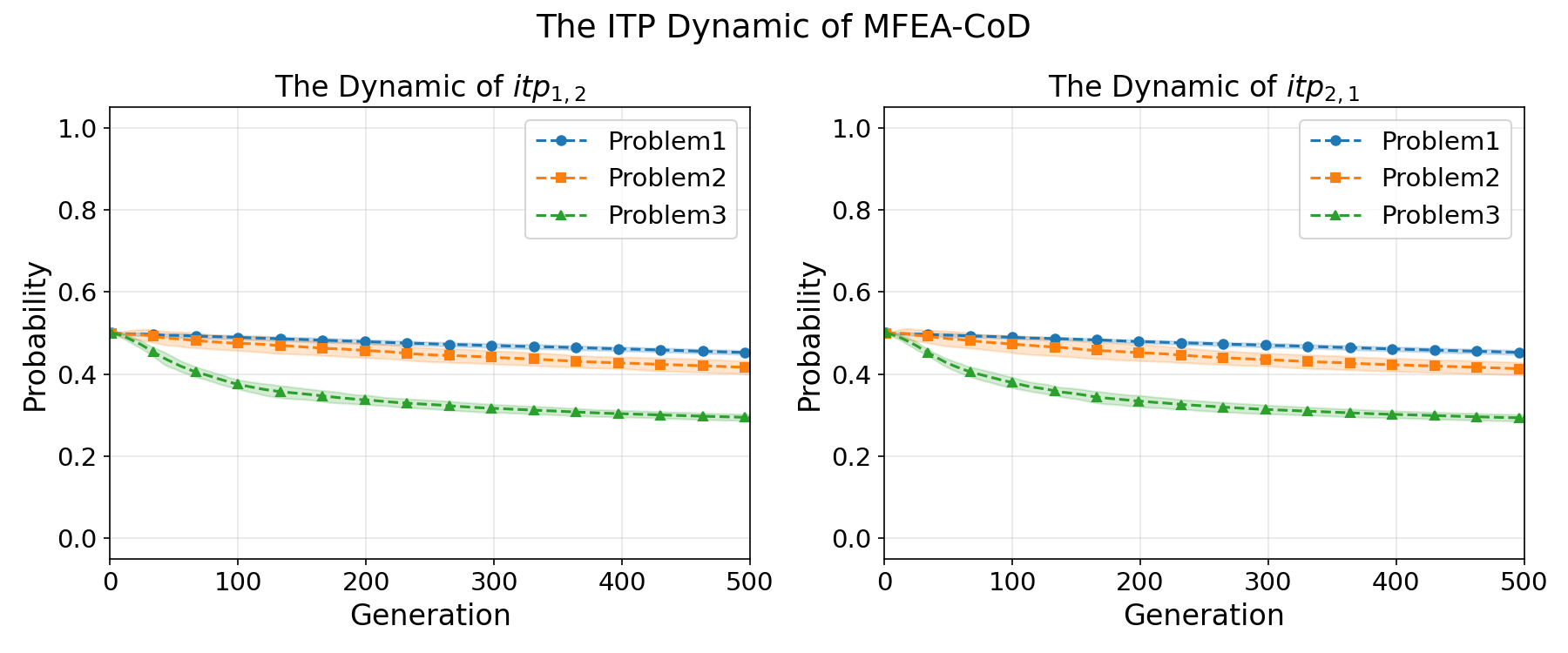}
        \caption{The dynamics of $\text{itp}_{1,2}$ and $\text{itp}_{2,1}$ in MTNS-CoD on Problems~1 to 3. The values of $\text{itp}_{1,2}$ and $\text{itp}_{2,1}$ represent the transfer probability from Task~1 to Task~2 and from Task~2 to Task~1, respectively.}
        \label{fig:ITP_Evolution}
        \vspace{-12pt}
	\end{center}
\end{figure}

To investigate the effectiveness of the proposed inter-task transfer probability adaptation strategy, we analyze the dynamic of the $\textit{ITP}$ matrix during optimization. Fig.~\ref{fig:ITP_Evolution} illustrates the dynamics of $\text{itp}_{1,2}$ and $\text{itp}_{2,1}$ in MTNS-CoD on Problems~1--3. These two values represent the transfer probability from Task~1 to Task~2 and from Task~2 to Task~1, respectively. As can be observed, the transfer probabilities decrease gradually over the evolution for all three problems. Among them, the decrease is slowest on Problem~1, followed by Problem~2, while the most pronounced decline occurs on Problem~3, where the two tasks exhibit the lowest similarity. This behavior is consistent with intuition: as task relatedness decreases from Problem~1 to Problem~3, the utility of inter-task transfer should also diminish, leading to lower transfer probabilities. These observations are also consistent with the analysis in Section~\ref{sec:intuitive_analysis}, which suggests that lower overlap between the target genotype spaces calls for weaker inter-task transfer. Overall, the results provide a basic evidence for the effectiveness and rationality of the proposed inter-task transfer probability adaptation mechanism.

\subsection{Results on Multitask Deceptive Maze Navigation Problems}
In this subsection, we further investigate the performance of MTNS-CoD on the multitask deceptive maze navigation problems. In addition to the baseline methods introduced in Section~\ref{sec:results_problem}, we further compare MTNS-CoD with two representative quality-diversity optimization algorithms in order to assess its ability to overcome deceptive objectives:
\begin{itemize}
    \item \textit{MAP-Elites}~\cite{mouret2015illuminating}: This is one of the most established baseline methods in quality-diversity optimization. It simultaneously emphasizes objective performance and diversity in the phenotype space. In our implementation, the Euclidean distance between the robot and the target is used as the objective signal.
    
    \item \textit{CMA-ME}~\cite{fontaine2020covariance}: This is another representative quality-diversity algorithm. Different from MAP-Elites, it adopts CMA-ES-based emitters to provide stronger search capability.
\end{itemize}
All methods are implemented using the \texttt{pyribs} library. For the two quality-diversity baselines, we follow their standard settings and employ the \texttt{GridArchive} provided by \texttt{pyribs} to record diverse solutions. The remaining parameter settings of the considered algorithms are provided in Appendix~\ref{app:settings2}. All of the results are obtained via 10 independent runs. The performance of MTNS-CoD on the multitask deceptive maze navigation problems is analyzed from the following two perspectives: 1) the number of solutions retained in the final archive and its convergence trend, and 2) the discovery of the successful path. 

\subsubsection{Results of the Archive Size}

\begin{figure}[!t]
	\begin{center}
        \includegraphics[width=0.25\columnwidth]{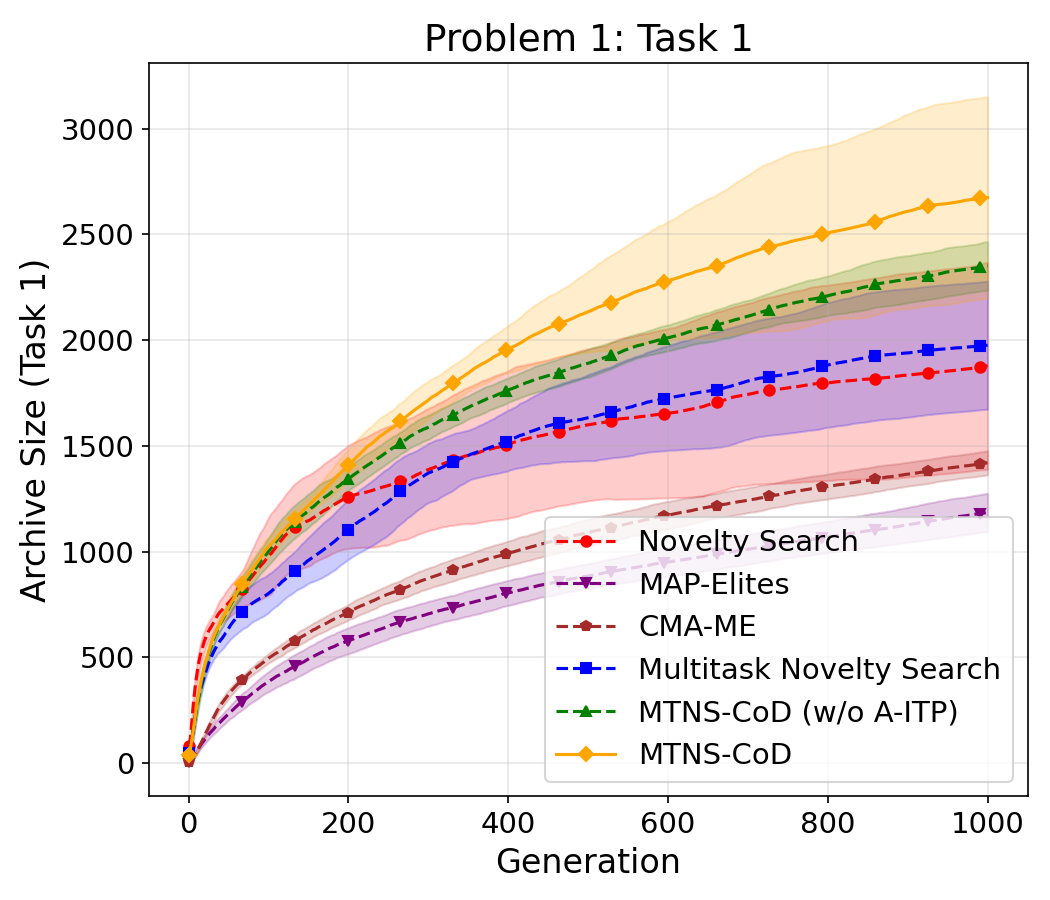}
		\includegraphics[width=0.25\columnwidth]{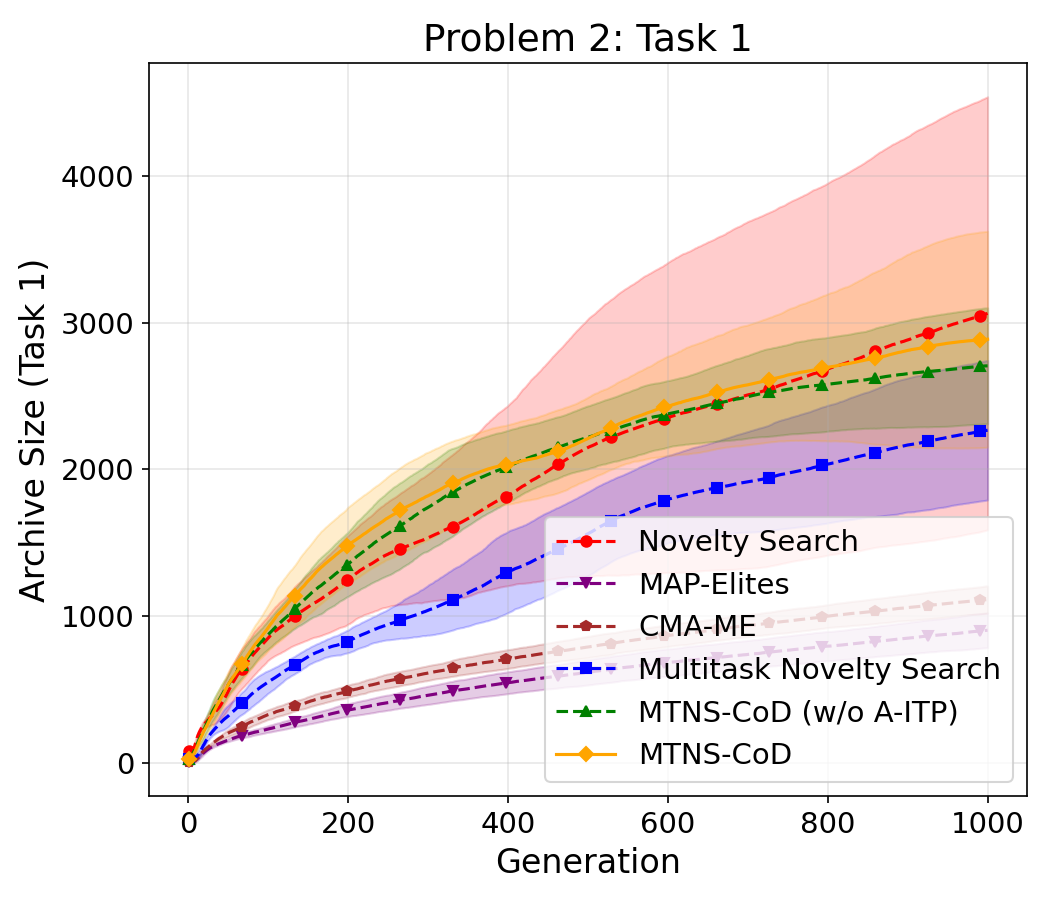}
        \includegraphics[width=0.25\columnwidth]{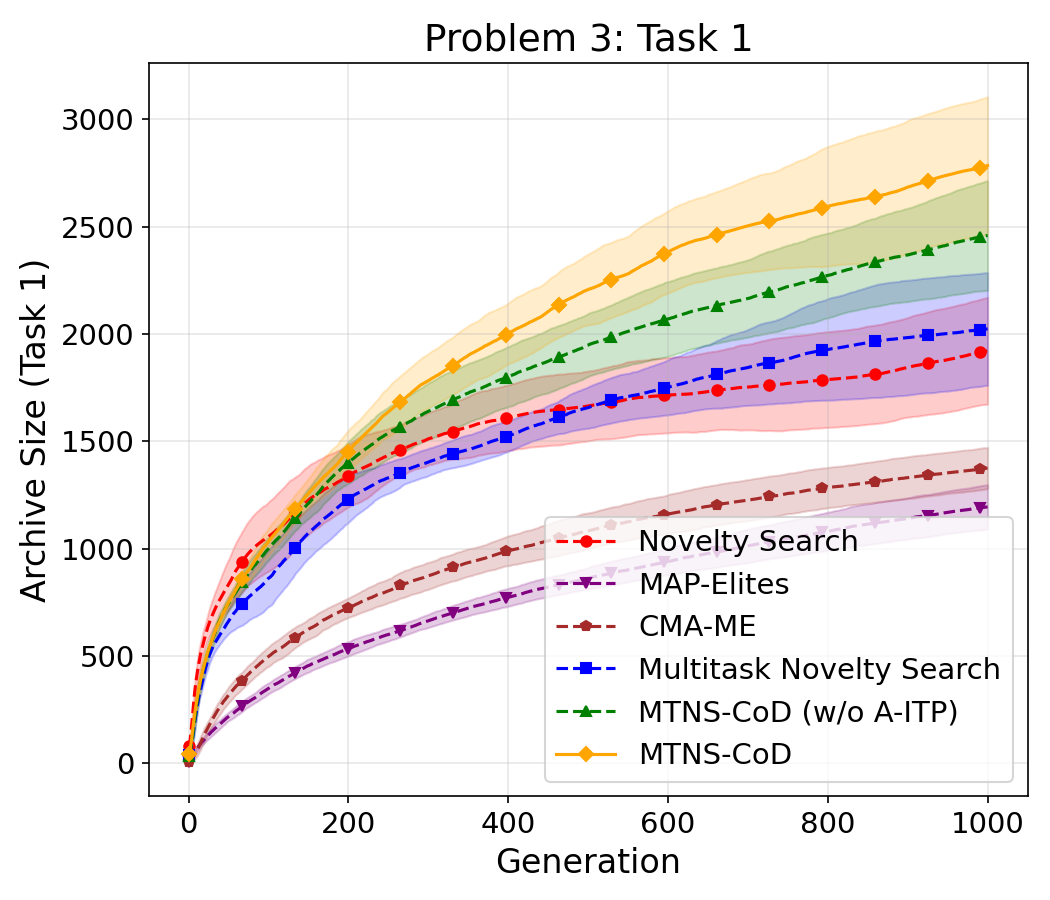}
        \includegraphics[width=0.25\columnwidth]{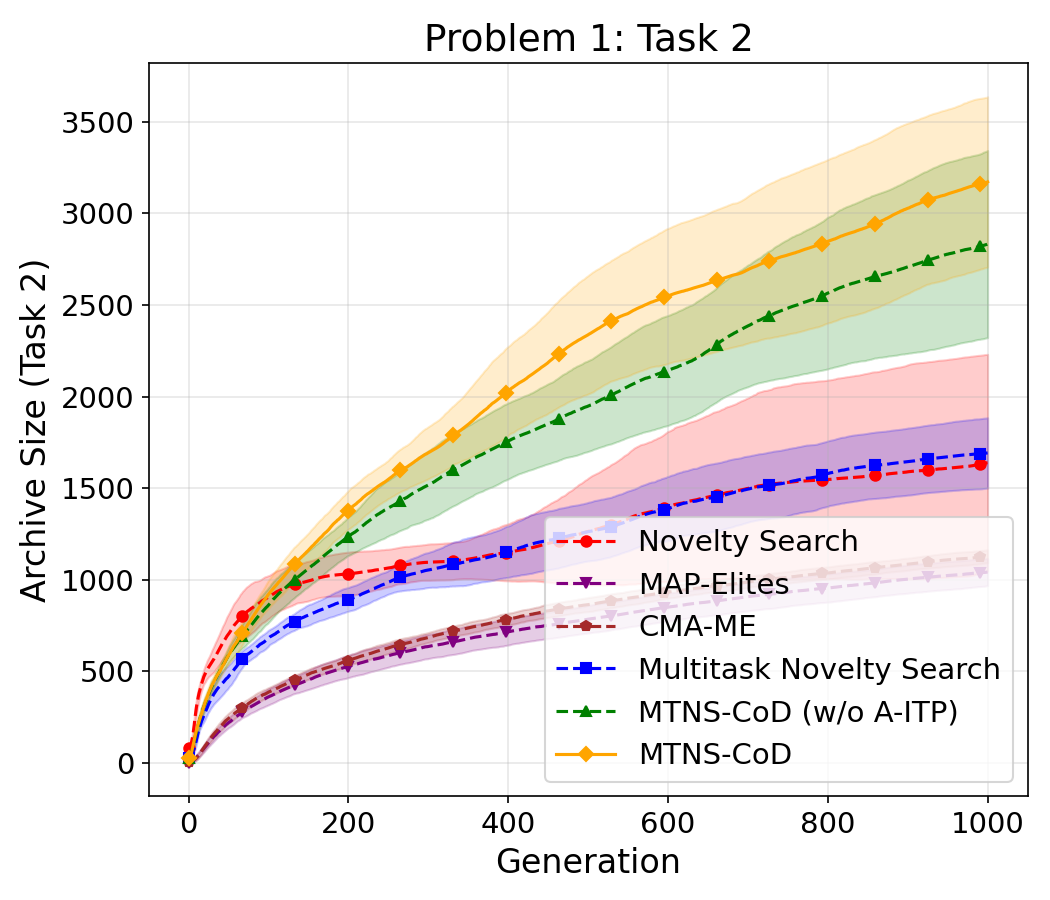}
		\includegraphics[width=0.25\columnwidth]{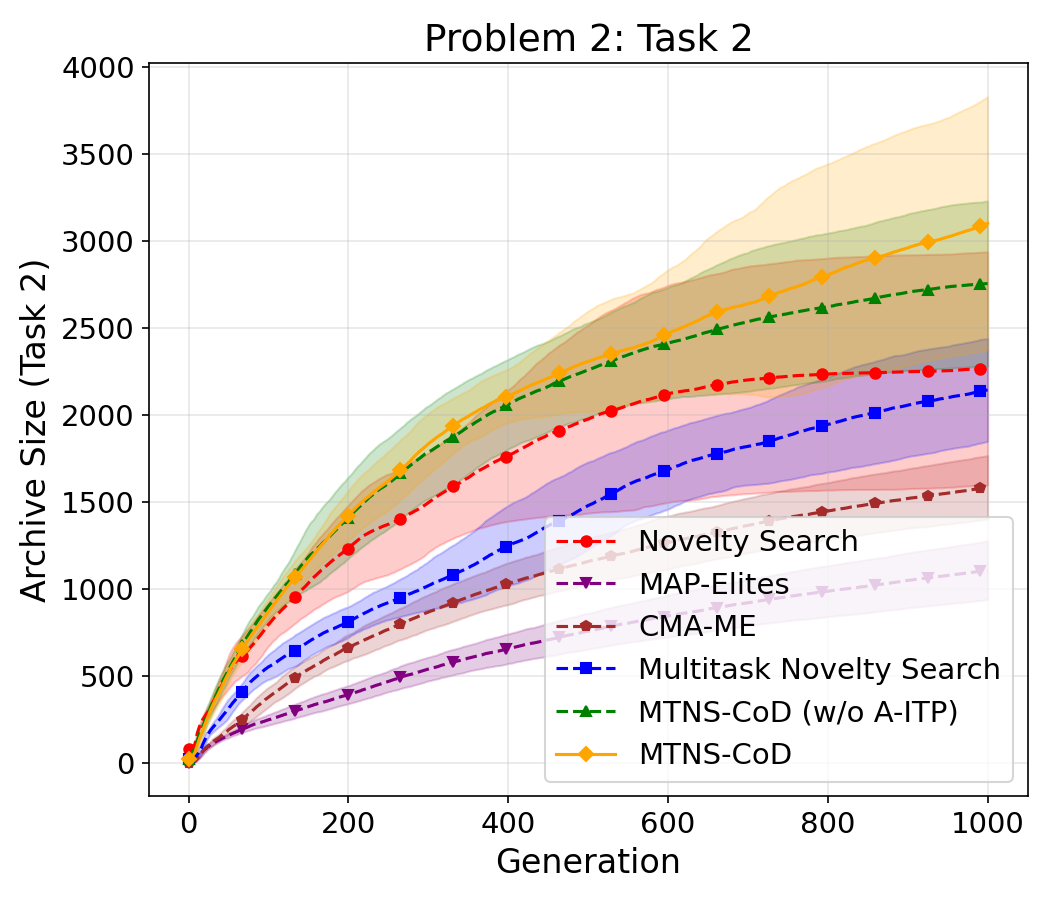}
        \includegraphics[width=0.25\columnwidth]{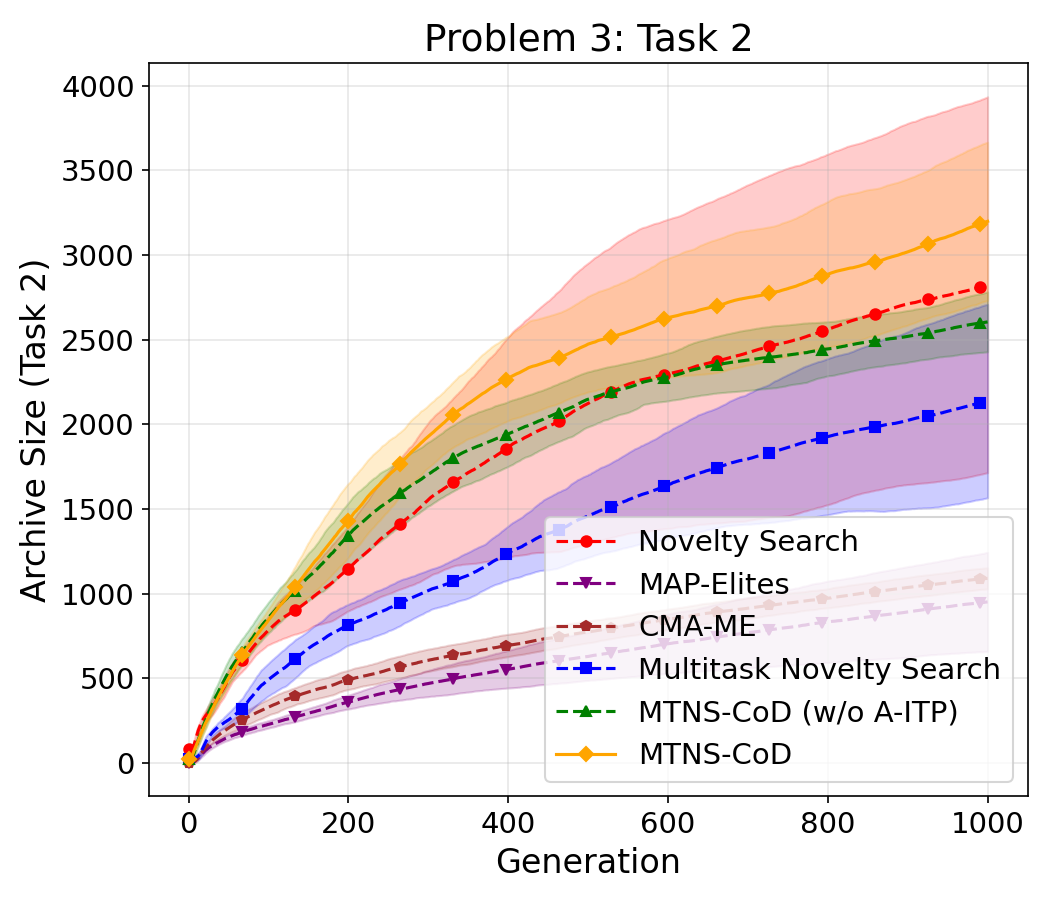}
        \caption{Archive size convergence trends averaged over 20 independent runs of Novelty Search, MAP-Elites, CMA-ME, Multitask Novelty Search, MTNS-CoD (w/o A-ITP), and MTNS-CoD on the multitask deceptive maze navigation problems. 
        }
        \label{fig:case_convergence}
        \vspace{-12pt}
	\end{center}
\end{figure}

\begin{table*}[!t]
\centering
\caption{Archive size results of Novelty Search, MAP-Elites, CMA-ME, Multitask Novelty Search, MTNS-CoD (w/o A-ITP), and MTNS-CoD on the multitask deceptive maze navigation problems. The Wilcoxon rank-sum test at the $0.05$ significance level was performed to compare MTNS-CoD with its competitors.}\label{tab:case_archive_size}
\renewcommand{\arraystretch}{1.12}
\resizebox{\textwidth}{!}{\begin{tabular}{c|c|c|c|c|c|c}
\toprule
\multirow{2}{*}{\textbf{Case / Task}}
& \textbf{Novelty Search}
& \textbf{MAP-Elites}
& \textbf{CMA-ME}
& \textbf{Multitask Novelty Search}
& \textbf{MTNS-CoD (w/o A-ITP)}
& \textbf{MTNS-CoD} \\
\addlinespace[2pt]
& \textit{Mean$\pm$Std}
& \textit{Mean$\pm$Std}
& \textit{Mean$\pm$Std}
& \textit{Mean$\pm$Std}
& \textit{Mean$\pm$Std}
& \textit{Mean$\pm$Std} \\
\midrule
\textbf{Case 1 -- Task 1}
& 1878.7$\pm$489.9\;($+$)
& 1185.2$\pm$90.5\;($+$)
& 1419.5$\pm$56.9\;($+$)
& 1975.7$\pm$303.0\;($+$)
& 2351.0$\pm$114.7\;($+$)
& \textbf{2673.6$\pm$476.6} \\
\textbf{Case 1 -- Task 2}
& 1639.2$\pm$591.8\;($+$)
& 1042.5$\pm$74.2\;($+$)
& 1128.5$\pm$35.8\;($+$)
& 1692.8$\pm$192.2\;($+$)
& 2832.7$\pm$510.5\;($\approx$)
& \textbf{3172.1$\pm$464.4} \\
\addlinespace[2pt]
\textbf{Case 2 -- Task 1}
& \textbf{3062.5$\pm$1475.9}\;($\approx$)
& 902.6$\pm$118.8\;($+$)
& 1111.6$\pm$93.0\;($+$)
& 2265.9$\pm$474.9\;($+$)
& 2704.7$\pm$395.2\;($\approx$)
& {2886.3$\pm$736.2} \\
\textbf{Case 2 -- Task 2}
& 2266.4$\pm$670.0\;($+$)
& 1107.2$\pm$168.6\;($+$)
& 1583.4$\pm$182.7\;($+$)
& 2142.7$\pm$295.7\;($+$)
& 2753.8$\pm$474.2\;($\approx$)
& \textbf{3099.6$\pm$726.5} \\
\addlinespace[2pt]
\textbf{Case 3 -- Task 1}
& 1922.1$\pm$248.8\;($+$)
& 1194.7$\pm$103.5\;($+$)
& 1375.7$\pm$96.1\;($+$)
& 2023.9$\pm$262.8\;($+$)
& 2459.3$\pm$255.3\;($+$)
& \textbf{2785.0$\pm$320.3} \\
\textbf{Case 3 -- Task 2}
& 2824.4$\pm$1110.1\;($\approx$)
& 951.5$\pm$292.7\;($+$)
& 1090.4$\pm$62.5\;($+$)
& 2137.8$\pm$571.2\;($+$)
& 2605.3$\pm$176.9\;($+$)
& \textbf{3199.1$\pm$467.5} \\
\midrule
\multicolumn{1}{c|}{\textbf{$+/-/=$}} & {$4/0/2$}  & {$6/0/0$} & {$6/0/0$} & {$6/0/0$} & {$3/0/3$} & / \\
\bottomrule
\end{tabular}}
\end{table*}

\begin{figure*}[!h]
	\begin{center}
        \includegraphics[width=0.32\columnwidth]{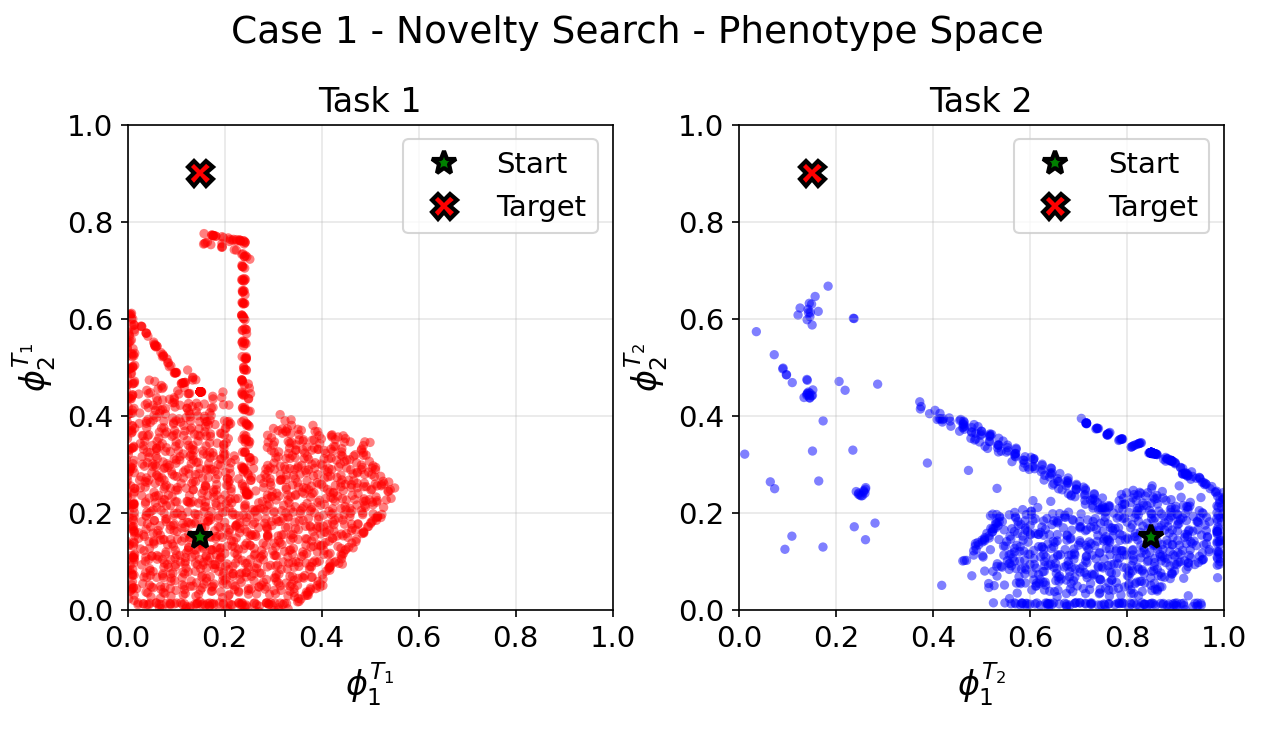}
		\includegraphics[width=0.32\columnwidth]{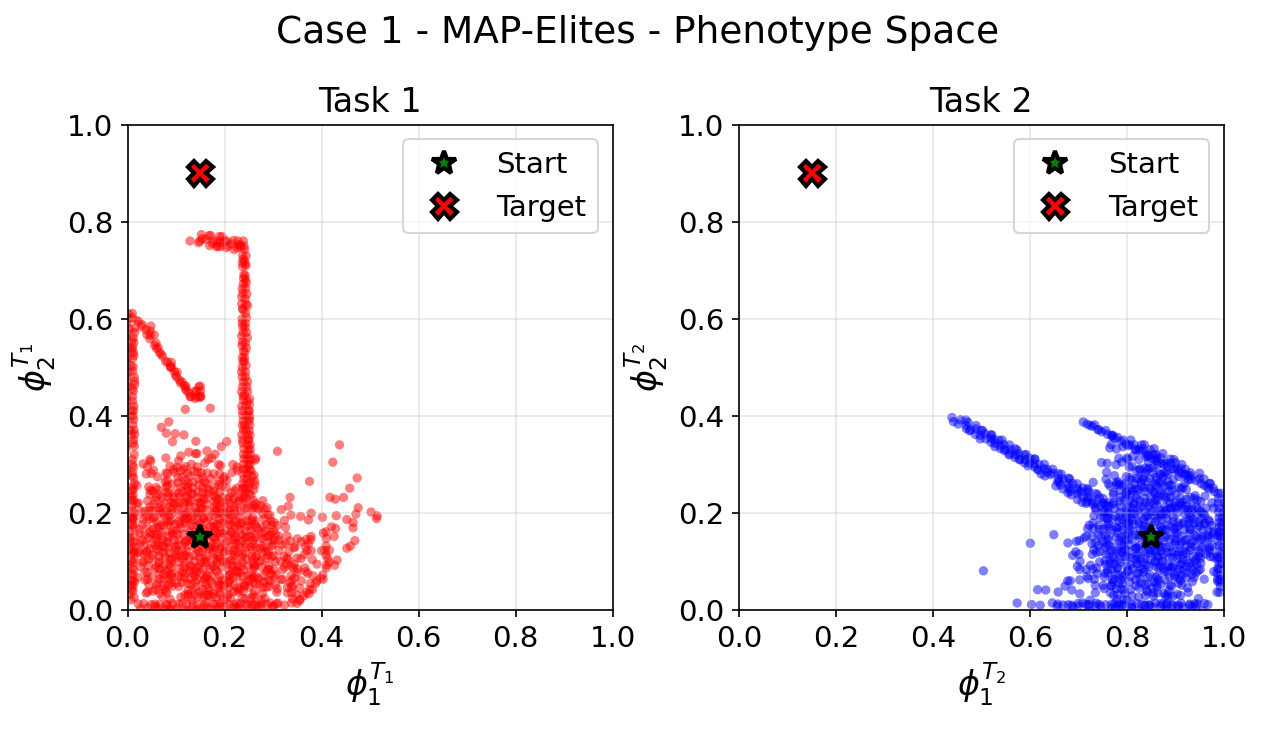}
        \includegraphics[width=0.32\columnwidth]{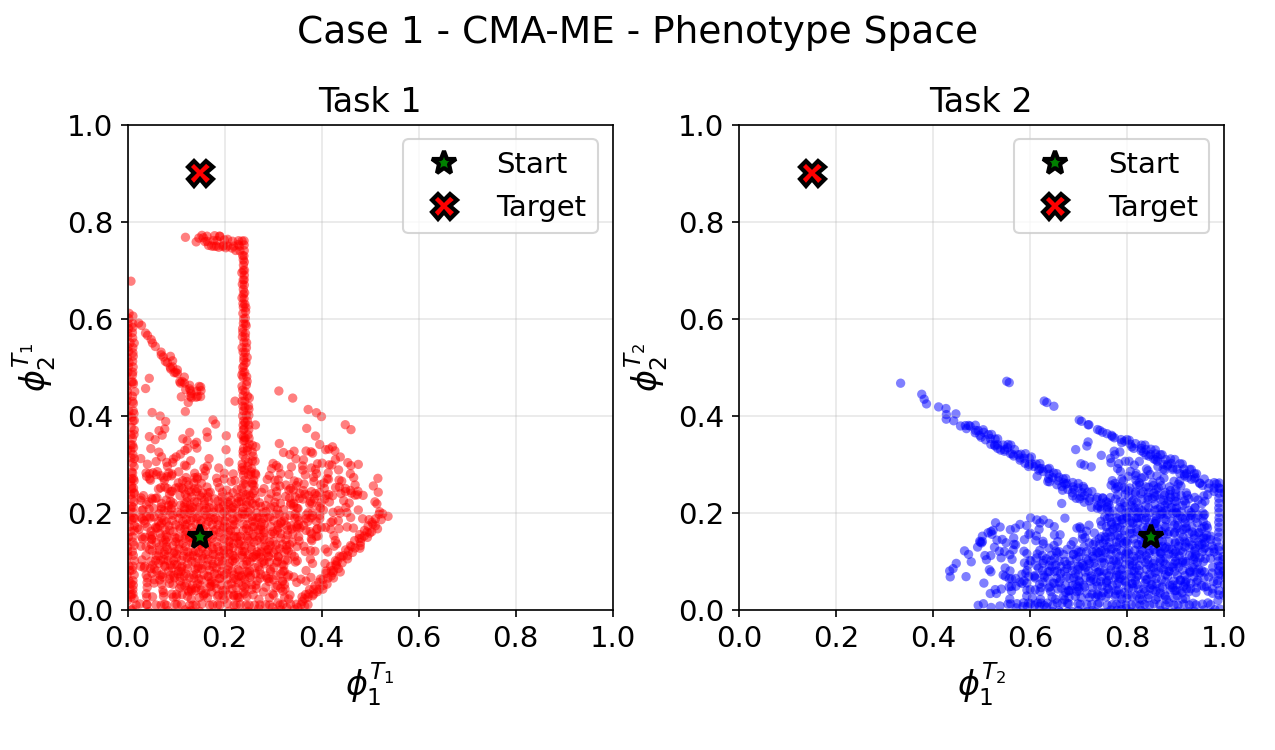}
        \includegraphics[width=0.32\columnwidth]{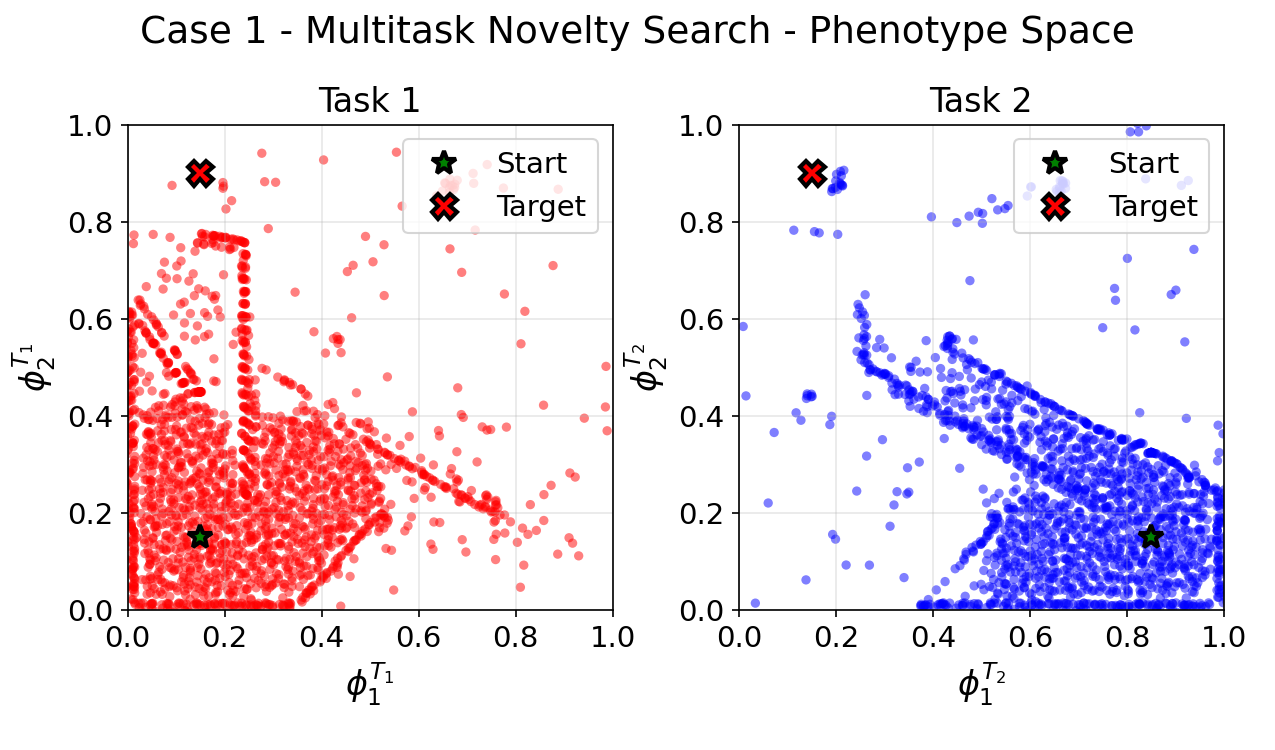}
		\includegraphics[width=0.32\columnwidth]{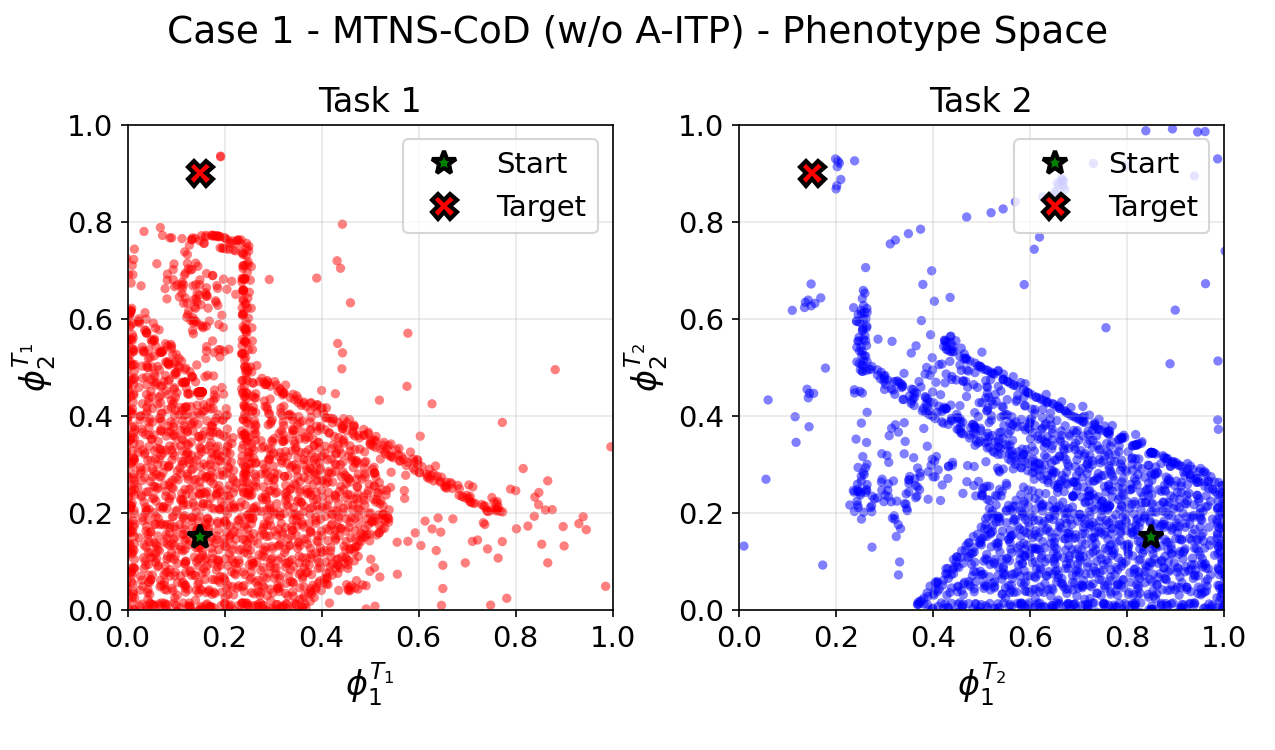}
        \includegraphics[width=0.32\columnwidth]{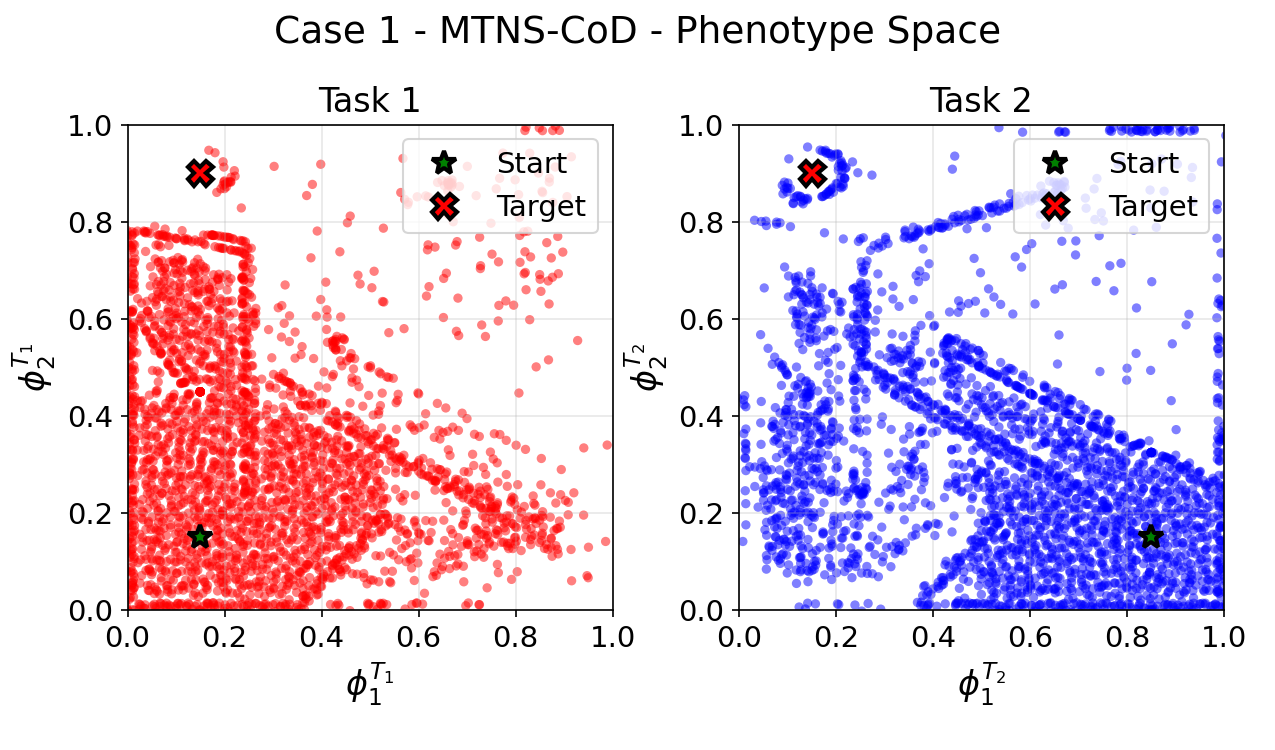}
        \caption{The final archives in the phenotype space of the Novelty Search, MAP-Elites, CMA-ME, Multitask Novelty Search, MTNS-CoD (w/o A-ITP), and MTNS-CoD on the two tasks of Case 1 of the multitask deceptive maze navigation problems. 
        }
        \label{fig:case_phenotpye}
        \vspace{-12pt}
	\end{center}
\end{figure*}

Fig.~\ref{fig:case_convergence} presents the archive-size convergence trends of the six compared methods, i.e., Novelty Search, MAP-Elites, CMA-ME, Multitask Novelty Search, MTNS-CoD (w/o A-ITP), and MTNS-CoD. As can be observed, the proposed MTNS-CoD exhibits the strongest overall convergence performance among all methods, achieving the fastest convergence on five out of the six tasks. This result demonstrates its superior capability in discovering novel solutions in deceptive maze navigation problems. Table~\ref{tab:case_archive_size} further reports the final archive sizes together with the results of the Wilcoxon rank-sum test at the $0.05$ significance level. The results show that MTNS-CoD outperforms the five baselines, achieving statistically significant improvements on four, six, six, six, and three tasks, respectively. 

To further illustrate the performance of MTNS-CoD, we also visualize the distributions of the final archives in the phenotype space, as shown in Fig.~\ref{fig:case_phenotpye}. It can be observed that, compared with the other methods, the two algorithms equipped with the repulsion mechanism, i.e., MTNS-CoD (w/o A-ITP) and MTNS-CoD, are able to discover a larger number of solutions that are also more broadly dispersed in the phenotype space. This indicates that the proposed repulsion mechanism effectively enhances the phenotype space novelty-search capability of the algorithm.

\begin{table*}[!t]
\centering
\caption{Success rates over 10 independent runs and first successful generations of Novelty Search, MAP-Elites, CMA-ME, Multitask Novelty Search, MTNS-CoD (w/o A-ITP), and MTNS-CoD on the multitask deceptive maze navigation problems. The Wilcoxon rank-sum test at the 0.05 significance level is used to compare MTNS-CoD with its competitors.}
\label{tab:maze_success}
\resizebox{\textwidth}{!}{\begin{tabular}{c|cc|cc|cc|cc|cc|cc}
\toprule
\multirow{2}{*}{\textbf{Case / Task}}
& \multicolumn{2}{c|}{\textbf{Novelty Search}}
& \multicolumn{2}{c|}{\textbf{MAP-Elites}}
& \multicolumn{2}{c|}{\textbf{CMA-ME}}
& \multicolumn{2}{c|}{\textbf{Multitask Novelty Search}}
& \multicolumn{2}{c|}{\textbf{MTNS-CoD (w/o A-ITP)}}
& \multicolumn{2}{c}{\textbf{MTNS-CoD}} \\
\cmidrule(lr){2-13}
& \textit{SR} & \textit{FSG Mean$\pm$Std}
& \textit{SR} & \textit{FSG Mean$\pm$Std}
& \textit{SR} & \textit{FSG Mean$\pm$Std}
& \textit{SR} & \textit{FSG Mean$\pm$Std}
& \textit{SR} & \textit{FSG Mean$\pm$Std}
& \textit{SR} & \textit{FSG Mean$\pm$Std} \\
\midrule
\textbf{Case 1 -- Task 1}
& 30.0\% & 600.0$\pm$288.9 ($+$)
& 0.0\% & N/A ($+$)
& 0.0\% & N/A ($+$)
& 60.0\% & 574.3$\pm$158.3 ($+$)
& \textbf{90.0\%} & 399.0$\pm$276.6 ($\approx$)
& \textbf{90.0\%} & \textbf{323.1$\pm$167.0} \\
\textbf{Case 1 -- Task 2}
& 10.0\% & 348.0$\pm$0.0 ($+$)
& 0.0\% & N/A ($+$)
& 0.0\% & N/A ($+$)
& 80.0\% & 441.1$\pm$227.4 ($+$)
& \textbf{100.0\%} & {177.4$\pm$189.1} ($\approx$)
& \textbf{100.0\%} & \textbf{136.2$\pm$138.6} \\
\addlinespace[2pt]
\textbf{Case 2 -- Task 1}
& 70.0\% & 384.6$\pm$225.9 ($+$)
& 20.0\% & 719.5$\pm$184.5 ($+$)
& 10.0\% & 816.0$\pm$0.0 ($+$)
& 90.0\% & 442.6$\pm$135.3 ($+$)
& 90.0\% & 158.4$\pm$173.0 ($+$)
& \textbf{100.0\%} & \textbf{97.2$\pm$69.1} \\
\textbf{Case 2 -- Task 2}
& 70.0\% & 475.6$\pm$244.6 ($+$)
& 30.0\% & 519.7$\pm$275.5 ($+$)
& 40.0\% & 309.8$\pm$328.0 ($+$)
& 90.0\% & 414.8$\pm$96.8 ($+$)
& \textbf{100.0\%} & 92.0$\pm$80.0 ($\approx$)
& \textbf{100.0\%} & \textbf{83.9$\pm$49.6} \\
\addlinespace[2pt]
\textbf{Case 3 -- Task 1}
& 40.0\% & 715.0$\pm$295.2 ($+$)
& 0.0\% & N/A ($+$)
& 0.0\% & N/A ($+$)
& 40.0\% & 711.8$\pm$155.8 ($+$)
& \textbf{100.0\%} & 406.6$\pm$257.5 ($+$)
& \textbf{100.0\%} & \textbf{323.2$\pm$136.1} \\
\textbf{Case 3 -- Task 2}
& 80.0\% & 446.2$\pm$269.7 ($+$)
& 20.0\% & 484.5$\pm$46.5 ($+$)
& 10.0\% & 629.0$\pm$0.0 ($+$)
& 80.0\% & 362.1$\pm$91.2 ($+$)
& \textbf{100.0\%} & 113.9$\pm$83.2 ($\approx$)
& \textbf{100.0\%} & \textbf{97.4$\pm$46.2} \\
\midrule
\multicolumn{1}{c|}{$+/-/=$} 
& \multicolumn{2}{c|}{$6/0/0$}  
& \multicolumn{2}{c|}{$6/0/0$} 
& \multicolumn{2}{c|}{$6/0/0$} 
& \multicolumn{2}{c|}{$6/0/0$} 
& \multicolumn{2}{c|}{$2/0/4$} 
& \multicolumn{2}{c}{ / } \\
\bottomrule
\end{tabular}}
\end{table*}

\subsubsection{Results on Searching Successful Path}
Table~\ref{tab:maze_success} reports the success rate (SR) of finding at least one feasible path over 10 independent runs, as well as the first successful generation (FSG), defined as the generation at which the algorithm first discovers a feasible path. These metrics are introduced to assess whether the proposed method can effectively overcome deceptive objectives and accelerate the discovery of successful paths through the stepping-stone effect. As can be observed, MTNS-CoD exhibits highly competitive performance in this regard. In particular, it achieves a success rate of $100\%$ on five out of the six tasks. Moreover, it attains the smallest mean FSG on all the six tasks, indicating a clear advantage in rapidly identifying feasible paths in several deceptive settings. We also report the results of the Wilcoxon rank-sum test at the $0.05$ significance level on the FSG results. The statistical results further support the effectiveness of the proposed method: MTNS-CoD achieves significant improvements over the five baselines on six, six, six, six, and two tasks, respectively. Overall, these observations demonstrate that the proposed method is not only effective in discovering feasible paths with high reliability, but also capable of accelerating path discovery in deceptive maze navigation tasks.


\subsection{Results on Multitask MuJoCo Policy Optimization Problems}
\begin{figure}[!h]
	\begin{center}
        \includegraphics[width=0.24\columnwidth]{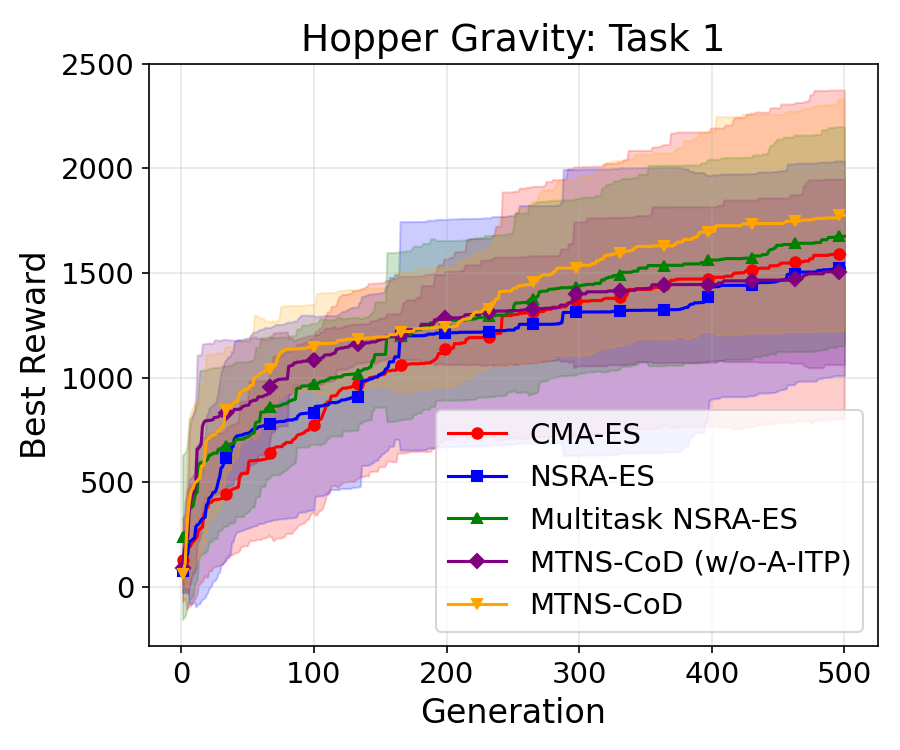}
		\includegraphics[width=0.24\columnwidth]{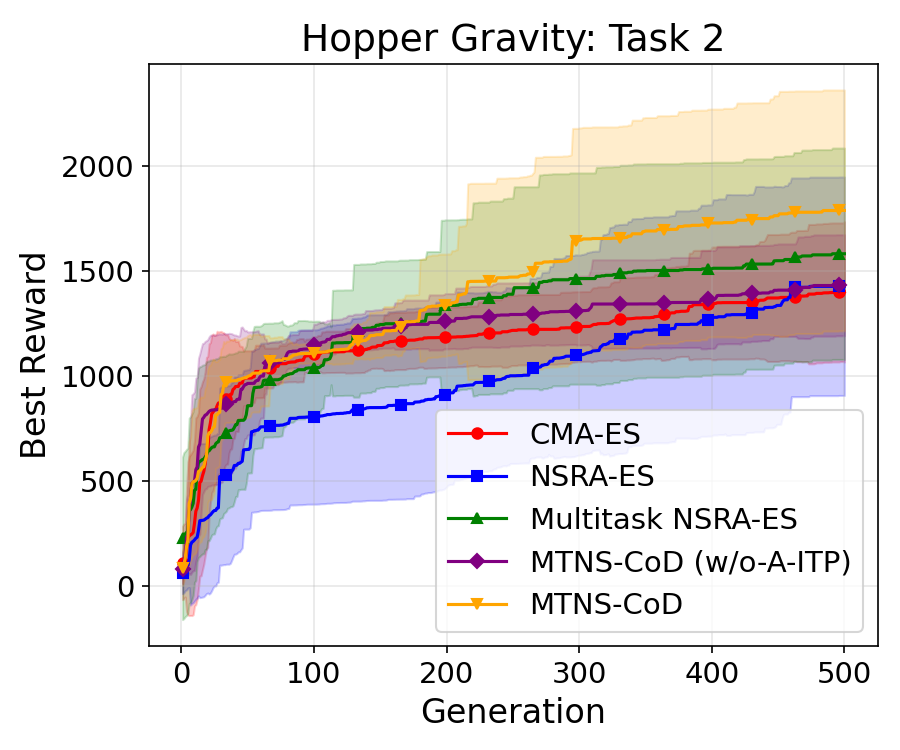}
        \includegraphics[width=0.24\columnwidth]{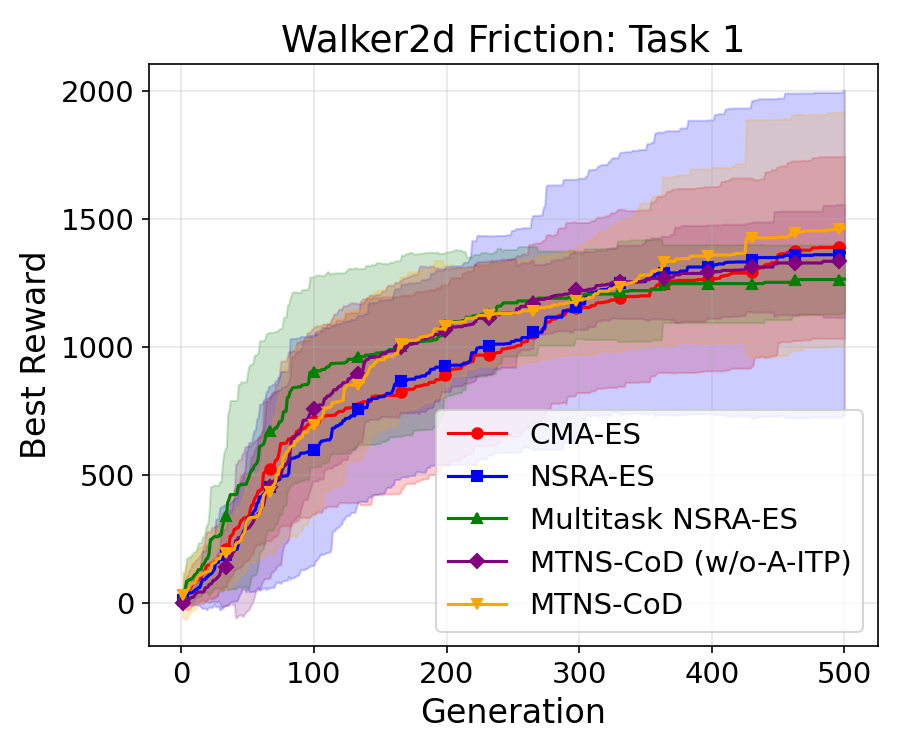}
		\includegraphics[width=0.24\columnwidth]{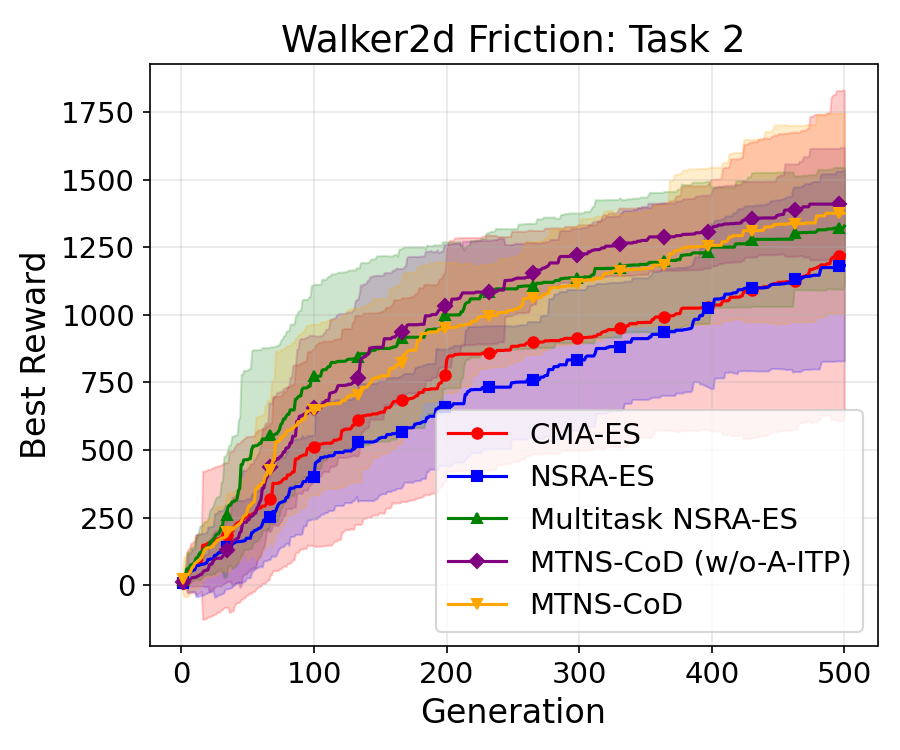}
        \caption{Best reward convergence trends averaged over 10 independent runs of CMA-ES, NSRA-ES, Multitask NSRA-ES, MTNS-CoD (w/o A-ITP), and MTNS-CoD on the multitask MuJoCo policy optimization problems. 
        }
        \label{fig:mujoco_convergence}
        \vspace{-12pt}
	\end{center}
\end{figure}

To assess the effectiveness of MTNS-CoD on novelty-augmented optimization, we evaluate its performance on multitask MuJoCo policy optimization problems. The following baselines are considered:
\begin{itemize}
    \item \textit{CMA-ES}~\cite{hansen2003reducing}: A classical evolutionary optimization algorithm with strong exploration and exploitation trade-off capability. Although it does not explicitly encourage novelty, its stochastic search dynamics may still help escape deceptive fitness landscapes and locate high-quality solutions.
    
    \item \textit{NSRA-ES}~\cite{lehman2011abandoning}: A single-task evolutionary method that adaptively balances novelty-guided exploration and objective-driven exploitation. For multitask problems, it is applied independently to each task.

    \item \textit{Multitask NSRA-ES}: A multitask variant of NSRA-ES obtained by incorporating inter-task transfer to enable cross-task solution information sharing.
\end{itemize}
Moreover, {MTNS-CoD (w/o A-ITP)} is still included to highlight the contribution of the proposed adaptive inter-task transfer probability mechanism.
All methods are implemented in \texttt{pyribs}, and all results are reported over 10 independent runs. The detailed parameter settings can be found in Appendix~\ref{app:settings3}. Performance is evaluated in terms of the convergence trend of the best return achieved by the learned policy.

The convergence trends in Fig.~\ref{fig:mujoco_convergence} show that MTNS-CoD remains competitive on the MuJoCo policy optimization benchmarks. On both tasks of the Hopper Gravity problem, it achieves the fastest convergence in terms of the best reward obtained. On the Walker2d Friction problem, MTNS-CoD converges the fastest on Task~1 and is only slightly inferior to its variant, i.e., MTNS-CoD (w/o A-ITP), on Task~2. These results confirm the effectiveness of MTNS-CoD, indicating that the proposed framework remains effective in novelty-augmented optimization settings. 

\subsection{Results on Multitask Generative Novelty Search Problems}
\begin{figure}[!t]
	\begin{center}
        \includegraphics[width=0.495\columnwidth]{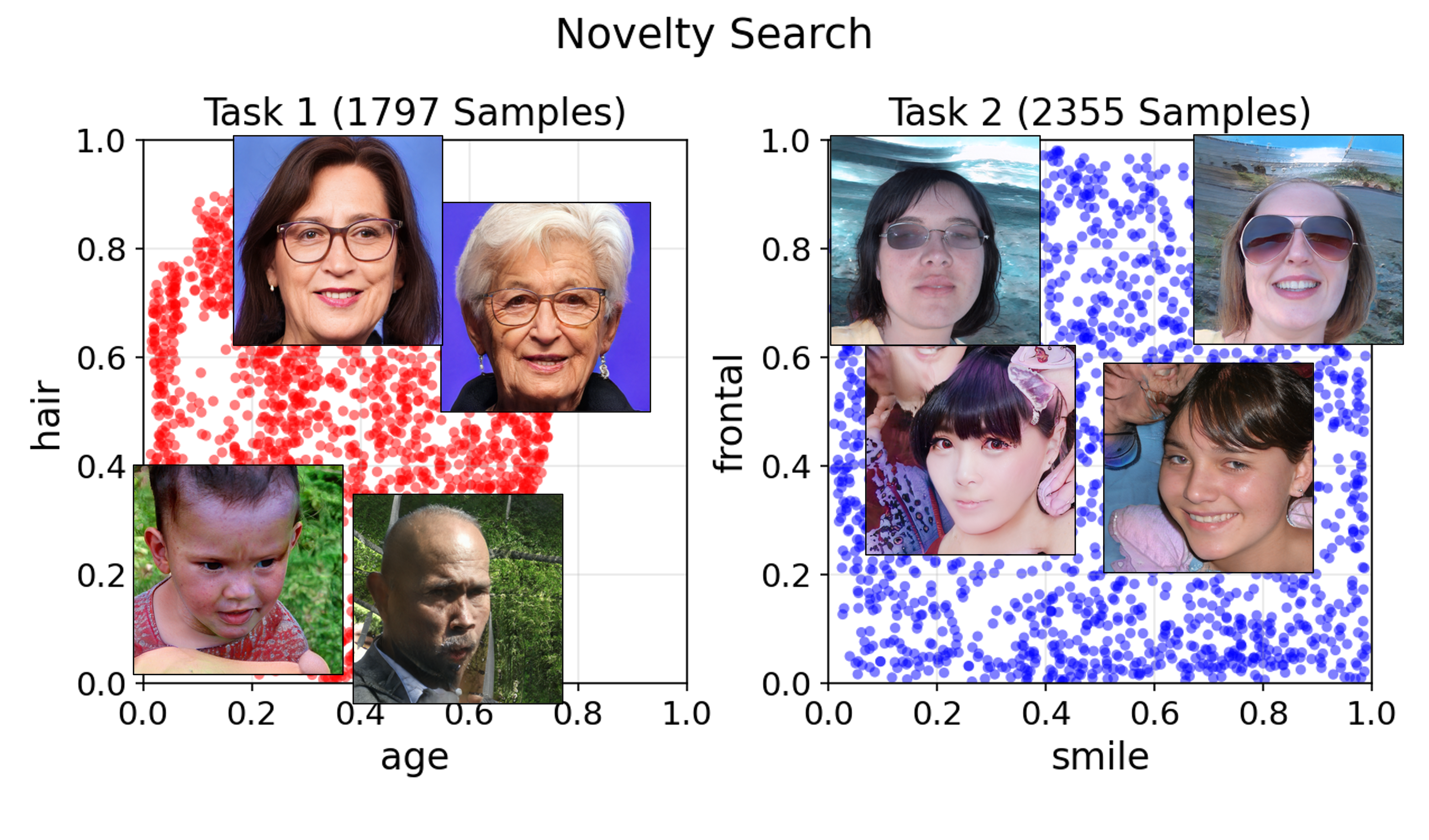}
		\includegraphics[width=0.495\columnwidth]{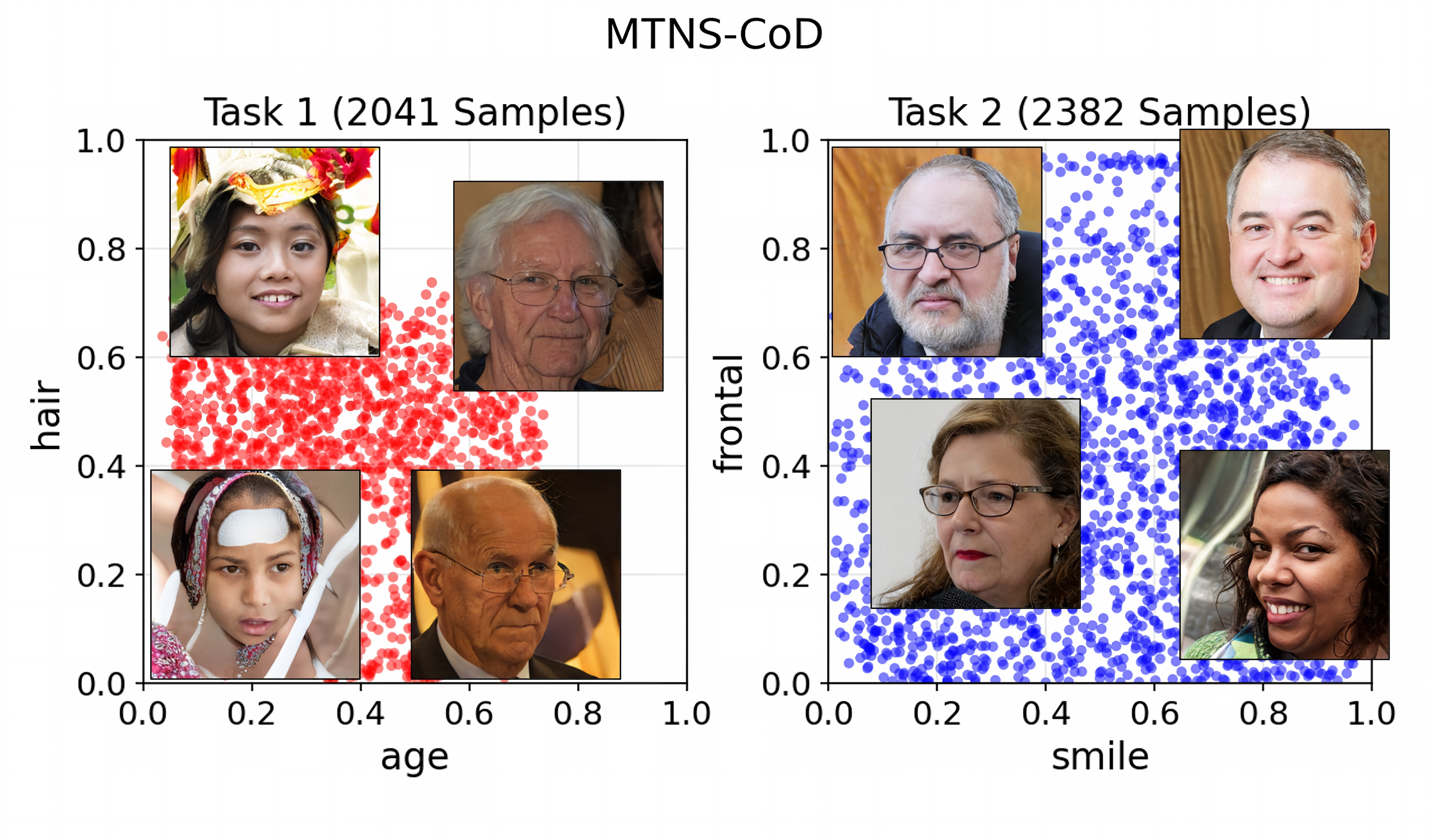}
        \caption{Distribution of the archived samples and representative generated images in the phenotype space obtained by Novelty Search and MTNS-CoD on the multitask generative novelty search problem.
        }
        \label{fig:gen_ns}
        \vspace{-12pt}
	\end{center}
\end{figure}
To provide a visually interpretable example of multitask novelty search, we further examine the behavior of MTNS-CoD on the multitask generative novelty search problem. Specifically, we compare MTNS-CoD with the single-task Novelty Search based on a StyleGAN generator and visualize the distribution of the final archives in the phenotype space. Representative generated samples are also presented to illustrate the novelty-discovery behavior of the two methods. The parameter settings for this experiment are provided in the supplementary file. 
Here, we mainly present representative results from a single run for visualization purposes. As shown in Fig.~\ref{fig:gen_ns}, after one run, MTNS-CoD produces larger archives (2041 samples on Task 1 and 2382 samples on Task 2) than conventional novelty search (1797 samples on Task 1 and 2355 samples on Task 2), indicating a stronger capability for discovering behaviorally novel samples. Moreover, the selected images exhibit clear behavioral variation across different regions of the phenotype space, with substantial differences in the intensities of the two behavior descriptors. For example, in Task 1, samples in the upper-right region show more pronounced characteristics of long hair and older age. Similarly, in Task 2, samples in the lower-right region exhibit clearer non-frontal facial poses and smiling expressions. Overall, these visual results provide intuitive evidence of the effectiveness of the proposed method in discovering diverse novel samples.

\subsection{Sensitivity Analysis}
To assess the robustness of MTNS-CoD to its hyperparameters, we conducted a sensitivity analysis on two key parameters: the size of the temporary archive $N^{\text{tem}}$ used by the repulsion mechanism and the learning rate $\alpha$ used to adapt the inter-task transfer probability. Two synthetic basin-type multitask task problem  were considered, with 20 independent runs for each configuration. The results are shown in Fig.~\ref{fig:sensitivity}.

For the temporary archive size, the final archive size remains relatively stable when the capacity varies from 25 to 200, while a clearer performance degradation and increased variance are observed when the capacity reaches 500. This behavior is consistent with the role of the temporary archive in the repulsion mechanism. The archive distribution is approximated by a single diagonal-covariance Gaussian, which is intended to provide a lightweight local description of the regions explored in recent iterations rather than to model the complete search distribution. When the temporary archive becomes excessively large, it may contain solutions distributed over multiple separated or multimodal regions. Approximating such a distribution with a single Gaussian can become inaccurate, potentially placing the estimated density around regions that are not representative of the actual search occupancy and thereby weakening the guidance provided by repulsion. In contrast, a moderate archive size better preserves the local and recent nature of the approximation, which is more consistent with the dynamic movement of the emitter during search.

For the learning rate of adaptive transfer, the overall performance remains relatively stable across the tested range from 0.05 to 0.8. Although some variation can be observed across individual problems, no strong sensitivity pattern is consistently present, suggesting that the adaptive transfer mechanism is reasonably robust to the learning-rate setting within the examined range.

\begin{figure*}[!t]
	\begin{center}
        \subfigure[]{\label{fig:scenario_1}\includegraphics[width=0.24\columnwidth]{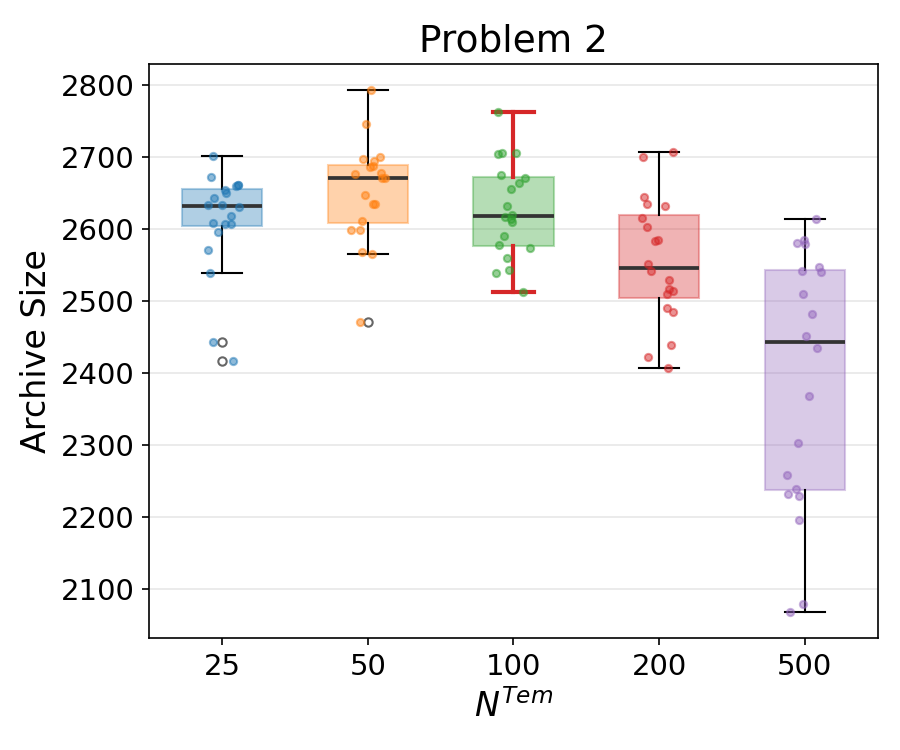}}
		\subfigure[]{\label{fig:scenario_2}\includegraphics[width=0.24\columnwidth]{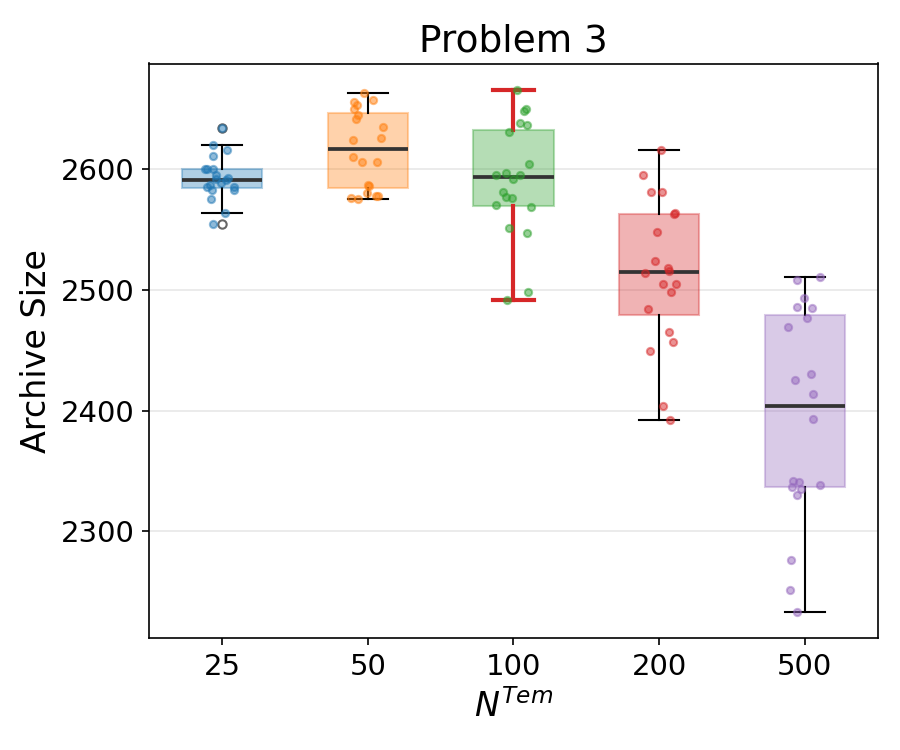}}
        \subfigure[]{\label{fig:scenario_3}\includegraphics[width=0.24\columnwidth]{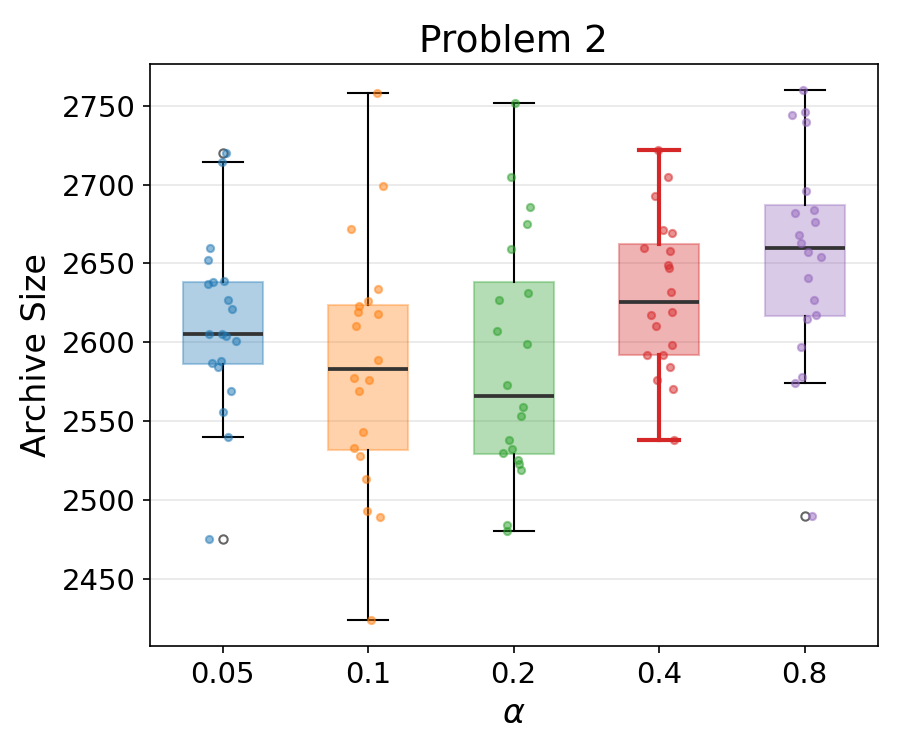}}
        \subfigure[]{\label{fig:scenario_3}\includegraphics[width=0.24\columnwidth]{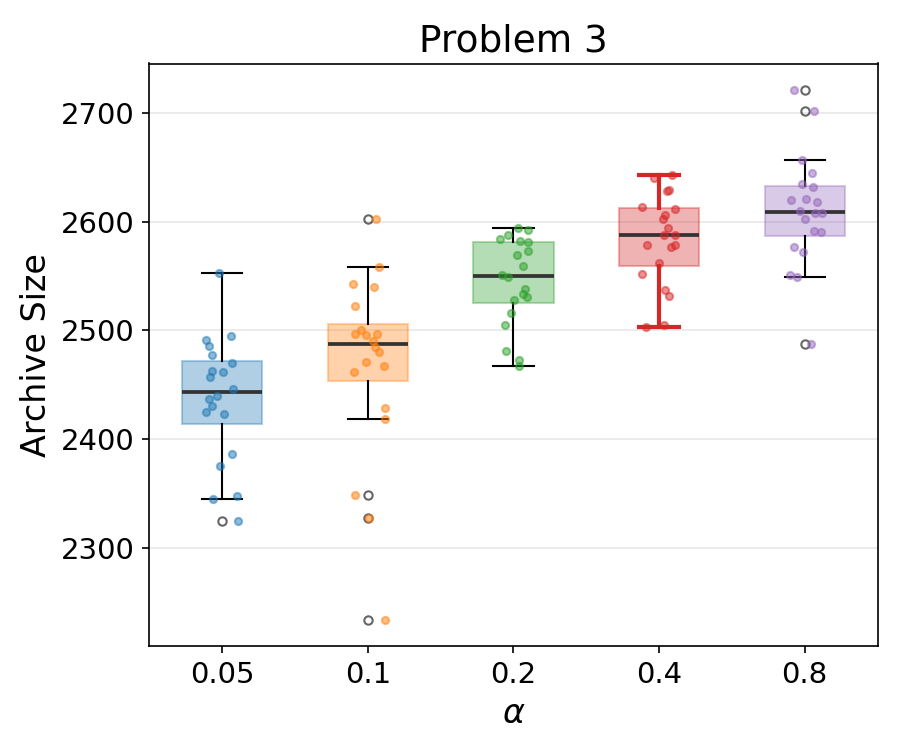}}
        \caption{Sensitivity of MTNS-CoD to the temporary archive size $N^{\mathrm{tem}}$ and the learning rate $\alpha$ used for adaptive inter-task transfer. (a)--(b) Final archive sizes obtained with different $N^{\mathrm{tem}}$ values on Problems~2 and~3. (c)--(d) Final archive sizes obtained with different $\alpha$ values on Problems~2 and~3. Results are based on 20 independent runs.}
        \label{fig:sensitivity}
        \vspace{-12pt}
	\end{center}
\end{figure*}

The default settings used in our experiments, namely a temporary archive size of 100 and a learning rate of 0.4, both lie within relatively stable regions of the sensitivity results. We therefore retain these moderate settings across all experiments rather than tuning them separately for individual problems.

\section{Conclusion}
In this paper, we investigated novelty search in a multitask setting, where multiple related novelty-search tasks are addressed within a unified evolutionary process. We first analyzed how cross-task collaboration can reduce redundant exploration through complementary search and selective reuse of useful discoveries, and provided an idealized theoretical interpretation of these effects. Based on this analysis, we proposed MTNS-CoD for multitask pure novelty search. MTNS-CoD combines a \textit{multitask repulsion} mechanism, which encourages task-specific searches to explore complementary genotype regions, with an \textit{adaptive inter-task transfer} mechanism that adjusts transfer probabilities according to the observed utility of cross-task exchanges. We further extended the framework to novelty-augmented optimization, where behavioral novelty is incorporated alongside objective information to support exploration under deceptive landscapes. Experiments on synthetic basin problems, deceptive maze navigation, MuJoCo policy optimization, and generative novelty search showed that MTNS-CoD generally improves behavioral discovery and search coverage over the considered baselines, while also providing benefits in the examined deceptive-objective settings.

However, several limitations still remain in this paper. First, the theoretical analysis relies on idealized assumptions, such as simplified search distributions, ideal complementary partitioning of search regions, and lower-cost cross-task transfer, and should therefore be interpreted as an analytical explanation of the potential benefits of cross-task collaborative discovery rather than an exact characterization of practical evolutionary dynamics. Second, the present study focuses primarily on multitask pure novelty search, together with an extension to novelty-augmented optimization. More general quality-diversity settings, where behavioral diversity and solution quality are jointly maintained through structured archives, have not yet been explicitly investigated. Future work will extend collaborative discovery to broader multitask quality-diversity settings and further study its behavior under more realistic and complex cross-task search dynamics.

\bibliographystyle{ACM-Reference-Format}
\bibliography{myreference}

\appendix
\setcounter{proposition}{0}
\section{Proof of Propositions}

\subsection{Proof of Proposition 1}

\begin{proof}
Under independent search, the expected cost is
\begin{equation}
\mathbb{E}[T_{\mathrm{ind}}]
=
(b_1+d)H_{a_1+c}
+
(b_2+d)H_{a_2+c}.
\end{equation}
Under collaborative discovery, Task~1 directly searches for
$a_1+c_1$ novelty units within $b_1+d_1$ search units, while Task~2
searches for $a_2+c_2$ novelty units within $b_2+d_2$ search units.
Thus,
\begin{equation}
\mathbb{E}[T_{\mathrm{mt}}]
=
(b_1+d_1)H_{a_1+c_1}
+
(b_2+d_2)H_{a_2+c_2}
+
c\,c_{\mathrm{tr}}.
\end{equation}

For Task~1, using $d-d_1=d_2$ and $c-c_1=c_2$,
\begin{equation}
\begin{aligned}
&(b_1+d)H_{a_1+c}
-
(b_1+d_1)H_{a_1+c_1} \\
&=
(b_1+d)
\sum_{k=a_1+c_1+1}^{a_1+c}\frac{1}{k}
+
d_2H_{a_1+c_1}.
\end{aligned}
\end{equation}
Since $b_1+d\geq a_1+c$, we have
\begin{equation}
(b_1+d)
\sum_{k=a_1+c_1+1}^{a_1+c}\frac{1}{k}
\geq c_2.
\end{equation}
Therefore,
\begin{equation}
(b_1+d)H_{a_1+c}
-
(b_1+d_1)H_{a_1+c_1}
\geq
c_2+d_2H_{a_1+c_1}.
\end{equation}

Similarly,
\begin{equation}
(b_2+d)H_{a_2+c}
-
(b_2+d_2)H_{a_2+c_2}
\geq
c_1+d_1H_{a_2+c_2}.
\end{equation}

Hence,
\begin{equation}
\begin{aligned}
\mathbb{E}[T_{\mathrm{ind}}]
-
\mathbb{E}[T_{\mathrm{mt}}]
&\geq
c_1+c_2-c\,c_{\mathrm{tr}} \\
&\quad
+d_2H_{a_1+c_1}
+d_1H_{a_2+c_2} \\
&=
c(1-c_{\mathrm{tr}})
+d_2H_{a_1+c_1}
+d_1H_{a_2+c_2} \\
&\geq 0,
\end{aligned}
\end{equation}
where the last inequality follows from
$c_1+c_2=c$, $d_1,d_2\geq0$, and
$0<c_{\mathrm{tr}}\leq1$.
Therefore,
\begin{equation}
\mathbb{E}[T_{\mathrm{mt}}]
\leq
\mathbb{E}[T_{\mathrm{ind}}].
\end{equation}
\end{proof}

\subsection{Proof of Proposition 2}
\begin{proof}
Under independent search, Task~$1$ and Task~$2$ explore search regions
containing $b_1+d$ and $b_2+d$ search units, respectively. According to
the coupon-collector abstraction, the expected total computational cost is
\begin{equation}
\mathbb{E}[T_{\mathrm{ind}}]
=
(b_1+d)H_{a_1}
+
(b_2+d)H_{a_2}.
\end{equation}

Under ideal collaborative discovery, the $d$ overlapping search units are
partitioned into two disjoint subsets of sizes $d_1$ and $d_2$, satisfying
\begin{equation}
d_1+d_2=d.
\end{equation}
By assumption, this partition preserves all novelty units required by each
task. Hence, Task~$1$ and Task~$2$ need only explore regions containing
$b_1+d_1$ and $b_2+d_2$ search units, respectively. Since the target
genotype sets are disjoint and ineffective cross-task transfer is fully
suppressed, no additional transfer cost is incurred. Therefore,
\begin{equation}
\mathbb{E}[T_{\mathrm{mt}}]
=
(b_1+d_1)H_{a_1}
+
(b_2+d_2)H_{a_2}.
\end{equation}

Subtracting the two expected costs gives
\begin{equation}
\begin{aligned}
\mathbb{E}[T_{\mathrm{ind}}]
-
\mathbb{E}[T_{\mathrm{mt}}]
&=
\big[(b_1+d)-(b_1+d_1)\big]H_{a_1}\\
&\quad+
\big[(b_2+d)-(b_2+d_2)\big]H_{a_2}\\
&=
(d-d_1)H_{a_1}
+
(d-d_2)H_{a_2}.
\end{aligned}
\end{equation}
Using $d_1+d_2=d$, we have
\begin{equation}
d-d_1=d_2,
\qquad
d-d_2=d_1,
\end{equation}
and therefore
\begin{equation}
\mathbb{E}[T_{\mathrm{ind}}]
-
\mathbb{E}[T_{\mathrm{mt}}]
=
d_2H_{a_1}
+
d_1H_{a_2}.
\end{equation}

Since $a_1,a_2>0$, their harmonic numbers satisfy
\begin{equation}
H_{a_1}>0,
\qquad
H_{a_2}>0.
\end{equation}
Moreover, $d_1,d_2\geq0$. Hence,
\begin{equation}
d_2H_{a_1}+d_1H_{a_2}\geq0,
\end{equation}
which establishes
\begin{equation}
\mathbb{E}[T_{\mathrm{mt}}]
\leq
\mathbb{E}[T_{\mathrm{ind}}].
\end{equation}
\end{proof}

\section{Details of the Benchmark Problems}\label{app:benchmark}

\subsection{Synthetic Basin-Type Multitask Test Problems}\label{app:benchmark1}
To enable intuitive and visually interpretable analysis, we first construct a family of synthetic multitask novelty search problems. These problems are deliberately designed to be simple yet representative, so that the effects of inter-task transfer and repulsion under different overlap structures can be clearly examined. Specifically, we consider a family of two-dimensional basin-type tasks. For each task $\mathcal{T}_i$, the decision variable is $\mathbf{x}=(x_1,x_2)\in\mathcal{X} \subset \mathbb{R}^2$, and the task is parameterized by a task-specific basin center $\mathbf{s}_{i}\in\mathbb{R}^2$. Let the basin half-width be fixed as $w=1$. The objective function, behavioral descriptor, behavioral space, and feasible genotype set of task $\mathcal{T}_i$ are defined as
\begin{equation*}
\label{eq:synthetic_basin_tasks}
\begin{aligned}
\max: f_{i}(\mathbf{x})
&=
\begin{cases}
0, & \mathbf{x}\in [\mathbf{s}_{i}-w,\;\mathbf{s}_{i} + w]^2,\\
-\|\mathbf{x}-\mathbf{s}_{i}\|_2^2-100, & \text{otherwise},
\end{cases}\\[2pt]
\phi_i(\mathbf{x})
&=
\operatorname{Clip}\!\left(
\frac{1}{2}(\mathbf{x}-\mathbf{s}_{i})+\frac{1}{2},
\,0,\,1
\right),\\
\Omega_i
&=
\left\{
\mathbf{x}\in\mathbb{R}^2 \;\middle|\; f_i(\mathbf{x})\ge \tau
\right\},\\[2pt]
\end{aligned}
\end{equation*}
where $\operatorname{Clip}(\cdot,0,1)$ truncates each coordinate into $[0,1]$, and $\tau$ is the objective threshold used in archive admission. In our experiments, $\tau$ is set close to zero so that only solutions located inside the basin are accepted into the novelty archive. Under this construction, novelty search is effectively restricted to the feasible basin, while the resulting behavioral space is normalized to the unit square for convenient visualization.

\begin{table*}[!t]
\centering
\caption{Synthetic basin-type multitask test problem settings.}
\label{tab:synthetic_cases}
\renewcommand{\arraystretch}{1.12}
\resizebox{0.8\textwidth}{!}{\begin{tabular}{clccc}
\toprule
\textbf{Problem} & \textbf{Task} & \textbf{Basin Center} & \textbf{Search Domain} $\mathcal{X}$ & \textbf{Characteristic} \\
\midrule
\multirow{2}{*}{\textbf{\textit{Problem} 1}}
& \textbf{Task 1} & $\mathbf{s}_{1} = (0,\,0)$ & \multirow{2}{*}{$[-1,1]^2$} & \multirow{2}{*}{Fully overlapping target genotype sets} \\
& \textbf{Task 2} & $\mathbf{s}_{2} = (0,\,0)$ &  & \\
\addlinespace[2pt]
\multirow{2}{*}{\textbf{\textit{Problem} 2}}
& \textbf{Task 1} & $\mathbf{s}_{1} = (0.4,\,0.4)$ & \multirow{2}{*}{$[-1.4,1.4]^2$} & \multirow{2}{*}{Partially overlapping target genotype sets} \\
& \textbf{Task 2} & $\mathbf{s}_{2} = (-0.4,\,-0.4)$ &  & \\
\addlinespace[2pt]
\multirow{2}{*}{\textbf{\textit{Problem} 3}}
& \textbf{Task 1} & $\mathbf{s}_{1} = (1,\,1)$ & \multirow{2}{*}{$[-2,2]^2$} & \multirow{2}{*}{Disjoint target genotype sets} \\
& \textbf{Task 2} & $\mathbf{s}_{2} = (-1,\,-1)$ &  & \\
\bottomrule
\end{tabular}}
\end{table*}

To investigate the effect of target genotype set relatedness, we construct three representative two-task problems by varying the basin centers of $\mathcal{T}_1$ and $\mathcal{T}_2$. The corresponding settings are summarized in Table~\ref{tab:synthetic_cases}. These problems represent overlapping, partially overlapping, and disjoint target genotype sets.
For each problem, the search domain is defined as the smallest axis-aligned square covering the union of the two basin regions, i.e.,
\begin{equation*}
\label{eq:synthetic_domain}
\mathcal{X}
=
\left[
\min\!\bigl(\textbf{s}_{1},\textbf{s}_{2}\bigr)-w,\;
\max\!\bigl(\textbf{s}_{1},\textbf{s}_{2}\bigr)+w
\right]^2.
\end{equation*}
Then, the resulting domains are $[-1,1]^2$, $[-1.4,1.4]^2$, and $[-2,2]^2$ for \textit{Problem} 1 to 3, respectively.

Overall, these synthetic basin-type problems provide a simple yet effective testbed for examining the exploration dynamics of MTNS-CoD, particularly in terms of how the proposed framework coordinates inter-task transfer and repulsion under different overlap structures.

\subsection{Multitask Deceptive Maze Navigation Tasks}\label{app:benchmark2}

The deceptive maze navigation task is adopted as a representative benchmark for evaluating the algorithms in the novelty search domain~\cite{lehman2011abandoning,pyribs_ns_maze}. In this problem, a mobile robot must navigate a two-dimensional maze from a given start position to a target location. Its deceptive nature stems from the fact that directly minimizing the distance to the goal is often misleading, as such a strategy may drive the robot toward blocked regions rather than feasible paths. As a result, successful solvers must discover exploratory behaviors that may initially move away from the target before eventually reaching it. This property makes the problem particularly suitable for studying novelty search and other divergent search paradigms without considering the objective.

\begin{table*}[!t]
\centering
\caption{Multitask deceptive maze navigation problem settings.}
\label{tab:maze_cases}
\renewcommand{\arraystretch}{1.12}
\resizebox{0.8\textwidth}{!}{\begin{tabular}{clcccc}
\toprule
\textbf{Case} & \textbf{Task} & \textbf{Maze} & \textbf{Start Position} & \textbf{Target Position} & \textbf{Characteristic} \\
\midrule
\multirow{2}{*}{\textbf{\textit{Case} 1}}
& \textbf{Task 1} & standard & $(0.85,\,0.15)$ & $(0.15,\,0.90)$ & \multirow{2}{*}{Same maze, different starts, same target} \\
& \textbf{Task 2} & standard & $(0.15,\,0.15)$ & $(0.15,\,0.90)$ & \\
\addlinespace[2pt]
\multirow{2}{*}{\textbf{\textit{Case} 2}}
& \textbf{Task 1} & snake & $(0.85,\,0.15)$ & $(0.15,\,0.90)$ & \multirow{2}{*}{Same maze, same start, different targets} \\
& \textbf{Task 2} & snake & $(0.85,\,0.15)$ & $(0.85,\,0.90)$ & \\
\addlinespace[2pt]
\multirow{2}{*}{\textbf{\textit{Case} 3}}
& \textbf{Task 1} & standard & $(0.15,\,0.15)$ & $(0.15,\,0.90)$ & \multirow{2}{*}{Different mazes} \\
& \textbf{Task 2} & snake & $(0.85,\,0.15)$ & $(0.15,\,0.90)$ & \\
\bottomrule
\end{tabular}}
\end{table*}

To further evaluate the proposed MTNS-CoD in deceptive and practically relevant settings, we consider a set of multitask continuous-control problems based on Kheperax maze~\cite{grillotti2023kheperax} navigation. In these tasks, the robot operates in a two-dimensional maze normalized to $[0,1]\times[0,1]$, starting from a predefined initial position and attempting to reach a specified target location. Each episode is executed for a maximum number of 250 control steps, and if the robot fails to reach the target within this horizon, the episode terminates at the final state.

In this paper, the tasks are considered as the \textit{pure novelty search} settings, which means that the objectives are absolutely abandoned. For each task $\mathcal{T}_i$, the behavioral descriptor is defined as the final robot position at the end of the episode, namely
\begin{equation*}
\mathbf{b}_{i}(\pi_{\psi})=(x_{\mathrm{final}},y_{\mathrm{final}})\in\mathcal{B}_i=[0,1]^2,
\end{equation*}
where $\pi_{\psi}$ denotes the control policy decided by a neural network parameterized by $\psi$. Accordingly, novelty search is performed in the two-dimensional behavioral space induced by reachable terminal positions. A rollout is regarded as successful if the Euclidean distance between the final position and the target is smaller than a predefined threshold $\varepsilon$, which is set to $0.1$ in our experiments.

In this paper, we construct three two-task cases to examine multitask transfer under three representative settings: identical maze with different start positions, identical maze with different target positions, and different maze environments. The maze layouts are selected from two predefined Kheperax environments, namely \textit{standard} and \textit{snake}. The former contains relatively complex wall structures, whereas the latter imposes stronger geometric constraints through a narrow winding corridor. The detailed settings are summarized in Table~\ref{tab:maze_cases}. Specifically, \textit{Case}~1 corresponds to the same maze with different starts, \textit{Case}~2 to the same maze with different targets, and \textit{Case}~3 to different maze environments. Together, these cases provide a diverse testbed for evaluating multitask novelty search under varying levels of structural similarity.

\subsection{Multitask MuJoCo Policy Optimization Problems}\label{app:benchmark3}

To further assess the proposed MTNS-CoD in high-dimensional \textit{novelty-augmented optimization} settings, we consider a family of multitask MuJoCo policy optimization problems. In contrast to the previous novelty-search benchmarks, the algorithm aims to maximize the cumulative task reward while leveraging novelty as an auxiliary search signal. Such problems provide a suitable testbed for examining whether MTNS-CoD can leverage collaborative discovery to escape local optima through novelty search more efficiently.

In this paper, we consider two representative MuJoCo locomotion benchmarks~\cite{todorov2012mujoco} derived from Gymnasium environments~\cite{towers2024gymnasium}, namely \texttt{Hopper-v4} and \texttt{Walker2d-v4}. For each benchmark, two related tasks are optimized simultaneously. The paired tasks within the same benchmark share the same robot morphology, observation space, action space, policy representation, and reward definition, but differ in one key physical property of the underlying simulator. Therefore, these benchmarks provide controlled multitask settings in which the tasks remain closely related while exhibiting distinct locomotion dynamics, making cross-task transfer meaningful but nontrivial.

For all tasks, the control policy is represented by a fully connected neural network with two hidden layers, each containing 16 neurons with \texttt{tanh} activation. The output layer also adopts \texttt{tanh}, producing bounded continuous actions. A candidate solution encodes the entire set of network parameters. During evaluation, each policy is executed for at most 1000 simulation steps, and the objective value is defined as the cumulative episodic return under the default reward function of the corresponding Gymnasium environment. In addition, a two-dimensional behavioral descriptor is defined for each benchmark for archive organization and behavior-space analysis.

For the \textit{Hopper Gravity} problem, the two tasks differ only in gravitational acceleration along the vertical axis. Specifically, the first task adopts the standard Earth-like gravity of $-9.81$~m/s$^2$, whereas the second task uses a Moon-like gravity of $-1.62$~m/s$^2$. The observation and action dimensions are 11 and 3, respectively. Following the implementation used in our experiments, the behavioral descriptor is defined as
\[
\mathbf{b}_i(\pi_{\psi})=\bigl(r_{\mathrm{contact}},\, \bar{\theta}_{\mathrm{torso}}\bigr),
\]
where $r_{\mathrm{contact}}$ denotes the fraction of simulation steps during which the foot geom is in contact with the floor, and $\bar{\theta}_{\mathrm{torso}}$ denotes the average torso-tilt-related state statistic over the episode. This descriptor captures both contact dynamics and body posture, thereby distinguishing different hopping strategies induced by the change in gravity.

For the \textit{Walker2d Friction} problem, the two tasks differ only in the friction condition of the simulator. More precisely, the sliding-friction coefficient of every geom in the MuJoCo model is multiplied by a task-specific factor, yielding a normal-friction task with multiplier 1.0 and a low-friction task with multiplier 0.3. The observation and action dimensions are 17 and 6, respectively. The behavioral descriptor is defined as
\[
\mathbf{b}_i(\pi_{\psi})=\bigl(r_{\mathrm{left}},\, r_{\mathrm{right}}\bigr),
\]
where $r_{\mathrm{left}}$ and $r_{\mathrm{right}}$ denote the fractions of simulation steps during which the left foot and the right foot are in contact with the ground, respectively. This descriptor reflects gait style and contact asymmetry, and thus enables the archive to distinguish different locomotion patterns under different friction conditions.

\subsection{Multitask Generative Novelty Search Problem}\label{app:benchmark4}
We further consider the multitask generative novelty search problem based on face image generation. This problem is studied under the pure novelty-search setting. Specifically, the optimization is performed in the latent space of a pre-trained image generator, while the search is driven entirely by novelty in task-specific semantic behavior spaces. By default, for each task, candidate solutions are represented as 512-dimensional latent vectors and are decoded into face images using a pre-trained StyleGAN2~\cite{karras2020analyzing} generator from the Hugging face.

The two tasks differ only in the definition of the behavior descriptor. In both tasks, the behavior coordinates are semantic proxy quantities extracted automatically from generated images by pre-trained vision models, rather than physical measurements or manually annotated attributes. In the current implementation, age is estimated by the vision transformer model \texttt{nateraw/vit-age-classifier}, whereas the remaining semantic attributes are extracted using CLIP, specifically \texttt{openai/clip-vit-base-patch32}~\cite{radford2021learning}. Given a generated RGB face image, the image is first encoded by CLIP into an $L_2$-normalized image feature vector. Each textual prompt is likewise encoded into an $L_2$-normalized text feature vector. Their similarities are then computed by scaled cosine similarity, followed by softmax normalization to obtain a prompt-level probability distribution.

For Task 1, denoted as the \textit{age--hair} task, the behavior descriptor is defined as
\[
\mathbf{b}_1(\mathbf{z}) = \bigl(s_{\mathrm{age}},\, s_{\mathrm{hair}}\bigr).
\]
Here, $s_{\mathrm{age}}$ is obtained from the age classifier by taking the softmax output over age categories, computing the expectation with respect to the representative age of each category, and then dividing the resulting age estimate by 100 so that it is approximately mapped to $[0,1]$. The second coordinate, $s_{\mathrm{hair}}$, is computed using CLIP-based prompt scoring over four hair-related prompts, i.e., ``\textit{a face with no hair bald}", ``\textit{a face with short hair}", ``\textit{a face with medium length hair}", and ``\textit{a face with long hair}". Let $\mathbf{p}_{\mathrm{hair}}$ denote the softmax probabilities of these four prompts. Then the hair score is defined as the weighted sum
\[
s_{\mathrm{hair}}=\sum_{k=1}^{4} p_{\mathrm{hair},k} \, \omega_k,
\]
where the ordered weights $\omega_k, k \in \{1,\ldots,4\}$ are set to $[0,\;0.33,\;0.66,\;1.0]$. Consequently, $s_{\mathrm{hair}}$ provides a scalar semantic coordinate in $[0,1]$ that increases monotonically from bald/short-hair appearance to long-hair appearance. Therefore, the first task organizes generated faces according to semantic variation in age and hair length.

For Task 2, denoted as the \textit{smile--frontal} task, the behavior descriptor is defined as
\[
\mathbf{b}_2(\mathbf{z}) = \bigl(s_{\mathrm{smile}},\, s_{\mathrm{frontal}}\bigr).
\]
Both coordinates are computed from CLIP-based binary prompt matching. More specifically, $s_{\mathrm{smile}}$ is obtained by contrasting a neutral-expression prompt, i.e., ``\textit{neutral expression face portrait}", with a smiling-face prompt, i.e., ``\textit{smiling happy face portrait}", applying softmax to the two corresponding CLIP similarity scores, and taking the probability associated with the smiling prompt. Similarly, $s_{\mathrm{frontal}}$ is obtained by contrasting a side-profile prompt, i.e., ``\textit{side profile view of a human face}" with a front-facing portrait prompt, i.e., ``\textit{front facing portrait looking at camera}", and taking the probability associated with the frontal-view prompt after softmax normalization. Hence, higher values of $s_{\mathrm{smile}}$ and $s_{\mathrm{frontal}}$ indicate stronger semantic evidence of smiling expression and frontal face orientation, respectively.

\section{Implementation Details}\label{app:implementation}
\subsection{Implementation Details of the MTNS-CoD for Pure Novelty Search}
This section provides several implementation details that are not explicitly shown in \textbf{Algorithms~1--3} but are used in the actual realization of MTNS-CoD. In our implementation, the algorithm is based on the emitter-driven framework. For each task, an independent novelty archive is maintained, while a task-specific emitter set is used to generate offspring in the genotype space. During evolution, the algorithm distinguishes between the \emph{source task} of an offspring, i.e., the task to which its emitter belongs, and the \emph{target task}, i.e., the task on which the offspring is actually evaluated after inter-task transfer. This distinction is maintained throughout one generation and is required for both archive update and emitter feedback.

In the implementation, novelty is always computed with respect to the archive of the \emph{target task}. Thus, once an offspring is reassigned to a target task through the inter-task transfer, its behavioral descriptor is evaluated only on that task, and its novelty score is computed only against the archive associated with the target task. If the offspring is admitted, it is directly inserted into the target-task archive, regardless of which source-task emitter generated it. In this way, transferred offspring can make effective contributions to the target task rather than being restricted to their original task domain.

Regarding the fitness feedback and the update of emitters, the implementation does not directly use the raw novelty value from a single archive as the scalar fitness. Instead, emitter update is driven by a multitask novelty ranking procedure. Specifically, for each target task, novelty is computed with respect to the corresponding task archive and ranked within the subset of offspring evaluated on that task. The resulting task-wise rankings are then aggregated into a unified scalar feedback fitness for emitter update, which is the same as the classical MFEA~\cite{gupta2015multifactorial}. Furthermore, under the pure novelty-search setting, an offspring generated by one task but evaluated on another task is assigned a large penalty during the target-task ranking stage, so that it is down-ranked relative to offspring originating from that task. This treatment prevents transferred offspring from dominating the local update of target-task emitters merely due to favorable cross-task evaluation outcomes, thus alleviating undesirable cross-task attraction caused by comparatively better fitness values. Each emitter maintains its own search distribution and is updated independently through an evolution strategy based emitter. After archive insertion and ranking construction are completed, the ranked offspring associated with each emitter are returned to its internal optimizer, which then updates the corresponding search distribution.

\subsection{Implementation Details of MTNS-CoD for Novelty-Augmented Optimization}
Beyond the modifications described in Section~V, the main implementation difference lies in the construction of the scalar fitness feedback used for emitter updates. In contrast to the pure novelty-search setting, emitter updating is no longer driven solely by multitask novelty ranking. Instead, for each target task $\mathcal{T}_i$, the evaluated offspring are first ranked separately according to their objective values and novelty scores. Let $r_i^{\mathrm{obj}}$ and $r_i^{\mathrm{nov}}$ denote the normalized rank-based scores obtained from the objective ranking and novelty ranking, respectively. The fitness used for emitter update is then computed as
$w_i\, r_i^{\mathrm{obj}} + (1-w_i)\, r_i^{\mathrm{nov}}$, 
where $w_i$ is the $i$-th task-specific trade-off coefficient. This normalized fitness is subsequently used by the emitter-side optimizer to rank offspring and update the search distribution. Such a normalization is particularly important in multitask novelty-augmented optimization, since the numerical scales of objective values and novelty scores may vary substantially across tasks; otherwise, emitter adaptation and transfer learning could be biased toward tasks with larger raw values.

Another implementation detail is that, unlike in the pure novelty-search setting, cross-task evaluated solutions are not penalized during ranking. This design is motivated by the objective-driven nature of novelty-augmented optimization: the search should first identify high-quality regions in the solution space through multitask cooperative search, after which novelty can further promote collaborative exploration in the neighborhoods of those promising regions.

\section{Parameter Settings}\label{app:settings}
This section summarizes the main parameter settings used in all experiments. Unless otherwise stated, all algorithms are evaluated under a common protocol within each benchmark category to ensure a fair comparison across task cases and algorithm variants. Since the considered benchmarks differ substantially in problem domain and task characteristics, several settings are problem-specific. Accordingly, the following subsections focus only on the settings that vary across different problem classes.

\subsection{Parameter Settings for the Synthetic Basin-Type Multitask Problems}\label{app:settings1}
For the experiments about the synthetic basin-type multitask problems, each algorithm executes 20 independent runs per instance and 500 generations per run. The main archive is implemented using \texttt{ProximityArchive} with $k=15$ nearest neighbors and a novelty threshold of 0.05. In addition, objective-gating is enabled: solutions with objective values lower than $-0.001$ are not inserted into the archive. This setting restricts archive growth to feasible basin regions while still allowing all evaluated solutions to participate in novelty-based ranking. All methods use the same \texttt{EvolutionStrategyEmitter} configuration, with 5 emitters per task, batch size 10 per emitter, initial step size $\sigma_0=0.1$, \texttt{mu} selection, and the \texttt{basic} restart rule. Initial emitter centers are sampled uniformly from $[-0.5,0.5]^2$. 

The compared multitask variants differ only in their \textit{ITP} and repulsion configurations. In the fixed-transfer setting, the off-diagonal transfer probabilities are fixed, i.e., $\text{itp}_{1,2}=\text{itp}_{2,1}=0.5$. In the adaptive-transfer setting, the \textit{ITP} matrix is updated online using the proposed adaptive method. The learner uses a learning rate $\alpha$ of 0.4, a decay factor $\lambda$ of 0.5, a regularization coefficient $\beta$ of 0.02, and constrains transfer probabilities to the interval $[0.05,0.95]$. For the repulsion-enabled variants of MTNS-CoD, the repulsion strength $\eta$ is set to $0.2$, and the size of temporary archives is set to 100.

\subsection{Parameter Settings for the Multitask Deceptive Maze Navigation Tasks}\label{app:settings2}
For the multitask deceptive maze navigation tasks, the archive and emitter structures follow the same general framework as in the synthetic basin-type multitask problem experiments, while the task representation and optimization budget differ. Each experiment is conducted for 10 independent runs and 1000 generations per run. Each maze task uses a two-dimensional behavior descriptor defined by the final robot position in $[0,1]^2$. Each episode lasts for at most 250 control steps. The controller is a multilayer perceptron with hidden-layer sizes $(8,8)$, and the flattened policy parameters form the search variable. A rollout is counted as successful when the Euclidean distance between the final robot position and the target position is smaller than 0.1.

The maze experiments use a novelty threshold of 0.02, and the size of temporary archives is set to 5000. The same \texttt{EvolutionStrategyEmitter} configuration as above, except that the batch size is increased to 16 per emitter and the initial emitter centers are sampled uniformly from $[0,1]^d$, where $d$ denotes the policy network dimension. The multitask-specific settings are also different: the multitask novelty weight is set to 0, the repulsion strength is set to 0.6, and the adaptive-\textit{ITP} learner uses a learning rate of 1.2, a decay factor of 0.5, and an \textit{ITP} range of $[0.05,0.95]$. Note that, for this group of problems, the repulsion strength will be linearly decayed from 0.6 to 0.3 over generations.

In the fixed-transfer setting, the off-diagonal entries of the transfer matrix are fixed to $\text{itp}_{1,2}=\text{itp}_{2,1}=0.3$. In addition to the four novelty search based methods, the two quality-diversity optimization baselines, namely MAP-Elites and CMA-ME, are also considered. To ensure a fair comparison, these baselines retain the same emitter settings as the novelty-search-based methods, but use a \texttt{GridArchive} of size $100 \times 100$ from the PyRibs library instead of a \texttt{ProximityArchive} based on their default settings.

\subsection{Parameter Settings for the Multitask MuJoCo Policy Optimization Problems}\label{app:settings3}

For the multitask MuJoCo policy optimization problems, the archive and emitter structures follow the same general framework as above, but the optimization is novelty-augmented rather than purely novelty-driven. Each experiment is repeated for 10 independent runs and 500 generations per run. Parallel evaluation is performed using a worker pool whose size is determined automatically from the available CPU resources.

For both problems, the main archive is implemented using \texttt{ProximityArchive} with $k=15$ and local competition enabled, since identical or highly similar phenotypes may correspond to distinct strategies with different objective values. The optimization objective is the cumulative episodic return of the target task, while novelty is computed in the corresponding behavior space. As in the novelty-augmented formulation described in the main text, the objective-novelty trade-off is controlled by adaptive NSRA-style weighting. By default, the initial objective and novelty weights are set to 0.8 and 0.2, respectively; the objective weight is decreased by 0.05 when a task fails to improve its best objective for 15 consecutive generations, and is lower bounded by 0.2.

The emitter configuration remains \texttt{EvolutionStrategyEmitter}, with one emitter per task and a batch size of 64 for each task. The emitter center is initialized from a standard Gaussian vector in the policy parameter space, and no explicit box bounds are imposed on the search distribution. 
For the multitask variants, the fixed-transfer setting uses $\text{itp}_{1,2}=\text{itp}_{2,1}=0.3$, whereas the adaptive-transfer setting updates the \textit{ITP} matrix online. In both MuJoCo scripts, the \textit{ITP} adapter uses a learning rate of 0.2, a decay factor of 0.5, and constrains transfer probabilities to the interval $[0.05,0.95]$. However, there are several different problem-specific parameters, due to the fact that the two problems differed significantly. For the \textit{Hopper Gravity} problem, the novelty threshold is set to 0.01, the temporary archive size is set to 1000, and the regularization coefficient for adaptive \textit{ITP} learning is set to 0.5. For the \textit{Walker2d Friction} problem, the corresponding values are 0.005, 200, and 0.1, respectively. Unless otherwise stated, all remaining settings are shared by the two MuJoCo problems, including the repulsion strength $\eta=0.5$.

\subsection{Parameter Settings for the Multitask Generative Novelty Search Problems}\label{app:settings4}

Finally, for the multitask generative novelty search problems, the overall MFEA novelty-search framework is the same as in the maze setting, but the search is conducted in the latent space of a pre-trained face generator. Each experiment is run for 100 generations. As in the previous pure novelty-search benchmarks, the main archive is implemented using \texttt{ProximityArchive} with $k=15$ nearest neighbors. The novelty threshold is 0.02.

The emitter and optimizer settings follow the same configuration as in the earlier pure novelty-search benchmarks, except that each task uses 3 emitters with batch size 8 per emitter. Since the latent space is high-dimensional, emitter centers are initialized from a Gaussian distribution with standard deviation 0.5. The size of temporary archives for both tasks are set to 2000. The MTNS-CoD settings use novelty weight $\lambda=0.3$, linear decayed repulsion strength from 0.5 to 0, regularization coefficient 0.5, adaptive-\textit{ITP} learning rate 0.1, and the initial transfer probability $\text{itp}_{1,2}=\text{itp}_{2,1}=0.3$.

\end{document}